\documentclass[twoside]{article}

\usepackage[accepted]{aistats2023}
\let\capsection\section
\renewcommand{\section}[1]{\capsection{\uppercase{#1}}}

\usepackage{graphicx}

\usepackage[utf8]{inputenc} 
\usepackage[T1]{fontenc}    
\usepackage{hyperref}       
\usepackage{url}            
\usepackage{booktabs}       
\usepackage{amsfonts}       
\usepackage{nicefrac}       
\usepackage{microtype}      
\usepackage{bm}
\usepackage{amsfonts}
\usepackage{placeins} 
\usepackage[usenames,dvipsnames,svgnames,x11names]{xcolor}
\usepackage{multirow}
\usepackage[disable]{todonotes}

\usepackage{ marvosym }  
\usepackage{subcaption}

\usepackage{siunitx}
\usepackage{xspace}

\usepackage[
backend=bibtex,
style=authoryear,
mincitenames=1,maxcitenames=2,
minbibnames=2,maxbibnames=5,
sorting=nyvt, 
doi=false,isbn=false,url=false,eprint=false,
uniquelist=false,dashed=false,terseinits=true,
giveninits=true]{biblatex}
\newcommand{\citep}{\parencite}
\newcommand{\citet}{\textcite}

\newcommand{\anonymized}[1]{[\textsc{anonymized for peer review}]}

\renewcommand{\vec}{\boldsymbol}%
\newcommand{\mat}{\boldsymbol}%
\newcommand{\Proba}{\mathbb{P}}%
\newcommand{\proba}[1]{\Proba\left (#1\right )}
\newcommand{\Exp}{\mathbb{E}}%
\renewcommand{\Re}{\mathbb{R}}%
\newcommand{\Diag}{\operatorname{Diag}}
\newcommand{\iTrans}{^{-\top}}%
\newcommand{\Trans}{^{\top}}%
\newcommand{\T}{^{\top}}%
\newcommand{\inv}[1]{#1^{-\!1}}%
\renewcommand{\O}{\mathcal{O}}%
\newcommand{\N}{\mathcal{N}}%
\newcommand{\GP}{\mathcal{GP}}%
\newcommand{\norm}[1]{\|#1\|}%
\newcommand{\abs}[1]{|#1|}%
\newcommand{\smallabs}{\abs}%
\renewcommand{\det}[1]{\operatorname{det}\left[#1\right]}
\newcommand{\given}{\operatorname{|}}%

\newcommand{\acgp}{ACGP\xspace}%
\newcommand{\cglb}{CGLB\xspace}%
\newcommand{\openblas}{\textsc{OpenBLAS}\xspace}%
\newcommand{\chol}{\operatorname{chol}}%
\newcommand{\A}{\kschg{p(\vec y_{\mathcal{A}})}}
\newcommand{\B}{\kschg{p(\vec y_{\mathcal{B}} \given \vec y_{\mathcal{A}})}}

\newcommand{\maxN}{N_{\max}}%
\newcommand{\blocksize}{m}%

\newcommand{\tzero}{s}%
\newcommand{\res}{\vec r} 

\newcommand{\lBoundDet}{\log\sigma^2}
\renewcommand{\L}{\mathcal{L}}%
\providecommand{\U}{}
\renewcommand{\U}{\mathcal{U}}%
\newcommand{\noisef}[1]{\sigma^2}%
\newcommand{\noisefm}[1]{\sigma^2\mat I}
\newcommand{\q}[1]{\vec y_{:#1}\Trans\inv{\mat K_{:#1,:#1}}\vec y_{:#1}}

\newcommand{\matQ}[1][j]{\mat Q_{\tzero+1:#1}}%
\newcommand{\estX}[1][t]{\mat X_{\tzero+1:#1}}
\newcommand{\remX}[1][t+1]{\mat X_{#1:N}} 
\newcommand{\avgcov}[1][t]{\rho_{#1}}%
\newcommand{\postk}[3]{k_{#3}(#1, #2)}%
\newcommand{\postkt}[2]{\postk{#1}{#2}{\tzero}}%
\newcommand{\lowermean}[1][t]{\mu_{#1}}%
\newcommand{\loglowermean}[1][t]{\log \lowermean[#1]}%
\newcommand{\GPmeanX}[2][\tzero]{m_{#1}(#2)}%
\newcommand{\GPmean}[2][\tzero]{\GPmeanX[#1]{\vec x_{#2}}}%
\newcommand{\GPvar}[2][\tzero]{\postk{\vec x_{#2}}{\vec x_{#2}}{#1}+\noisef{\vec x_{#2}}}

\newcommand{\Abs}[1]{\left|#1\right|}
\newcommand{\sign}{\operatorname{sign}}%
\newcommand{\cov}{\operatorname{cov}}%
\newcommand{\dx}[1]{\ \mathrm{d}#1}%
\newcommand{\stack}[2]{\begin{bmatrix}
		#1 \\ #2
\end{bmatrix}}%
\newcommand{\concat}[2]{\begin{bmatrix}
		#1 & #2
\end{bmatrix}}%

\newcommand{\vf}{\bm{f}}

\newcommand{\vj}{\bm{j}}

\newcommand{\vm}{\bm{m}}

\newcommand{\vx}{\bm{x}}
\newcommand{\vy}{\bm{y}}

\newcommand{\mA}{\bm{A}}

\newcommand{\mC}{\bm{C}}

\newcommand{\mH}{\bm{H}}

\newcommand{\mJ}{\bm{J}}
\newcommand{\mK}{\bm{K}}
\newcommand{\mL}{\bm{L}}

\newcommand{\mN}{\bm{N}}

\newcommand{\mV}{\bm{V}}
\newcommand{\mW}{\bm{W}}
\newcommand{\mX}{\bm{X}}

\newcommand{\Kff}{\mK_\text{ff}}%
\newcommand{\I}{\mat{I}}%
\newcommand{\sI}{\sigma^2\I}%
\providecommand{\C}{}
\renewcommand{\C}{\bm{\Sigma}}%

\usepackage{colonequals}
\newcommand{\ce}{\colonequals}%
\newcommand{\METRO}{\texttt{metro}}%
\newcommand{\PM}{\texttt{pm25}}%
\newcommand{\PROTEIN}{\texttt{protein}}%
\newcommand{\KINFORTYK}{\texttt{kin40k}}%
\newcommand{\PUMADYN}{\texttt{pumadyn}}%
\newcommand{\BIKE}{\texttt{bike}}
\newcommand{\ELEVATORS}{\texttt{elevators}}
\newcommand{\POLETELECOMM}{\texttt{poletelecomm}}

\newcommand{\latin}[1]{\emph{#1}}
\newcommand{\ie}{\latin{i.e.}\xspace}%
\usepackage{algorithm}
\usepackage{algpseudocode}
\definecolor{newcode}{RGB}{220,220,220}
\newcommand{\augmentedStatement}[1]{\colorbox{newcode}{{#1}}}%
\newcommand{\AState}[1]{\State \augmentedStatement{#1}}
\algrenewcommand{\algorithmiccomment}[1]{\hfill {$\sslash$ \scriptsize #1}}
\algrenewcommand\alglinenumber[1]{\tiny #1}

\usepackage{stmaryrd} 
\newcommand{\eqcomment}[1]{\quad\sslash\text{{\small\emph{#1}}}}%
\newcommand{\leqcomment}[1]{\\&\eqcomment{#1}\notag}

\newcommand{\reqcomment}[1]{}%

\newcommand{\otherX}{\overline{\mat X}}%

\usepackage{pgfplots}
\pgfplotsset{compat=newest}

\pgfplotsset{
	every axis legend/.append style =
	{
		cells = { anchor = east },
		draw  = none
	},
}
\makeatletter
\pgfplotsset{
	range frame/.style={
		tick align = outside,
		axis line style={opacity=0},
		after end axis/.code={
			\draw ({rel axis cs:0,0}-|{axis cs:\pgfplots@data@xmin,0}) -- ({rel axis cs:0,0}-|{axis cs:\pgfplots@data@xmax,0});
			\draw ({rel axis cs:0,0}|-{axis cs:0,\pgfplots@data@ymin}) -- ({rel axis cs:0,0}|-{axis cs:0,\pgfplots@data@ymax});
		}
	}
}
\makeatother

\pgfkeys{/pgfplots/mytuftestyle/.style={
		semithick,
		tick style={major tick length=4pt,semithick,black},
		separate axis lines,
		axis x line*=bottom,
		axis x line shift=5pt,
		xlabel shift=0pt,
		axis y line*=left,
		tick align = outside,
		axis y line shift=5pt,
		ylabel shift=0pt}}

\usepgfplotslibrary{external}
\tikzsetexternalprefix{tikz_out/}
\usetikzlibrary{plotmarks}
\usepgfplotslibrary{groupplots}
\pgfplotsset{plot coordinates/math parser=false}
\pgfkeys{/pgfplots/mystyle/.style={
		semithick,
		tick style={major tick length=4pt,semithick,gray},
		xtick align = inside,
		ytick align = inside,
		xlabel near ticks,
		ylabel near ticks,
}}
\pgfkeys{/pgfplots/resultfigstyle/.style={
		semithick,
		tick style={major tick length=4pt,semithick,gray},
		xtick align = inside,
		xticklabel style = {align=center},
		ytick align = inside,
		xlabel near ticks,
		ylabel near ticks,
		every axis x label/.style={
			at={(ticklabel cs:-0.05, -20)},anchor=near ticklabel
		},
		every axis y label/.style={
			at={(ticklabel cs:1, -25)},anchor=near ticklabel
		},
		every tick label/.append style={font=\tiny}
}}

\newlength\figureheight
\newlength\figurewidth
\newlength\figheight
\newlength\figwidth

\usepackage{mathtools}
\usepackage{amssymb}
\usepackage{amsthm}
\usepackage[capitalise,noabbrev]{cleveref} 
\newtheorem{theorem}{Theorem}

\newtheorem{lemma}[theorem]{Lemma}

\newtheorem{remark}[theorem]{Remark}

\newtheorem{assumption}[theorem]{Assumption}

\usepackage{fixme}                      
\fxsetup{final, english, marginclue, footnote}
\FXRegisterAuthor{ref}{anref}{\textsf{Ref.}}

\FXRegisterAuthor{ks}{ankr}{\color{Orange}Kristoffer}

\newcommand{\kschg}[1]{{\color{Orange}#1}}%
\renewcommand{\kschg}[1]{#1}%
\newcommand{\sbchg}[1]{{\color{Blue}#1}}%
\renewcommand{\sbchg}[1]{#1}%

\newcommand{\revm}[1]{{#1}} 

\newcommand{\kscmt}[2][]{\todo[color=YellowGreen, inline, #1]{Kristoffer: #2}}%

\newcommand{\move}[1]{}

\newcommand{\goldstandard}{\itshape}

\newcommand{\web}[2][this UCI page]{Available at \href{#2}{#1}.}

\definecolor{dark-red}{rgb}{0.4,0.15,0.15}
\definecolor{dark-blue}{rgb}{0.15,0.15,0.4}
\definecolor{medium-blue}{rgb}{0,0,0.5}
\definecolor{light-grey}{rgb}{0.98,0.98,0.98}
\colorlet{light-blue}{blue!50!cyan}
\colorlet{dark-cyan}{blue!30!cyan}
\colorlet{light-red}{red!50!purple}
\hypersetup{
    colorlinks,
    linkcolor={},
    citecolor={light-blue},
    urlcolor={dark-cyan},
    unicode=true,
}

\usepackage{afterpage}
\usepackage{soul}

\renewcommand{\U}{\mathcal{U}}
\renewcommand{\L}{\mathcal{L}}
\renewcommand{\C}{\bm{\Sigma}}
\renewcommand{\O}{\mathcal{O}}

\renewcommand{\det}[1]{\operatorname{det}\left[#1\right]}  
\begin{document}
	\runningauthor{Bartels, Stensbo-Smidt, Moreno-Mu\~noz, Boomsma, Frellsen, Hauberg}
	\twocolumn[
	
	\aistatstitle{Adaptive Cholesky Gaussian Processes}
	
	\aistatsauthor{Simon Bartels \And Kristoffer Stensbo-Smidt \And Pablo Moreno-Mu\~noz}
	\aistatsaddress{University of Copenhagen \And Technical University of Denmark \And Technical University of Denmark}
	\aistatsauthor{Wouter Boomsma \And Jes Frellsen \And S\o{}ren Hauberg}
	\aistatsaddress{University of Copenhagen \And Technical University of Denmark \And Technical University of Denmark}
	]
	\tikzset{external/force remake=false}
	\tikzexternaldisable
	%
	\setlength{\figwidth}{.5\textwidth}%
	\setlength{\figheight}{.25\textheight}%
	\definecolor{color1}{rgb}{1,0.6,0.2}%
	\definecolor{color0}{rgb}{0,0.4717,0.4604}%
		\renewcommand{\vec}{\boldsymbol}%

\begin{abstract}
	We present a method to approximate Gaussian process regression models for large datasets by considering only a subset of the data. 
	Our approach is novel in that the size of the subset is selected on the fly during exact inference with little computational overhead.
	%
	%
%
		From an empirical observation that the log-marginal likelihood often exhibits a linear trend once a sufficient subset of a dataset has been observed, we conclude that many large datasets contain redundant information that only slightly affects the posterior.
		Based on this, we
		provide probabilistic bounds on the full model evidence 
		that can identify such subsets.
		%
		%
		%
		Remarkably, these bounds are largely composed of terms that appear in intermediate steps
	  of the standard Cholesky decomposition, allowing us to
		modify the algorithm
		to adaptively stop
		the decomposition once enough data have been observed.
\end{abstract}

		\section{Introduction}
\begin{figure}
	\centering
	\includegraphics[width=1.015\columnwidth]{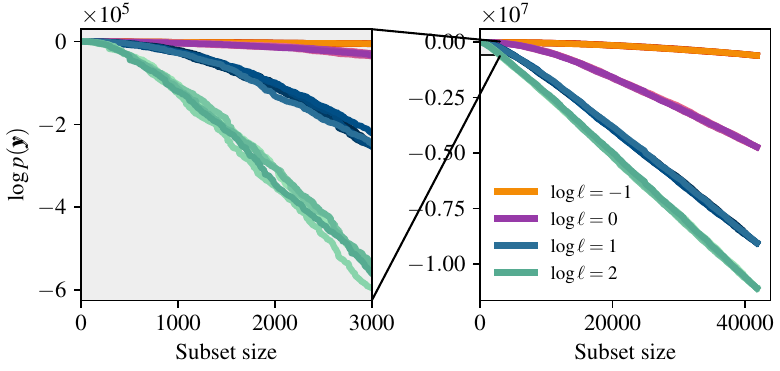}
	\caption{
		The figure shows the log-marginal likelihood as a function of the size of the training set for five random permutations of the \PM{} dataset. The different colors correspond to different Gaussian process models, using the squared exponential kernel with length scale $\ell$.
		Depending on the model (but little on the permutation), the log-likelihood starts to exhibit a linear trend after processing a certain amount of inputs.
		More examples can be found in \cref{app:evidence}.
}
	\label{fig:evidence}
\end{figure}

The key computational challenge in Gaussian process regression is to evaluate
the log-marginal likelihood of the $N$ observed data points, which is known to have
cubic complexity \citep{rasmussenwilliams}.
It has been observed \citep{Chalupka2013comparison} that the random-subset-of-data approximation can be a hard-to-beat baseline for approximate Gaussian process inference.
However, the question of how to choose the size of the subset is non-trivial to answer.
Here we make an attempt.

We first make an empirical observation when studying the behavior of the log-marginal likelihood with increasing number of observations.
\cref{fig:evidence}
show this progression for a variety of models. We elaborate on this figure in \cref{subsec:intuition},
but for now note that after a certain number of observations,
determined by model and dataset,
the log-marginal likelihood starts to progress with a linear trend.
This suggest that we may leverage this near-linearity to estimate the log-marginal likelihood of the full dataset after having seen only a subset of the data.
However, as the point of linearity differs between models and datasets, this point cannot be set in advance but must be estimated on-the-fly. 

In this paper,
we investigate three main questions, namely
1) how to detect the near linear trend when processing datapoints sequentially, 2) when it is safe to assume that this trend will continue
, and 3) how to implement an efficient stopping strategy, that is, without too much overhead to the exact computation.
We approach these questions from a (frequentist)
probabilistic numerics perspective \citep{hennigRSPA2015}.
By treating the dataset as a collection of independent and identically distributed random variables,
we provide expected upper and lower bounds on the log-marginal likelihood, which become tight when the above-mentioned linear trend arises.
These bounds can be evaluated with little computational overhead by leveraging intermediate computations performed
by the Cholesky decomposition that is commonly used for evaluating the log-marginal
likelihood.
We refer to our method as \emph{Adaptive Cholesky Gaussian Process} (\acgp{}).
Our approach has a complexity of $\O(M^3)$, where $M$ is the processed subset-size, inducing an overhead of $\O(M)$ to the Cholesky decomposition.
The main difference to previous work is that our algorithm does \emph{not necessarily} look at the whole dataset, which makes it particularly useful in settings where the dataset is so large that even linear-time approximations are not tractable.
When a dataset contains a large amount of redundant data, \acgp allows the inference procedure to stop early, saving precious compute---especially when the kernel function is expensive to evaluate.
%
		%
		\section{Background}
\label{sec:background}
We use a \textsc{python}-inspired index notation, abbreviating for example $[y_1, \ldots, y_{n-1}]\T$ as $\vy_{:n}$; observe that the indexing starts at 1.
With $\Diag$ we define the operator 
that sets all off-diagonal entries of a matrix to $0$.

\subsection{Gaussian Process Regression}
We start by briefly reviewing Gaussian process (GP) regression models and how they are trained (see \citet[Chapter 2 and 5.4]{rasmussenwilliams}).
We consider the training dataset $\mathcal{D} = \{\vx_n, y_n\}^N_{n=1}$ with inputs $\vx_n \in \mathbb{R}^{D}$ and outputs $y_n \in \mathbb{R}$.
The inputs are collected in the matrix $\mX = [\vx_1, \vx_2, \ldots, \vx_N]\T \in \mathbb{R}^{N\times D}$. 
%
%
A GP $f \sim \GP(m(\vx), k(\vx, \vx'))$ is a collection of random variables defined in terms of a mean function, $m(\vx)$, and a covariance function or \emph{kernel}, $k(\vx, \vx') = \cov(f(\vx), f(\vx'))$, such that any finite amount of random variables has a Gaussian distribution. 
%
Hence, the prior over $\vf\ce f(\mX)$ is $\N(\vf; m(\mat X), \Kff)$,
where we have used the shorthand notation $\Kff = k(\mX, \mX)$. 
Without loss of generality, we assume a zero-mean prior, $m(\cdot)\ce 0$.
We will consider the observations $\vy$ as being noise-corrupted versions of the function values $\vf$, and we shall parameterize this corruption through the likelihood function $p(\vy \given \vf)$, which for regression tasks is typically assumed to be Gaussian, $p(\vy\given\vf) = \N(\vf, \sI)$.
For such a model, the posterior over test inputs $\mat X_*$ can be computed in closed-form: $p(\vf_*\given\vy)= \N(\vm_* , \mat S_*)$, where
\begin{align*}
              \vm_* &= k(\mat X_*, \mat X)\mK^{-1}\vy \quad \text{ and } \\
           \mat S_* &= k(\mat X_*, \mat X_*) - k(\mat X_*, \mat X)\mK^{-1}k(\mat X, \mat X_*)
\end{align*}
with $\mK\ce \Kff + \sI$.
%
By marginalizing over the function values of the likelihood distribution, we obtain the marginal likelihood, $p(\vy)= \int p(\vy\given\vf)p(\vf) d\vf$,
the de facto metric
for comparing the performance of models in the Bayesian framework.
While this integral is not tractable in general, it does have a closed-form solution for Gaussian process regression. Given the GP prior, $p(\vf) = \N(\mathbf{0}, \Kff)$, and the Gaussian likelihood, the log-marginal likelihood distribution can be found to be
\begin{align}
	\label{eq:log_marginal}
  \log p(\vy) = -\frac{1}{2}\left(\log \det{2\pi\mK} + \vy\T\mK^{-1}\vy \right)\, .
\end{align}
Evaluating this expressions costs $\O(N^3)$ operations.

\subsection{Background on the Cholesky decomposition}
\label{sec:cholesky}
Inverting covariance matrices such as $\mK$ is a slow and numerically unstable procedure.
Therefore, in practice, one typically leverages the Cholesky decomposition of the covariance matrices to compute the inverses.
The Cholesky decomposition of a symmetric and positive definite matrix $\mat K$ is the unique, lower\footnote{Equivalently, one can define $\mat L$ to be upper triangular such that $\mat K=\mat L\Trans \mat L$.} triangular matrix $\mat L$ such that $\mat K=\mat L\mat L\Trans$ \citep[Theorem 4.2.7]{golub2013matrix4}.
The advantage of having such a decomposition is that inversion with triangular matrices amounts to Gaussian elimination.
There are different ways to compute $\mat L$.
The Cholesky of a $1\times 1$ matrix is the square root of the scalar.
For larger matrices,
\begin{align}
	\label{eq:cholesky}
	\chol[\mat K]=
	\begin{bmatrix}\chol[\mat K_{:\tzero,:\tzero}] & \mat 0\\
		\mat T & \chol\left[\mat K_{\tzero:,\tzero:}-\mat T\mat T\Trans\right]
	\end{bmatrix},
\end{align}
where $\mat T\ce \mat K_{\tzero:,:\tzero}{\chol[\mat K_{:\tzero,:\tzero}]}\iTrans$ and $\tzero$ is any integer between $1$ and the size of $\mat K$.
Hence, extending a given Cholesky to a larger matrix requires three steps:
\begin{enumerate}
	\setlength\itemsep{0em}
	\item solve the linear equation system $\mat T$,
	\item apply the downdate $\mat K_{\tzero:,\tzero:}-\mat T\mat T\Trans$ and
	\item compute the Cholesky of the down-dated matrix.
\end{enumerate}
An important observation is that $\mat K_{\tzero:,\tzero:}-\mat T\mat T\Trans$ is the posterior covariance matrix $\mat S_*+\sigma^2\mat I$ when considering $\mat X_{\tzero:}$ as test points.
We will make use of this observation in \cref{sec:bound_realization}.
The log-determinant of $\mat K$ can be obtained from the Cholesky using $\log\det{\mat K}=2\sum_{n=1}^{N}\log \mat L_{nn}$.
A similar recursive relationship exists between the quadratic form $\vec y\Trans \inv{\mat K}\vec y$ and $\inv{\mat L}\vec y$ (see appendix, \cref{eq:recursive_LES}).

\subsection{Related work}
Much work has gone into tractable approximations to the log-marginal likelihood. 
Arguably, the most popular approximation methods for GPs are inducing point methods \citep{quionero2005unifying,snelson2006sparse,titsias2009variational,hensman2013gaussian,hensman2017variational,shi2020sparse,artemev2021cglb}, where the dataset is approximated through a set of pseudo-data points (inducing points), summarizing information from nearby data.
Other approaches involve building approximations to $\mat K$ \citep{Fine2001Lowrank,Rahimi2008RandomFeatures,LazaroGredilla2010SparseSpectrum,Harbrecht2012pivotedCholesky,wilson2015kernel,Rudi2017Falkon,wang2019exact} or aggregating of distributed local approximations \citep{gal2014distributed,deisenroth2015distributed}.
One may also consider separately the approximation of the quadratic form via linear solvers such as conjugate gradients \citep{hestenes1952methods,Cutajar2016Preconditioning} and the approximation of the log-determinant \citep{Fitzsimons2017BayesianDeterminant,Fitzsimons2017EntropicDeterminant,Dong2017scalable}.
Another line of research is scaling the hardware \citep{Nguyen2019distributedGP}.

All above referenced approaches have computational complexity at least $\O(N)$
(with the exception of \citet{hensman2013gaussian} since it uses mini-batching).
However, the size of a dataset is seldom a particularly chosen value but rather the ad-hoc end of the sampling procedure.
The dependence on the dataset size implies that more data requires more computational budget even though more data might not be helpful.
This is the main motivation for our work: to derive an approximation algorithm where computational complexity does not depend on redundant data.

The work closest in spirit to the present paper is by \citet{artemev2021cglb}, who also propose lower and upper bounds on quadratic form and log-determinant. 
There are a number of differences, however. Their bound relies on the method of conjugate gradients where we work directly with the Cholesky decomposition.
Furthermore, while their bounds are deterministic, ours are probabilistic, which can make them tighter in certain cases, as they do not need to hold for all worst-case scenarios.
This is also the main difference to the work of \citet{hensman2013gaussian}.
Their bounds allow for mini-batching, but these are inherently deterministic when applied with full batch size.
\section{Methodology}
\label{sec:methods}
In the following, we will sketch our method.
Our main goal is to convey the idea and intuition.
To this end, we use suggestive notation.
We refer the reader to the appendix for a more thorough and formal treatment.


\subsection{Intuition on the linear extrapolation}
\label{subsec:intuition}
The marginal likelihood is typically presented as a joint distribution, but,
 using Bayes rule,
one can also view it from a cumulative perspective as the sum of log-conditionals:
\begin{align}
	\label{eq:log_marginal_alternative}
	\log p(\vy) &= \sum_{n=1}^N \log p(y_n \given \vy_{:n})\ .
\end{align}
With this equation in hand, the phenomena in \cref{fig:evidence} becomes much clearer.
The figure shows the value of \cref{eq:log_marginal_alternative} for an increasing
number of observations $n$. When the plot exhibits a linear trend, it is
because the summands $\log p(y_{n}\given \vy_{:n})$ become approximately constant,
implying that the model is not gaining additional knowledge.
In other words, new outputs are conditionally independent given the output observations seen so far.

The key problem addressed in this paper is how to estimate the full marginal likelihood, $p(\vy)$, from only a subset of $M$ observations. 
The cumulative view of the log-marginal likelihood in \cref{eq:log_marginal_alternative} is our starting point.
In particular, we will provide probabilistic bounds, which are functions of seen observations, on the estimate of the full marginal likelihood. These bounds will allow us to decide, on the fly, when we have seen enough observations to accurately estimate the full marginal likelihood.


\subsection{Stopping strategy}
Suppose that we have processed $M$ data points with $N-M$ data points
yet to be seen.
We can then decompose \cref{eq:log_marginal_alternative} into a sum of terms, which have already been computed, and a remaining sum
\begin{align*}
  \log p(\vec y) &= \underbrace{\sum_{n=1}^M \log p(y_{n}\mid \vec y_{:n})}_{\A: \text{ processed}}
   + \underbrace{\sum_{n=M+1}^{N} \log p(y_{n}\mid \vec y_{:n})}_{\B: \text{ remaining}}.
\end{align*}
Recall that we consider the $\vec x_i, y_i$ as independent and identically distributed random variables.
Hence, we could estimate $\B$ as $(N-M)\A/M$. 
Yet this is estimator is biased, since $(\vec x_{M+1}, y_{M+1}), \dots, (\vec x_N, y_N)$ interact non-linearly through the kernel function.
Instead, we will derive unbiased lower and upper bounds, $\L$ and $\U$.
To obtain unbiased estimates, we use the last-$m$ processed points, such that conditioned on the points up to $\tzero\ce M-m$, the expected value of $\log p(\vec y)$ can be bounded from above and below:
\begin{align*}
  \Exp[\L\given \mX_{:\tzero}, \vy_{:\tzero}]
  \leq
  \Exp[p(\vec y)\given \mX_{:\tzero}, \vy_{:\tzero}]
  \leq
  \Exp[\U \given \mX_{:\tzero}, \vy_{:\tzero}],
\end{align*}
and the observations from $\tzero$ to $M$ can be used to estimate $\L$ and $\U$.
\cref{fig:method_sketch} shows a sketch of our approach.
\begin{figure}%
	\includegraphics[width=0.48\textwidth]{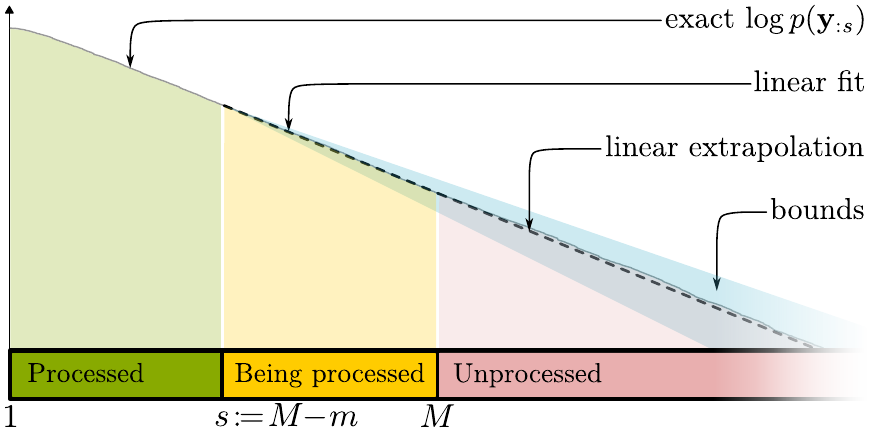}%
	\caption{
			Illustration of how \acgp proceeds during estimation of the full $\log p(\vy)$.
			The Cholesky decomposition works by processing data in blocks of size $m$ (see \cref{eq:cholesky}), so \acgp computes the log-marginal likelihood in blocks of size $m$ as well. In the illustration, $\tzero$ datapoints have been fully processed, meaning we have the exact $\log p(\vy_{:\tzero})$ for those. As the next $m$ data are being processed, we can compute the bounds on the \emph{full} $\log p(\vy)$ (\ie, including the unprocessed data) after step 2 of the Cholesky decomposition.
		If the stopping conditions in \cref{eq:stop_cond_r} are met, we return the linear extrapolation as estimate of $\log p(\vy)$.
	\cref{thm:main,thm:guaranteed_precision} describe the conditions under which this estimate achieves the desired error with high probability.
	}%
	\label{fig:method_sketch}%
\end{figure}%
%
We can then detect when the upper and lower bounds are sufficiently near each other, and stop computations early when the approximation is sufficiently good.
More precisely, given a desired relative error $r$, we stop when
\begin{align}
\label{eq:stop_cond_r}
\frac{\U-\L}{2\min(|\U|, |\L|)}<r \quad\text{and}\quad \sign(\U)=\sign(\L)\, .
\end{align}
If the bounds hold, then the estimator $(\L+\U)/2$ achieves the desired relative error (\cref{lemma:rel_err_bound} in appendix).
This is in contrast to other
approximations, where one specifies a computational budget, rather than a desired
accuracy.
\subsection{Bounds on the log-marginal likelihood}
\label{sec:bounds}
From \cref{eq:log_marginal}, we see that the log-marginal likelihood requires computing a log-determinant of the kernel matrix and a quadratic term. 
In the following we present upper and lower bounds for both the log-determinant ($\U_\text{D}$ and $\L_\text{D}$, respectively) and the quadratic term ($\U_\text{Q}$ and $\L_\text{Q}$).
We will need the posterior equations for the observations, \ie, $p(y_n \given \vec y_{:n})$, and we will need them as functions of test inputs $\vec x_*$ and $\vec x_{*}'$. 
To this end, define
\begin{align*}
	\vm_*^{(n)}(\vec x_*)&\ce k(\vec x_*, \mat X_{:n})\inv{\mat K_{:n, :n}}\vec y_{:n} \\
	\intertext{and}
	\begin{split}
	\C_*^{(n)}(\vec x_*, \vec x_*')&\ce k(\vec x_*, \vec x_*')+\sigma^2\delta_{\vec x_*,\vec x_*'}\\
																 &\phantom{\ce\ } -k(\vec x_*, \mat X_{:n})\inv{\mat K_{:n,:n}}k(\mat X_{:n}, \vec x_*'),
	\end{split}
\end{align*}
such that $p(y_n \given \vec y_{:n})=\N(y_n; \vm_*^{(n)}(\vec x_n), \C_*^{(n)}(\vec x_n, \vec x_n))$, which allows us to rewrite \cref{eq:log_marginal_alternative} as
\begin{align}
	\label{eq:log_marginal_elementwise}
	\begin{split}
		\log p(\vec y)&\propto
		\sum_{n=1}^N \log \C_*^{(n-1)}(\vec x_{n}, \vec x_{n})\\
									&\quad +\sum_{n=1}^N \frac{(y_n-\vm_*^{(n-1)}(\vec x_{n}))^2}{\C_*^{(n-1)}(\vec x_n, \vec x_n)}\, .
	\end{split}
\end{align}
This reveals that the log-determinant can be written as a sum of posterior variances and
the quadratic form has an expression as normalized square errors.
Other key ingredients for our bounds are estimates for average posterior variance and average covariance.
Therefore define the shorthands
\begin{align*}
\mat V&\ce \Diag \left[ \mat \Sigma_*^{(\tzero)}(\mat X_{\tzero:M}, \mat X_{\tzero:M}) \right] \\
\intertext{and}
\mat C&\ce \sum_{i=1}^{\frac{M}{2}} \mat \Sigma_*^{(\tzero)}(\vec x_{\tzero+2i}, \vec x_{\tzero+2i-1})\vec e_{2i}\vec e_{2i}\Trans\ ,
\end{align*}
where $\vec e_j \in \Re^{m}$ is the $j$-th standard basis vector.
The matrix $\mat V$ is simply the diagonal of the posterior covariance matrix $\mat\Sigma_*$.
The matrix $\mat C$ consists of every \emph{second} entry of the first off-diagonal of $\mat \Sigma_*$.
These elements are placed on the diagonal with every second element being $0$.
The reason for taking every second element is of theoretical nature, see \cref{remark:estimated_correlation} in the appendix.

\subsubsection{Bounds on the log-determinant}
Both bounds, lower and upper, use that
$\log\det{\mat K}=\log\det{\mat K_{:\tzero,:\tzero}}+\log\det{\mat \Sigma_*^{(\tzero)}(\mat X_{\tzero:}, \mat X_{\tzero:})}$
which follows from the matrix-determinant lemma.
The first term is available from the already processed datapoints.
It is the second addend that needs to be estimated, which we approach from the perspective of \cref{eq:log_marginal_elementwise}.
It is
well-established
that, for a fixed input, more observations decrease the posterior variance, and this decrease cannot cross the threshold $\sigma^2$ \citep[Question 2.9.4]{rasmussenwilliams}.
This remains true when taking the expectation over the input.
Hence, the average of the posterior variances for inputs $\mat X_{\tzero:M}$ is with high probability an overestimate of the average posterior variance for inputs with higher index.
This motivates our upper bound on the log-determinant:
\begin{align}
	\U_\text{D} &= \log\det{\mat K_{:\tzero,:\tzero}} + (N-\tzero)\mu_D, \label{eq:bound_ud}
	\\\mu_D&\ce \frac{1}{\blocksize} \sum_{i=1}^\blocksize \log\left(\mat V_{ii}\right). \notag
	\eqcomment{average log posterior variance}
\end{align}
To arrive at the lower bound on the log-determinant, we need an
expression for
how fast the average posterior variance could decrease which is governed by the covariance between inputs.
The variable $\rho_D$ measures the average covariance, and we show in \cref{thm:log_det_lower_bound} in the appendix that this overestimates the decrease per step with high probability.
Since the decrease cannot exceed $\sigma^2$, we introduce $\psi_D$ to denote the step which would cross this threshold.
\begin{align}
	\begin{split}
	\L_\text{D} &= \log\det{\mat K_{:\tzero,:\tzero}}
	+(N-\psi_D)\log \sigma^2\\
	&\quad
	+(\psi_D-\tzero)\left(\mu_D-\frac{\psi_D-\tzero-1}{2}\rho_D\right)\label{eq:bound_ld}
	\end{split}
	\\\rho_D&\ce \frac{2}{\blocksize\sigma^4}\sum_{i=1}^{\blocksize}\mat C^2_{2i,2i} \notag
	\eqcomment{average square covariance}
	\\\psi_D&\ce \max\left(N, \tzero+\left\lfloor\frac{\tilde{\mu}_D-\log\sigma^2}{\tilde{\rho}_D}+\frac{1}{2}\right\rfloor\right)\label{eq:variance_cutoff_step}
	\leqcomment{steps $\mu_D$ can decrease by $\rho$}
\end{align}
where variables with a tilde refer to a preceding estimate, that is, exchanging the indices $M$ for $M-\blocksize$ and $\tzero$ for $\tzero-\blocksize$.
\kscmt{Instead of $\psi_D$, could we use something like $N_{\sigma^2}$, $n_{\sigma^2}$ or $n_\rho$?}
Both bounds collapse to the exact solution when $\tzero=N$.
The bounds are close when the average covariance between inputs, $\rho_D$, is small.
This occurs for example when the average variance is close to $\sigma^2$ since the variance is an upper bound to the covariance. 
Another case where $\rho_D$ is small is when points are not correlated to begin with.


\subsubsection{Bounds on the quadratic term}
Denote with $\res_{*}\ce \vec y_{\tzero:}-\vec m_*^{(\tzero)}(\mat X_{\tzero:})$ the prediction errors (the residuals), when considering the first $\tzero$ points as training set and the remaining inputs as test set.
Analogous to the bounds on the log-determinant, one can show with the matrix inversion lemma that
$\vec y\Trans\inv{\mat K}\vec y=\vec y_{:\tzero}\Trans\inv{\mat K_{:\tzero, :\tzero}}\vec y_{:\tzero}+\res_{*}\Trans\inv{(\mat \Sigma_*^{(\tzero)}(\mat X_{\tzero:}))}\res_{*}$. 
Again, the first term will turn out to be already computed. 
With a slight abuse of notation let $\res_{*}\ce \vec y_{\tzero:M}-\vec m_*^{(\tzero)}(\mat X_{\tzero:M})$, that is, we consider only the first $\blocksize$ entries.
Our lower bound arises from another well-known lower bound: $\vec a\Trans\inv{\mat A}\vec a\geq 2\vec a\Trans\vec b-\vec b\Trans\mat A\vec b$ for all $\vec b$ (see for example \citet{kim2018scalableStructureDiscovery,artemev2021cglb}).
We write $\vec a\Trans \inv{\mat A}\vec a$ as $\vec a\Trans\Diag[\vec a]\inv{\left(\Diag[\vec a]\mat A\Diag[\vec a]\right)}\Diag[\vec a]\vec a$ and choose $\vec b\ce\inv{\Diag[\mat A]}\vec 1$.
The result, after some cancellations, is the following probabilistic lower bound on the quadratic term: 
\begin{align}
	\L_\text{Q} &= \vec y_{:\tzero}\Trans\inv{\mat K_{:\tzero,:\tzero}}\vec y_{:\tzero} 
    + (N-\tzero)\left(\mu_Q-\max(0, \rho_Q)\right) \label{eq:bound_lq}
	\\\mu_Q&\ce \frac{1}{m}\res_{*}\Trans\inv{\mat V} \res_{*} \notag
	\eqcomment{average calibrated square error}\\
	\begin{split}
		\rho_Q&\ce \frac{N-\tzero-1}{2m} \\
					&\phantom{\ce\ }\cdot\sum_{j=\frac{\tzero+2}{2}}^{\frac{M}{2}}\frac{\res_{*,2j}\res_{*,2j-1}\C_*^{(\tzero)}(\vec x_{2j}, \vec x_{2j-1})}{\C_*^{(\tzero)}(\vec x_{2j}, \vec x_{2j})\C_*^{(\tzero)}(\vec x_{2j-1}, \vec x_{2j-1})} \notag
	\end{split}
	\leqcomment{calibrated error correlation}
\end{align}

Our upper bound arises from the element-wise perspective of \cref{eq:log_marginal_elementwise}.
We assume that the expected mean square error $(y_n-\vm_*^{(n-1)}(\vec x_n))^2$ decreases with more observations.
However, though mean square error and variance decrease, their expected ratio may increase or decrease depending on the choice of kernel, dataset and number of processed points.
Using the average error calibration with a correction for the decreasing variance, we arrive at our upper bound on the quadratic term:
\begin{align}
	\U_\text{Q} &= \vec y_{:\tzero}\Trans\inv{\mat K_{:\tzero,:\tzero}}\vec y_{:\tzero} 
	+ (N-\tzero)\left(\mu_Q + \rho_Q'\right) \label{eq:bound_uq}
	\\\rho_Q'&\ce \frac{N-\tzero-1}{m}\frac{1}{\sigma^4}\res_{*}\Trans\mat C\inv{\mat V}\mat C\res_{*} \notag
	\leqcomment{square error correlation}
\end{align}
In the appendix (\cref{thm:quad_upper_bound}), we present a tighter bound which uses a similar construction as for the lower bound on the log-determinant, switching the form at a step $\psi$.
Again, the bounds collapse to the true quantity when $\tzero=N$.
The bounds will give good estimates when the average covariance between inputs is low or when the model can predict new data well, that is, when $\res_{*}$ is close to $0$.

\subsection{Validity of bounds and stopping condition}
\label{sec:main_theoretical_result}
For the upper bound on the quadratic form, we need to make a (technical) assumption.
It expresses the intuition that the (expected) mean square error should not increase with more data---a model should not become worse as its training set increases. %
It is possible to construct counter-examples where this assumption is violated: for example when $\vec y\sim\N(\vec 0, \mat I)$ and $p(\vec f)=\N(\vec 0, \mat K)$, the posterior mean is with high probability no longer zero-mean.
However, our experiments in \cref{sec:experiments} indicate that this assumption is not problematic in practice.
\begin{assumption}
	\label{assumption:targets}
	Assume that
	\begin{align*}
		\Exp\left[f(\vec x, \vec x')(y_j-\vec m_{*}^{(j-1)}(\vec x))^2\mid \mat X_{:\tzero}, \vec y_{:\tzero}\right]
		\\\quad\leq
		\Exp\left[f(\vec x, \vec x')(y_j-\vec m_{*}^{(\tzero)}(\vec x))^2\mid \mat X_{:\tzero}, \vec y_{:\tzero}\right]
	\end{align*}
	for all $\tzero\in\{1, \dots, N\}$ and for all $\tzero <j\leq N$, where $f(\vec x, \vec x')$ is either $\frac{1}{\mat\Sigma_*^{(\tzero)}(\vec x, \vec x)}$ or $\frac{\mat\Sigma_*^{(\tzero)}(\vec x, \vec x')^2}{\sigma^4 \mat\Sigma_*^{(\tzero)}(\vec x, \vec x)}$.
\end{assumption}

\begin{theorem}
\label{thm:main}
Assume that $(\vec x_1, y_1), \dots, (\vec x_N, y_N)$ are independent and identically distributed and that \cref{assumption:targets} holds.
For any $\tzero \in \{1, \dots, N\}$, the bounds defined in \cref{eq:bound_ld,eq:bound_ud,eq:bound_uq,eq:bound_lq} hold in expectation:
\begin{align*}
	\begin{split}
		\Exp[\L_D\mid \mat X_{:\tzero}, \vec y_{:\tzero}]&\leq \Exp[\log\det{\mat K} \mid \mat X_{:\tzero}, \vec y_{:\tzero}]\\
																										 &\leq \Exp[\U_D\mid \mat X_{:\tzero}, \vec y_{:\tzero}] 
	\end{split}\\
	\intertext{and}
	\begin{split}
		\Exp[\L_Q\mid \mat X_{:\tzero}, \vec y_{:\tzero}]&\leq \Exp[\vec y\Trans\inv{\mat K}\vec y \mid \mat X_{:\tzero}, \vec y_{:\tzero}]\\
																										 &\leq \Exp[\U_Q\mid \mat X_{:\tzero}, \vec y_{:\tzero}]\ .
	\end{split}
\end{align*}
\end{theorem}
The proof can be found in \cref{sec:main_theorem_proof}, and a sketch in \cref{sec:proof_sketch}.

\begin{theorem}
\label{thm:guaranteed_precision}
Let $r>0$ be a desired relative error and set $\U\ce -\frac{1}{2}\left(\L_D+\L_Q+N\log2\pi\right)$ and $\L\ce -\frac{1}{2}\left(\U_D+\U_Q+N\log2\pi\right)$.
If the stopping conditions hold, that is, $\sign(\U)=\sign(\L)$ and \cref{eq:stop_cond_r} is true, then $\log p(\vec y)$ can be estimated from $(\U+\L)/2$ such that, under the condition $\L_D\leq \log(\det{\mat K}) \leq \U_D \text{ and }
\L_Q \leq \vec y\Trans\inv{\mat K}\vec y \leq \U_Q$, the relative error is smaller than $r$, formally:
\begin{align}
	\label{eq:main_result}
	\left|{\log p(\vec y)-(\U+\L)/2}\right| \leq r|{\log p(\vec y)}|.
\end{align}

\end{theorem}
The proof follows from \cref{lemma:rel_err_bound} in the appendix.

\cref{thm:main} is a first step to obtain a probabilistic statement for \cref{eq:main_result}, that is, a statement of the form $\proba{\left|\frac{\log p(\vec y)-\frac{1}{2}(\U+\epsilon_{\U,\delta}+\L-\epsilon_{\L,\delta})}{\log p(\vec y)}\right| > r} \leq \delta$.
In earlier work \citep{bartels2021stoppedcholesky}, we have shown that such a statement can be obtained for the log-determinant.
Theoretically, we can obtain such a statement using standard concentration inequalities and a union bound over $\tzero$.
In practice, the error guarding constants $\epsilon$ would render the result trivial.
A union bound can be avoided using Hoeffding's inequality for martingales \citep{fan2012hoeffding}.
However, this requires to replace $\tzero\ce M-m$ by a stopping time independent of $M$, which we regard as future work.

\subsection{Practical implementation}
\label{sec:bound_realization}
The proposed bounds turn out to be surprisingly cheap to compute.
If we set the block-size of the Cholesky decomposition to be $m$, the matrix $\mat\Sigma_*^{(\tzero)}$ is exactly the downdated matrix in Step~2 of the
algorithm outlined in \cref{sec:cholesky}.
Similarly, the expressions for the bounds on the quadratic form appear while solving the linear equation system $\inv{\mat L}\vec y$.
A slight modification to the Cholesky algorithm is enough to compute these bounds on the fly during the decomposition with little overhead.

The stopping conditions can be checked
before or after Step~3 of the Cholesky decomposition (\cref{sec:cholesky}).
Here, we explore the
former
option since Step 3 is the bottleneck due to being less parallelizable than the other steps.

Note that the definition of the bounds does not involve variables $\vec x, y$ which have not been processed.
This allows an on-the-fly construction of the kernel matrix, avoiding potentially expensive kernel function evaluations.
Furthermore, it is \emph{not} necessary to allocate $\O(N^2)$ memory in advance; a user can specify a maximal amount of processed datapoints, hoping that stopping occurs before hitting that limit.
We provide the pseudo-code for this modified algorithm, our key algorithmic contribution, in \cref{sec:proof_sketch} (\cref{algo:chol_blocked,alg:bounds}).
\sbchg{
For technical reasons, the bounds we use in practice, deviate in some places from the ones presented.
We describe the details fully in \cref{sec:stopping}.
}%
Additionally, we provide a \textsc{Python} implementation of our modified Cholesky decomposition and scripts to
replicate the experiments of this paper.\footnote{The code is available at the following repository: \url{https://github.com/SimonBartels/acgp}}
%
		%
		\section{Experiments}
\label{sec:experiments}
We now examine the bounds and stopping strategy for \acgp{}{}. 
When running experiments without GPU support, all linear algebra operations are substituted for direct calls to the \textsc{OpenBLAS} library \citep{wang2013OpenBLAS}, for efficient realization of \textit{in-place} operations. 
To still benefit from automatic differentiation, we used \textsc{PyTorch} \citep{paszke2019pytorch} with a custom backward function for $\log p(\vec y)$ which wraps \textsc{OpenBLAS}.
The details of our experimental setup can be found in \cref{sec:app_experimental_details}.

\subsection{Performance on synthetic data}
\label{sec:synthetic_experiment}

\acgp{} will stop the computation when the posterior covariance matrix of the remaining points conditioned on the processed points is essentially diagonal.
This scenario occurs for example when using a squared exponential kernel with long lengthscale and small observational noise on densely sampled dataset.

To test \acgp in this scenario, we sample a function from a GP prior with zero mean and a squared exponential kernel with length scale $\log\ell=-2$. From this function, we uniformly sample $10^{12}$ observations $(\vx, y)$ in the interval $[0,1]$ using an observation noise of $\sigma^2\ce 0.1$, that is, $\vy = f(\vx) + \N(0, 0.1)$, see \cref{fig:visualization}.
This is a scenario where \acgp excels, since it does not need to load the dataset into memory in advance, whereas methods with at least linear complexity cannot even start computation.


\begin{figure}
	\centering
	\includegraphics[width=\columnwidth]{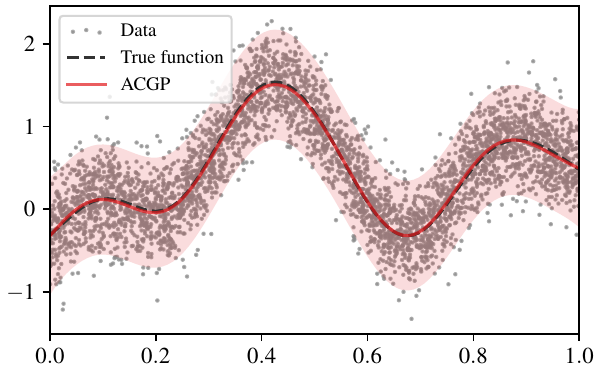}
	\caption{
			The figure shows one of the sampled functions from the synthetic experiment in \cref{sec:synthetic_experiment} as well as the posterior predictive distribution recovered by \acgp. Since the entire dataset of $10^{12}$ observations is too large to visualize, we show only the data that were selected by \acgp before stopping; in this case just $4000$. The relative error on $\log p(\vy)$ for $10^4$ observations was $0.054$. Notice how, despite the larger relative error, the posterior process mean closely follows the actual underlying function.
%
	}
	\label{fig:visualization}
\end{figure}

The task is to estimate the true $\log p(\vy)$ of the full dataset, and we run \acgp{} with a relative error of $r=0.01$ and a blocksize of $1000$ to obtain this estimate. 
Since we cannot evaluate the actual $\log p(\vy)$ for all $10^{12}$ observations, we use the predicted and actual $\log p(\vy)$ for $10^4$ observations as proxy for assessing the performance of \acgp{}. 
We repeat the experiment for 10 different random seeds.
Recall that \acgp{} estimates $\Exp[\log p(\vec y)\given \mat X_{:\tzero}, \vec y_{:\tzero}]$ as opposed to $\log p(\vec y)$, directly. 
Hence, there are two sources of error for \acgp{}: the deviation of $\log p(\vec y)$ from its expected value and the deviation of the empirical estimates from their expectations.\footnote{This shows the benefit of developing our theory further, to obtain probably-approximately-correct bounds. Such bounds introduce error-guarding constants to protect against fluctuations.}
The average $\log p(\vec y_{:10^4})$ is $-2699.67\pm70.81$, and thus, due to the relative variance, a relative error of $r=0.01$ will be hard to achieve.
When run on all $10^{12}$ observations, \acgp{} stops after processing just $4600 \pm 1562$ on average, obtaining an actual relative error on the estimate of $\log p(\vy)$ of $0.047 \pm 0.034$. To decrease this error, one can either decrease the specified relative error of \acgp or increase the blocksize, which will lead to more stable predictions. For the experiments in the remainder of this paper, we choose the latter strategy and set the blocksize to $10^4$, which is also better suited for parallel computations.
\setlength{\figwidth}{.5\textwidth}%
\setlength{\figheight}{.2\textheight}%
\subsection{Bound quality}
\label{sec:bound_experiments}
The purpose of this section is to demonstrate that with a large enough blocksize $\blocksize{}$, our estimates are often correct on large datasets.
We examine our bounds presented in \cref{sec:methods} and compare them to those proposed by \citet[Lemma 2 and Lemma 3]{artemev2021cglb} (CGLB).
Specifically, for the determinant we compare to their $\O(N)$ upper bound \citep[Eq.~11]{artemev2021cglb} and their $\log(\det{\mat Q})$ as lower bound.
We set the number of inducing inputs $M$ for \cglb{} to 512, 1024, 2048, and 4096.
For \acgp{}, we define $\blocksize{}\ce 40\cdot 256=10\,240$ which is the number of cores times the default \textsc{OpenBLAS} blocksize for our machines.
We compare both methods using squared exponential kernel (SE) and the Ornstein-Uhlenbeck kernel (OU),
\begin{align}
	k_\text{SE}(\vec x, \vec z)&\ce \theta \exp\left(-\frac{\norm{\vec x - \vec z}^2}{2\ell^2}\right), \label{eq:kernel_se}\\
	\qquad k_\text{OU}(\vec x, \vec z)&\ce \theta \exp\left(-\frac{\norm{\vec x - \vec z}}{\ell}\right), \label{eq:kernel_ou}
\end{align}
where we fix $\sigma^2\ce 10^{-3}$ and $\theta\ce 1$, and we vary $\ell$ as $\log\ell\in \{-1,0,1,2\}$.
As benchmarking datasets we use the two datasets consisting of more than 20\,000 instances used by \citet{artemev2021cglb}: \texttt{kin40k} and \texttt{protein}.
We further consider two additional datasets from the UCI repository \citep{dua2019uci}: \texttt{metro}
and \texttt{pm25} \citep{liang2015pmDataset}.
We chose these datasets in addition as they are of similar size, they are marked as regression tasks and without missing values.
\sbchg{
We note here that shuffling the datasets does not exactly establish the i.i.d.~assumption of \cref{thm:main}.
In practice, the results of this section demonstrate that \acgp{} performs satisfactorily also in the sampling-without-replacement case.
}%
\newcommand{\bex}{metroOU0.0} 
\begin{figure}[htb]
	\centering
	\includegraphics[width=\linewidth]{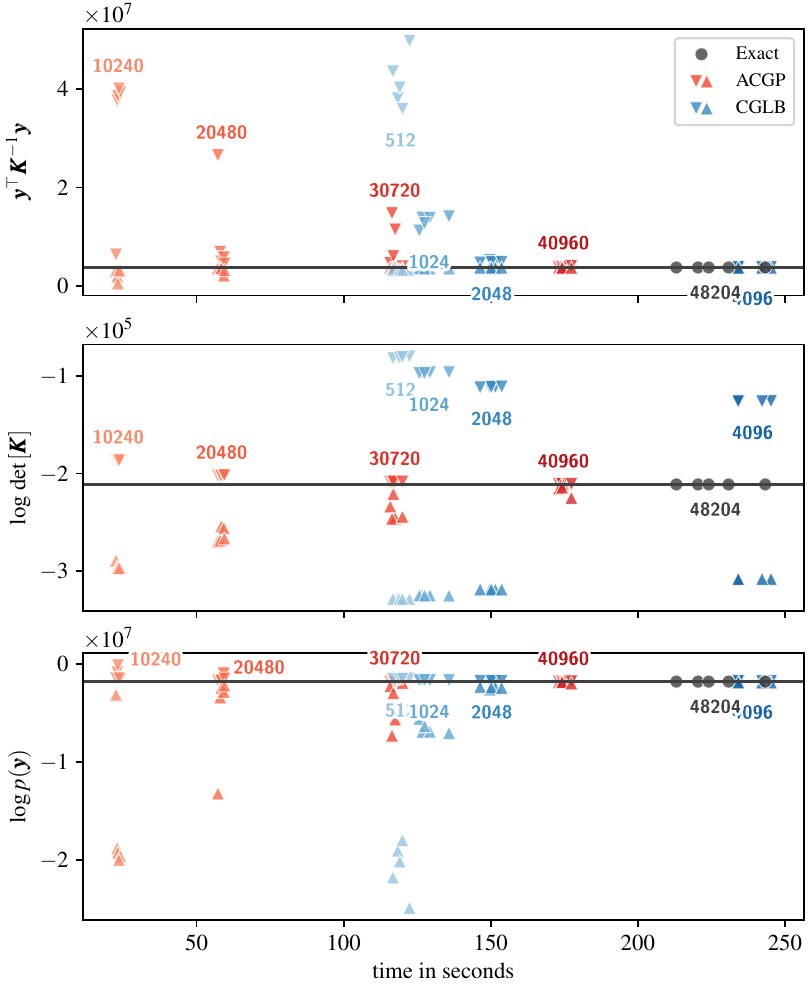}
\caption{
	\revm{Comparison of the upper and lower bounds for ACGP and CGLB on the \texttt{metro} dataset using the OU kernel with a length scale of $\log\ell=0$ and the time it takes to compute them.}
	The black line indicates the result obtained using exact GP regression with points above {and below it marking the upper and lower bounds, respectively.}
	The experiment was repeated five times with different seeds to illustrate the variability in the computation time, shown here as multiple points \revm{of the same color}.
	For \acgp{} the number near the points shows $M$, the size of the used subset; for \cglb{} it is the number of inducing inputs.
}
\label{fig:bound_quality}
\end{figure}

Empirically, \cglb{} seems
to better estimate
the quadratic term, whereas \acgp{} is faster to identify the log-determinant.
\Cref{fig:bound_quality} shows a typical example.
Note that, 
for the quadratic form, the upper bounds tend to be less tight than the lower bounds.
Generally,
there is no clear winner; sometimes \acgp{} estimates both quantities faster and sometimes \cglb{}.
See \cref{sec:app_additional_results} for figures on all results.

The reason why \cglb{} has more difficulties to approximate the log-determinant is that the bound involves $\operatorname{trace}[\mat K-\mat Q]$ where $\mat Q$ is a low rank approximation to $\mat K$.
If $\bm{K}_{\text{ff}}$ is of high rank, the gap in the trace can be large.
For \cglb{} the time to compute the bounds is dominated by the pivoted Cholesky decomposition to select the inducing inputs.
This overhead becomes irrelevant for the following hyper-parameter tuning experiments, since the selection is computed only once in the beginning.
One conclusion from these experiments is to keep in mind that when high precision is required, simply computing the exact solution can be a hard-to-beat baseline.
%
\setlength{\figwidth}{.5\textwidth}%
\setlength{\figheight}{.2\textheight}%
\subsection{Application in hyper-parameter tuning}
\label{sec:hyperparameter_experiments}
We repeat the hyper-parameter tuning experiments performed by \citet{artemev2021cglb} using the same set-up, see \cref{sec:app_experimental_details} for details. 
We use the same kernel function, a Mat\'ern$\frac{3}{2}$, and the same optimizer, 
L-BFGS-B \citep{liu1989LFBGSB}, with \textsc{SciPy} \citep{Virtanen2020SciPy} default parameters.
\citet{artemev2021cglb} report their best results using $M=2048$ inducing inputs.
For reference, we also compare against Sparse Variational Gaussian process regression (SGPR) by \citet{titsias2009variational} initialized with the same 512, 1024 and 2048 inducing inputs as CGLB.
We use root mean square error (RMSE), negative log predictive density (NLPD) and exact, marginal log-likelihood on the training set, $\log p(\vec y)$, as performance metrics.
The results for all experiments discussed in this section can be found in \cref{subsec:hyperparameters}. Here, we will focus on the behavior of each method during training.

A possible application of \acgp{} is that an optimizer can decide how precise function evaluations need to be.
To explore this possibility, we successively decrease the ``relative change in function value'' (\texttt{ftol}) convergence criterion of L-BFGS-B as $(2/3)^{\text{restart}+1}$ and set this as value for $r$.
With this choice, \acgp does not have any more free parameters than a standard optimizer.
The blocksize is a problem independent parameter and it is set to the same value as in \cref{sec:bound_experiments}.

We explore two different computing environments.
For datasets smaller than 20\,000 data points, we ran our experiments on a single GPU.
The results can be summarized in one paragraph:
all methods converge the latest after two minutes.
The time difference between methods is less than twenty seconds.
Exact Gaussian process regression is fastest, more often than not.
The results can be found in \cref{subsec:hyperparameters}.
We conclude that in an environment with significantly more processing resources than memory, approximation may just cause overhead.


For datasets larger than 20\,000 datapoints, our setup differs from \citet{artemev2021cglb}
in that we use only CPUs on machines where the kernel matrix still fits fully into memory.
On all datasets, \acgp{} is essentially exhibiting the same optimization behavior as the exact Gaussian process regressor, just stretched out.
\acgp{} can provide results faster than exact optimization but may be slower in convergence as \cref{fig:ht_results_faster} shows for the \texttt{protein} dataset.
This observation is as expected.
However, approximation can also hinder fast convergence as \cref{fig:ht_results_slower} reveals on for the \texttt{metro} dataset.
\cglb{} benefits from caching the chosen inducing inputs and reusing the solution from the last solved linear equation system.
The algorithm is faster, though it often plateaus at worse objective function values. 
The results for \texttt{kin40k} are similar to \texttt{protein} and the results for \texttt{pm25} are similar to \texttt{metro}.
These and additional results can be found in \cref{subsec:hyperparameter_plots}.
Again, when the available memory permits, the exact computation is a hard-to-beat baseline. 
However, the Cholesky as a standard numerical routine has been engineered over decades, whereas for the implementations of \cglb{} and \acgp{} there is opportunity for improvement.

\begin{figure}[htb]
\centering
\begin{subfigure}{0.95\columnwidth}%
	\includegraphics[width=\linewidth]{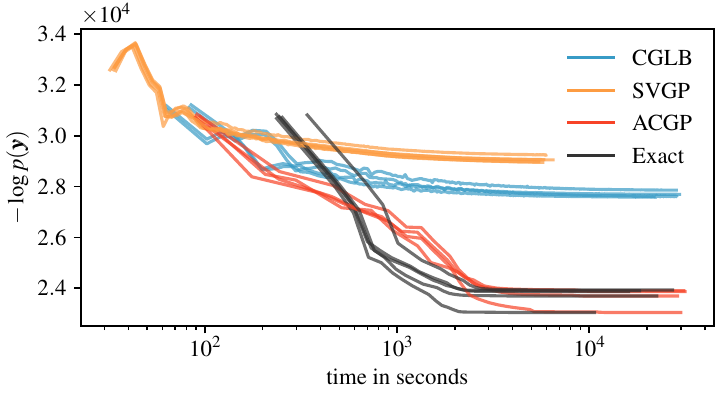}
	\caption{\texttt{protein} dataset. The iteratively increasing precision may allow \acgp{} to reach better solutions faster than exact inference at the price of later convergence.
	}
	\label{fig:ht_results_faster}
\end{subfigure}
\begin{subfigure}{0.95\columnwidth}%
	\includegraphics[width=1.0\linewidth]{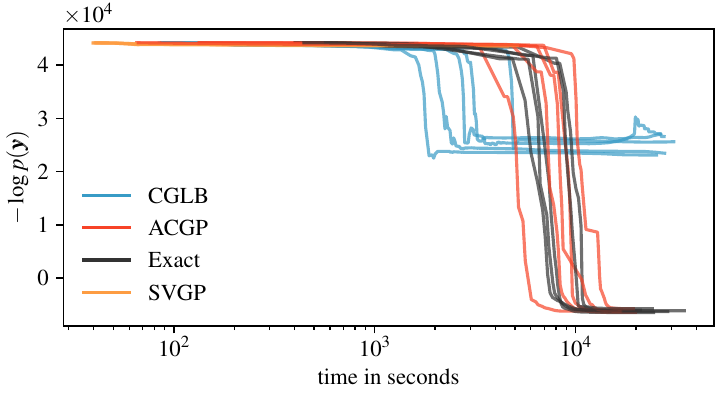}
	\caption{\texttt{metro} dataset. Function evaluations with \cglb{} are generally the fastest at the cost of plateauing at higher objective function values.
	}
	\label{fig:ht_results_slower}
\end{subfigure}
	\caption{
		Typical examples of the evolution of the exact log marginal likelihood $p(\vec y)$ while optimizing hyper-parameters.
		See \cref{subsec:hyperparameters} for additional plots for all datasets, as well as for SVGP runs.
		\label{fig:ht}
	}
\end{figure}

		\section{Conclusions}
The Cholesky decomposition is the de facto way to invert matrices when training Gaussian processes, yet it tends to be considered a black box. However, if one opens this black box, it turns out that the Cholesky decomposition computes the marginal log-likelihood of the full dataset, and, crucially, in intermediate steps, the posteriors of unprocessed training data conditioned on the processed.
Making the community aware of this remarkable insight is one of our main contributions of our paper.
Our main novelty is to use this insight to bound the (expected) marginal log-likelihood of the full dataset from only a subset.
%
%
%
With
only small
modifications to this classic matrix decomposition, we can use these upper and lower bounds to stop the decomposition before all observations have been processed.
This has the practical benefit that the kernel matrix $\mat{K}$ does not have to computed prior to performing the decomposition, but can rather be computed on-the-fly.

Empirical results indicate that the approach carries
significant promise. 
In general, we find that exact GP inference leads to better behaved optimization
than approximations such as \cglb{} and inducing point methods, and that
a well-optimized Cholesky implementation is surprisingly competitive in terms
of performance. 
%
An advantage of our approach is that it is essentially parameter-free. The user has to specify a requested numerical accuracy and the computational demands will be scaled accordingly.
Finally, we note that \acgp{} is complementary to much existing work,
and should be seen as an addition to the GP toolbox, rather than a substitute
for existing tools.

%
%
%

%
		%
		\subsubsection*{Acknowledgements}
			Shout-out to Damien Garreau for a substantial amount of suggestions for this paper.
			Further, we are grateful for the valuable feedback of all anonymous reviewers who saw the different iterations of this article.
			
			This work was funded in part by the Novo Nordisk Foundation through the Center for Basic Machine Learning Research in Life Science (NNF20OC0062606, NNF20OC0065611).
			It also received funding from the European Research Council (ERC) under the European Union’s Horizon 2020 research, innovation programme (757360), from a research grant (15334, 42062) from VILLUM FONDEN, from the Danish Ministry of Education and Science, and from Digital Pilot Hub and Skylab Digital.
			The authors acknowledge the Pioneer Centre for AI, DNRF grant P1.
		%
		%
		\FloatBarrier{}%
		\subsubsection*{References}%
		%
		\printbibliography[heading=none]
		
		\appendix
		\onecolumn
		\renewcommand{\eqcomment}[1]{\\&\sslash\text{{\small\emph{#1}}}\notag}%
		\section{Evolution of the log-marginal likelihood}%
\label{app:evidence}%
This section contains figures for the progression of the log-marginal likelihood for five different permutations of the same datasets as used in \cref{sec:bound_experiments} of the main paper.
\cref{fig:evidence_rbf} shows the results for the squared exponential kernel (\cref{eq:app_kernel_se}) with $\theta\ce 1$ and $\sigma^2\ce 10^{-3}$, and \cref{fig:evidence_ou} shows the results for the Ornstein-Uhlenbeck kernel (\cref{eq:app_kernel_ou}) using the same parameters.
\newcommand{\evidencefig}[3]{
\begin{minipage}[b]{.5\textwidth}
	\centering
	\includegraphics[width=0.95\textwidth]{./figs/llh_progression/llh_#1_tight.pdf}
	\subcaption{
		Log-marginal likelihood evolution for #2.
	}
	\label{fig:evidence_#1}
\end{minipage}
%
}%
\begin{figure}[htb]%
\evidencefig{pm25RBF}{\PM{}}{RBF}%
\evidencefig{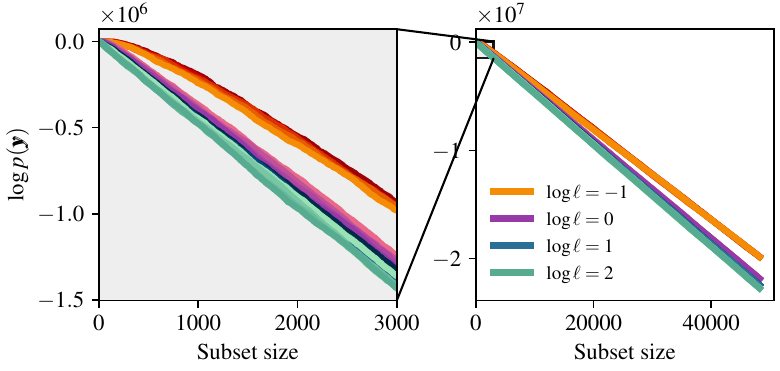}{\METRO{}}{RBF}\\[5mm]%
\evidencefig{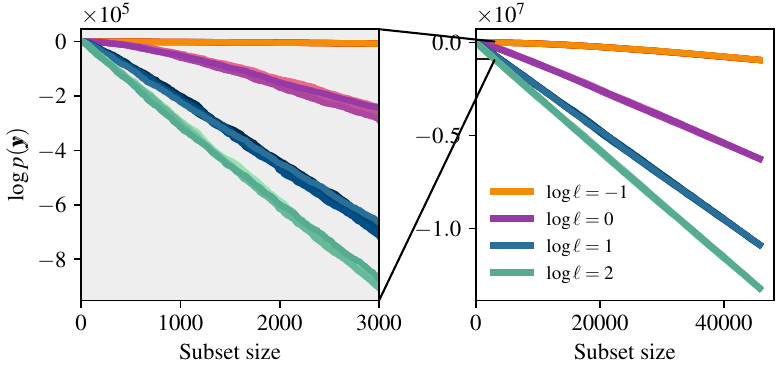}{\PROTEIN{}}{RBF}%
\evidencefig{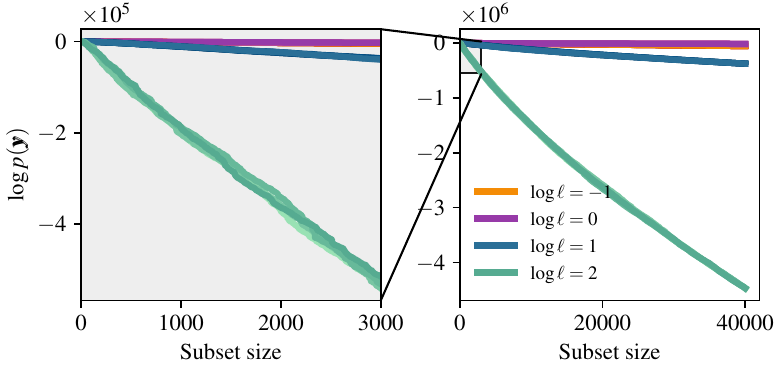}{\KINFORTYK{}}{RBF}%
\caption{
The figure shows the log-marginal likelihood as a function of the size of the training set for the large datasets described in \cref{tbl:datasets} using the squared exponential kernel.
See \cref{app:evidence} for a description of the experimental setup.
}%
\label{fig:evidence_rbf}%
\end{figure}%
\begin{figure}[htb]%
	\evidencefig{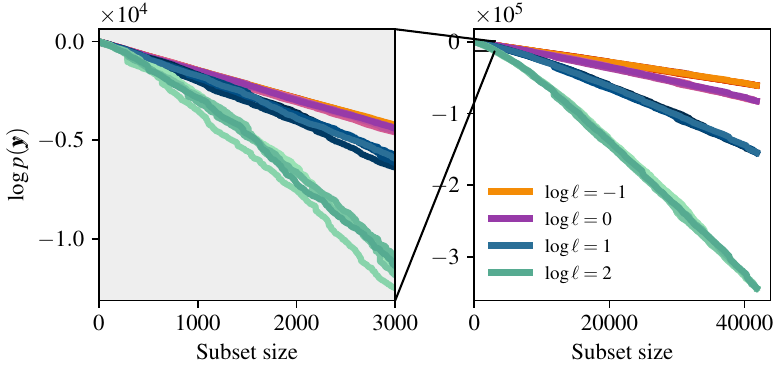}{\PM{}}{OU}%
	\evidencefig{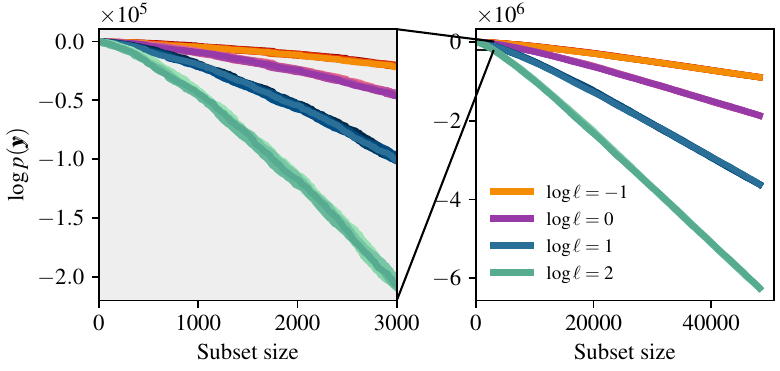}{\METRO{}}{OU}\\[5mm]%
	\evidencefig{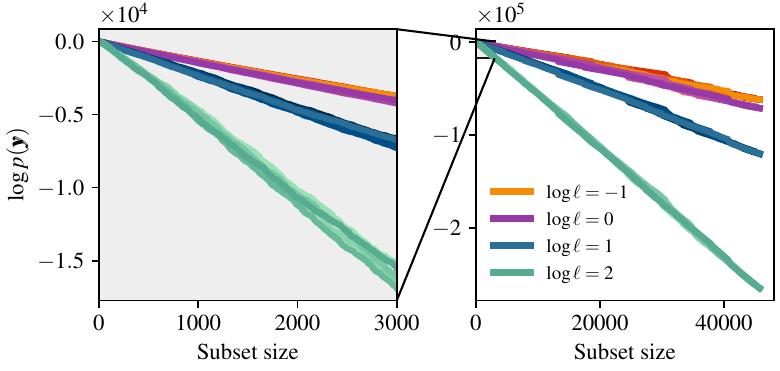}{\PROTEIN{}}{OU}%
	\evidencefig{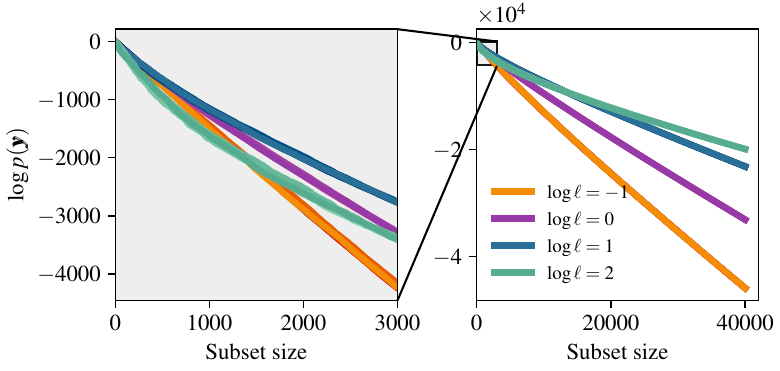}{\KINFORTYK{}}{OU}%
	\caption{The figure shows the log-marginal likelihood as a function of the size of the training set for the large datasets described in \cref{tbl:datasets} using the Ornstein-Uhlenbeck kernel.
	See \cref{app:evidence} for a description of the experimental setup.
	}%
	\label{fig:evidence_ou}
\end{figure}%
\FloatBarrier{}%
\section{Experimental details}
\label{sec:app_experimental_details}

\begin{table}[htb]
	\centering
	\caption{
		Overview over all datasets used for the experiments in \cref{sec:experiments}.
		The total dataset size (training and testing) is denoted $N$ and $D$ denotes the dimensionality.
	}
	\begin{tabular}{lS[table-format=5]S[table-format=2]l} 
		\toprule
		Key & {$N$} & {$D$} & Source 
		\\\midrule
		\BIKE & 17379 & 17 & \citet{Fanaee2013BikeDataset}. \web{http://archive.ics.uci.edu/ml/datasets/bike+sharing+dataset}
		\\\ELEVATORS & 16599 & 18 &  \citet{Camachol1998AileronsElevators}. 
		\\\KINFORTYK & 40000 & 8 & \citet{Schwaighofer2002kin40k}.
		\\\METRO & 48204 & 66 & No citation request. \web{http://archive.ics.uci.edu/ml/datasets/Metro+Interstate+Traffic+Volume} 
		\\\PM & 43824 & 79 & \citet{Liang2015pmDataset}. \web{http://archive.ics.uci.edu/ml/datasets/Beijing+PM2.5+Data}
		\\\POLETELECOMM & 15000 & 26 &  \citet{Weiss1995Poletelecomm}. 
		\\\PROTEIN & 45730 & 9 & No citation request. \web{http://archive.ics.uci.edu/ml/datasets/Physicochemical+Properties+of+Protein+Tertiary+Structure}
		\\\PUMADYN & 8192 & 32 & No citation request. \web[this website]{https://www.cs.toronto.edu/~delve/data/pumadyn/desc.html} 
		\\\bottomrule
	\end{tabular}
	\label{tbl:datasets}
\end{table}


For an overview of the datasets we use, see~\cref{tbl:datasets}.
The datasets are all normalized to have zero mean and unit variance for each feature.
We explore two different computing environments.
For datasets smaller than \num{20000} data points, we ran our experiments on a single GPU.
This is the same setup as in \citet{artemev2021cglb} with the difference that we use a \textsc{Titan RTX} whereas they have used a \textsc{Tesla V100}.
For datasets larger than \num{20000} datapoints, our setup differs from \citet{artemev2021cglb}.
We use only CPUs on machines where the kernel matrix still fits fully into memory.
Specifically, we used machines running Ubuntu 18.04 with 50 Gigabytes of RAM and two \textsc{Intel Xeon E5-2670 v2} CPUs.

%

\subsection{Bound quality experiments}
For \cglb, we compute the bounds with varying number of inducing inputs $M \ce \{512, 1024, 2048, 4096\}$ and measure the time it takes to compute the bounds. 
For \acgp, we define the blocksize $\blocksize \ce 256\cdot 40=\num{10192}$ which is the default \openblas block size on our machines times the number of cores.
This ensures that the sample size for our bounds is sufficiently large for accurate estimation, and at the same time the number of page-faults should be comparable to the default Cholesky implementation.
We measure the elapsed time every time a block of data points is added to the processed dataset and the bounds are recomputed.

We compare both methods using squared exponential kernel (SE) and the Ornstein-Uhlenbeck kernel (OU).
\begin{align}
	k_\text{SE}(\vec x, \vec z)&\ce \theta \exp\left(-\frac{\norm{\vec x - \vec z}^2}{2\ell^2}\right) \label{eq:app_kernel_se}
	\\k_\text{OU}(\vec x, \vec z)&\ce \theta \exp\left(-\frac{\norm{\vec x - \vec z}}{\ell}\right). \label{eq:app_kernel_ou}
\end{align}
where we fix $\theta\ce 1$ and we vary $\ell$ as $\log\ell\in \{-1,0,1,2\}$.
We use a Gaussian likelihood and fix the noise to $\sigma^2\ce 10^{-3}$.

\subsection{Hyper-parameter tuning}
In this section, we describe our experimental setup for the hyper-parameter optimization experiments, which closely follows that of \citet{artemev2021cglb}.
We randomly split each dataset into a training set consisting of 2/3 of examples, and a test set consisting of the remaining third.
We use a Mat\'ern$\frac{3}{2}$ kernel function and L-BFGS-B as the optimizer with \textsc{SciPy} \citep{Virtanen2020SciPy} default parameters if not specified otherwise.
All algorithms are stopped the latest after 2000 optimization steps, after 12 hours of compute time, or when optimization has failed three times.
We repeat each experiment five times with a different shuffle of the dataset and report the results in \cref{tbl:results,tbl:results_gpu}.

For \cglb, it is necessary to decide on a number of inducing inputs.
From the results reported by \citet{artemev2021cglb}, it appears that using $M=2048$ inducing inputs yields the best trade-off in terms of speed and performance, hence we use this value in our experiments.
For the exact Cholesky and \cglb, the L-BFGS-B convergence criterion ``relative change in function value'' (\texttt{ftol}) is set to 0.

For \acgp, we need to decide on both the desired relative error, $r$, as well as the block size $m$.
We successively decrease the optimizer's tolerance \texttt{ftol}
as $(2/3)^{\text{restart}+1}$ and we set the same value for $r$.
That is, regardless of whether the optimization of \acgp stopped successfully or for abnormal reasons, the optimization restarts aiming for higher precision.
The effect of this is that, early in the hyper-parameter optimization, \acgp will stop early, thus providing only an approximation to the optimal hyper-parameter values, but also saving computations. With each restart, \acgp increases the precision, ensuring that we get closer and closer to the optimal hyper-parameter values at the expense of approaching the computational demand of an exact GP.
The block size $\blocksize$ is set to the same value as for the bound quality experiments, \cref{sec:bound_experiments}, $40\cdot 256=\num{10192}$, which is the number of cores times the \openblas block size.
This ensures that the sample size for our bounds is sufficiently large for accurate estimation, and at the same time the number of page-faults should be comparable to the default Cholesky implementation.
Note that $\blocksize$ is a global parameter, independent of the dataset.
Hence, natural choices for both $r$ and $\blocksize$ are determined by parameters of standard software, which have sensible, machine-dependent default values. \acgp can therefore be considered parameter-free.

Differing from the previous section, we use for \acgp the biased estimator $(N-M)\log p(\vec y_{:M})/M$ instead of $\U/2+\L/2$ to approximate $\log p(\vec y)$ when stopping.
Since stopping occurs when log-determinant and quadratic form evolve roughly linearly, the two estimators are not far off each other.
The main reason for using the biased estimator is of technical nature: for auto-differentiation, it is easier and faster to implement a custom backward function which can handle the in-place operations of our Cholesky implementation.
This custom backward function needs roughly a factor two of the computation of $\log p(\vec y)$ whereas the \textsc{Torch}-default needs a factor six.
This shows that when comparing to exact inference, auto-differentiation can be disadvantageous and make the Cholesky appear slower than it is.
Regarding \cglb, computation time is not dominated by the gradient but only the function evaluation itself.

\section{Additional results}
\label{sec:app_additional_results}
In this section, we report additional results for both the hyper-parameter tuning experiments (section~\ref{subsec:hyperparameters}) as well as plots to show the quality of the bounds on both the log-determinant term, the quadratic term, and the log-marginal likelihood (see \cref{subsec:bound_evolution}).

\subsection{Additional results for hyper-parameter tuning}
\label{subsec:hyperparameters}

Denote with $N_*$ the number of test instances, and with $\mu$ and $\sigma^2$ the mean and variance approximations of a method.
As performance metrics we use root mean square error (RMSE)
$$\sqrt{\frac{1}{N^*}\sum_{n=1}^{N^*} (y_n^*-\mu(\vec x_n^*))^2}\ ,$$
negative log predictive density (NLPD)
$$\frac{1}{2N^*}\sum_{n=1}^{N^*} \frac{(y_n^*-\mu(\vec x_n^*))^2}{\sigma^2(\vec x_n^*)}+\log\left(2\pi\sigma^2(\vec x_n^*)\right) \ ,$$
and the negative marginal log likelihood $-\log p(\vec y)$.
\cref{tbl:results,tbl:results_gpu} summarize the results reported for each dataset, averaging over the outcomes of the final optimization step of each repetition.
For each metric, we indicate whether a higher ($\uparrow$) or lower ($\downarrow$) value indicates a better result.

The results for the exact GP regression are marked in italics to emphasize that these are results we are trying to approach, not to beat. As the other methods are all approximations to the exact GP, there is little hope of achieving better performance.
The best result among the approximation methods for each dataset is highlighted in bold.

\begin{table}[htb]
    \centering
    \caption{Summary of the CPU hyper-parameter tuning results from \cref{sec:hyperparameter_experiments}. For each metric, we report its final value over the course of optimization. For SVGP, we did not compute the exact marginal log-likelihoods, to save cluster time.
    \label{tbl:results}}
    \sisetup{
    separate-uncertainty,
    retain-zero-uncertainty,
    drop-exponent = true,
    exponent-mode = fixed,
    text-series-to-math = true,
    detect-all = true,
}
\begin{tabular}{ll
S[table-format = 2.2(4), round-precision=2, fixed-exponent=-2]
S[table-format = +2.2(6), round-precision=3, fixed-exponent=-1]
S[table-format = +1.4(5), round-precision=3, fixed-exponent=4]
}
\toprule{\bfseries Dataset} & {\bfseries Model}  & {\bfseries RMSE / $10^{-2}$ ($\downarrow$)} & {\bfseries NLPD / $10^{-1}$ ($\downarrow$)} & {\bfseries $\log p(\vy)$ / $10^{4}$ ($\uparrow$)}\\
\midrule
\multirow{8}{*}{\METRO} & \goldstandard Exact  & \goldstandard     0.2401(26) & \goldstandard     -1.290(17) & \goldstandard       6193(243)\\
 & ACGP  & \bfseries     0.2411(27) & \bfseries     -1.283(23) & \bfseries       6078(250)\\
 & CGLB (1024)  &     0.4645(680) &      0.818(85) &     -24526(1995)\\
 & CGLB (2048)  &     0.4057(794) &      0.745(102) &     -25217(1745)\\
 & CGLB (4096)  &     0.3583(104) &      0.680(24) &     -25193(533)\\
 & SVGP (1024)  &     0.9328(35) &      1.350(4) &     -43892(52)\\
 & SVGP (2048)  &     0.9232(37) &      1.340(4) &     -43701(53)\\
 & SVGP (4096)  &     0.8898(156) &      1.304(17) &     -42950(326)\\
\midrule
\multirow{8}{*}{\PM} & \goldstandard Exact  & \goldstandard     0.4599(268) & \goldstandard      0.363(64) & \goldstandard     -20216(739)\\
 & ACGP  &     0.4445(134) & \bfseries      0.325(40) &     -19438(130)\\
 & CGLB (1024)  & \bfseries     0.3463(291) &      0.560(42) & \bfseries     -19259(909)\\
 & CGLB (2048)  &     0.4430(194) &      0.711(50) &     -23956(867)\\
 & CGLB (4096)  &     0.5306(89) &      0.789(13) &     -25681(322)\\
 & SVGP (1024)  &     0.7208(732) &      1.097(99) &     -31405(2014)\\
 & SVGP (2048)  &     0.5655(96) &      0.870(11) &     -26308(206)\\
 & SVGP (4096)  &     0.5930(1122) &      0.895(166) &     -26904(4079)\\
\midrule
\multirow{8}{*}{\KINFORTYK} & \goldstandard Exact  & \goldstandard     0.0742(7) & \goldstandard     -1.238(5) & \goldstandard      20835(41)\\
 & ACGP  & \bfseries     0.0742(7) & \bfseries     -1.238(5) & \bfseries      20835(41)\\
 & CGLB (1024)  &     0.0917(6) &     -0.713(2) &      14726(58)\\
 & CGLB (2048)  &     0.0868(6) &     -0.827(2) &      16181(48)\\
 & CGLB (4096)  &     0.0846(6) &     -0.918(2) &      16905(51)\\
 & SVGP (1024)  &     0.1394(9) &     -0.420(2) &       6124(46)\\
 & SVGP (2048)  &     0.1205(8) &     -0.568(2) &       8858(47)\\
 & SVGP (4096)  &     0.1068(7) &     -0.702(2) &      11290(43)\\
\midrule
\multirow{8}{*}{\PROTEIN} & \goldstandard Exact  & \goldstandard     0.5585(62) & \goldstandard      0.652(46) & \goldstandard     -23686(331)\\
 & ACGP  & \bfseries     0.5555(17) & \bfseries      0.629(4) & \bfseries     -23847(92)\\
 & CGLB (1024)  &     0.5758(41) &      0.851(5) &     -28194(93)\\
 & CGLB (2048)  &     0.5685(46) &      0.831(6) &     -27689(107)\\
 & CGLB (4096)  &     0.5606(44) &      0.806(6) &     -27070(108)\\
 & SVGP (1024)  &     0.6221(35) &      0.941(4) &     -29994(110)\\
 & SVGP (2048)  &     0.6004(38) &      0.900(5) &     -29049(105)\\
 & SVGP (4096)  &     0.5811(43) &      0.860(6) &     -28133(106)\\
\bottomrule
\end{tabular}
\end{table}
\begin{table}[htb]
    \centering
    \caption{Summary of the GPU hyper-parameter tuning results from \cref{sec:hyperparameter_experiments}. For each metric, we report its final value over the course of optimization. We did not compute the exact marginal log-likelihoods, to save cluster time.
    \label{tbl:results_gpu}}
		\sisetup{
    separate-uncertainty,
    retain-zero-uncertainty,
    drop-exponent = true,
    exponent-mode = fixed,
    text-series-to-math = true,
    detect-all = true,
}
\begin{tabular}{ll
S[table-format = 2.2(4), round-precision=2, fixed-exponent=-2]
S[table-format = +2.2(6), round-precision=3, fixed-exponent=-1]
S[table-format = +1.4(5), round-precision=3, fixed-exponent=4]
}
\toprule{\bfseries Dataset} & {\bfseries Model}  & {\bfseries RMSE / $10^{-2}$ ($\downarrow$)} & {\bfseries NLPD / $10^{-1}$ ($\downarrow$)} & {\bfseries $\log p(\vy)$ / $10^{4}$ ($\uparrow$)}\\
\midrule
\multirow{8}{*}{\BIKE} & \goldstandard Exact  & \goldstandard     0.0009(4) & \goldstandard     -5.032(8) & \goldstandard      49424(73)\\
 & ACGP  & \bfseries     0.0021(10) & \bfseries     -5.031(3) & \bfseries      49321(19)\\
 & CGLB (1024)  &     0.0053(6) &     -3.324(82) &      34946(503)\\
 & CGLB (2048)  &     0.0032(4) &     -3.781(51) &      38572(487)\\
 & CGLB (4096)  &     0.0038(15) &     -4.123(108) &      41049(1105)\\
 & SVGP (1024)  &     0.0127(7) &     -2.663(38) &      27446(447)\\
 & SVGP (2048)  &     0.0093(12) &     -3.048(59) &      30752(516)\\
 & SVGP (4096)  &     0.0090(23) &     -3.240(151) &      31980(1411)\\
\midrule
\multirow{8}{*}{\POLETELECOMM} & \goldstandard Exact  & \goldstandard     0.0813(41) & \goldstandard     -0.781(266) & \goldstandard       8423(874)\\
 & ACGP  & \bfseries     0.0730(23) & \bfseries     -1.232(16) & \bfseries      10149(86)\\
 & CGLB (1024)  &     0.0790(15) &     -1.073(4) &       8705(52)\\
 & CGLB (2048)  &     0.0765(17) &     -1.146(6) &       9238(51)\\
 & CGLB (4096)  &     0.0739(18) &     -1.212(9) &       9822(60)\\
 & SVGP (1024)  &     0.0920(11) &     -0.909(3) &       7378(54)\\
 & SVGP (2048)  &     0.0824(15) &     -1.048(5) &       8450(47)\\
 & SVGP (4096)  &     0.0753(17) &     -1.181(7) &       9543(55)\\
\midrule
\multirow{8}{*}{\ELEVATORS} & \goldstandard Exact  & \goldstandard     0.3512(36) & \goldstandard      0.378(9) & \goldstandard      -4690(48)\\
 & ACGP  & \bfseries     0.3479(24) & \bfseries      0.370(7) & \bfseries      -4653(30)\\
 & CGLB (1024)  &     0.3542(40) &      0.386(10) &      -4714(45)\\
 & CGLB (2048)  &     0.3526(37) &      0.381(10) &      -4699(46)\\
 & CGLB (4096)  &     0.3541(47) &      0.386(13) &      -4703(51)\\
 & SVGP (1024)  &     0.3562(37) &      0.391(9) &      -4746(42)\\
 & SVGP (2048)  &     0.3542(36) &      0.386(9) &      -4715(49)\\
 & SVGP (4096)  &     0.3523(35) &      0.381(9) &      -4697(48)\\
\midrule
\multirow{8}{*}{\PUMADYN} & \goldstandard Exact  & \goldstandard     0.2255(103) & \goldstandard     -0.063(51) & \goldstandard         88(261)\\
 & ACGP  &     0.2285(8) &     -0.045(1) &        -16(17)\\
 & CGLB (1024)  & \bfseries     0.2055(5) & \bfseries     -0.163(3) & \bfseries        559(30)\\
 & CGLB (2048)  &     0.2197(287) &     -0.104(120) &        259(620)\\
 & CGLB (4096)  &     0.2055(6) &     -0.163(3) &        555(46)\\
 & SVGP (1024)  &     0.9865(122) &      1.406(12) &      -7749(0)\\
 & SVGP (2048)  &     0.9865(122) &      1.406(12) &      -7749(0)\\
 & SVGP (4096)  &     0.9865(122) &      1.406(12) &      -7749(0)\\
\bottomrule
\end{tabular}
\end{table}

\clearpage
\subsection{Additional plots for hyper-parameter tuning}
\label{subsec:hyperparameter_plots}
The plots for the hyper-parameter optimization are shown in figures~\ref{fig:hyp_metro_lml}--\ref{fig:hyp_pumadyn_nlpd}. Each point in the plots corresponds to one accepted optimization step for the given methods. Each point thus corresponds to a particular set of hyper-parameters during the optimization.
In figures~\ref{fig:hyp_metro_rmse}--\ref{fig:hyp_pumadyn_rmse}, we show the root-mean-square error, RMSE, that each methods obtains on the test set at each optimisation step, and figures ~\ref{fig:hyp_metro_nlpd}--\ref{fig:hyp_pumadyn_nlpd} show the same for NLPD.
In figures~\ref{fig:hyp_metro_lml}--\ref{fig:hyp_pumadyn_lml}, we show the log-marginal likelihood, $\log p(\vy)$, that an exact GP would have achieved with the specific set of hyper-parameters at each optimization step for each method.

\begin{figure}[htb]
\begin{minipage}[b]{.5\textwidth}
    \centering
    \includegraphics[width=0.9\textwidth]{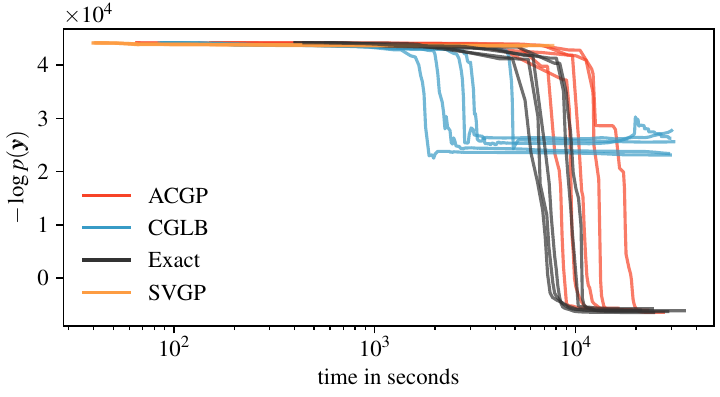}
    \caption{$\log p(\vy)$ for the \texttt{metro} dataset.}
    \label{fig:hyp_metro_lml}
\end{minipage}
\begin{minipage}[b]{.5\textwidth}
    \centering
    \includegraphics[width=0.9\textwidth]{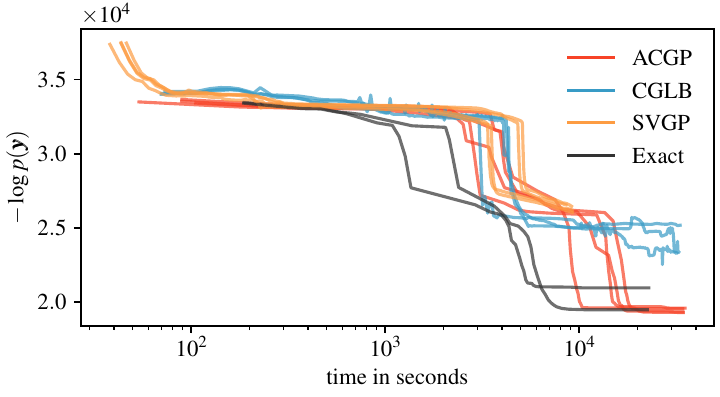}
    \caption{$\log p(\vy)$ for the \texttt{pm25} dataset.}
    \label{fig:hyp_pm25_lml}
\end{minipage}\\[5mm]
\begin{minipage}[b]{.5\textwidth}
    \centering
    \includegraphics[width=0.9\textwidth]{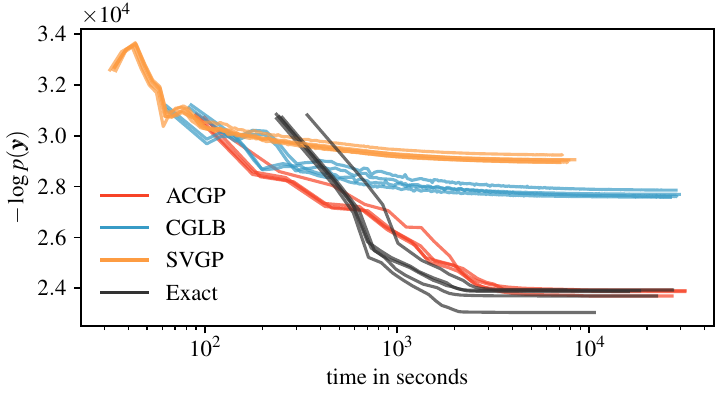}
    \caption{$\log p(\vy)$ for the \texttt{protein} dataset.}
    \label{fig:hyp_protein_lml}
\end{minipage}
\begin{minipage}[b]{.5\textwidth}
    \centering
    \includegraphics[width=0.9\textwidth]{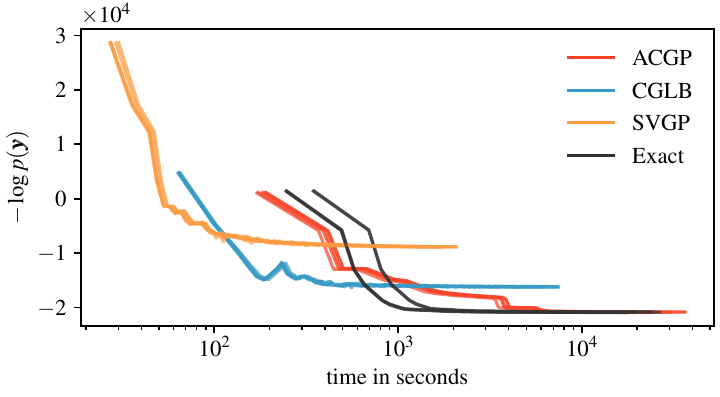}
    \caption{$\log p(\vy)$ for the \texttt{kin40k} dataset.}
    \label{fig:hyp_kin40k_lml}
\end{minipage}\\[5mm]
\begin{minipage}[b]{.5\textwidth}
\centering
\includegraphics[width=0.9\textwidth]{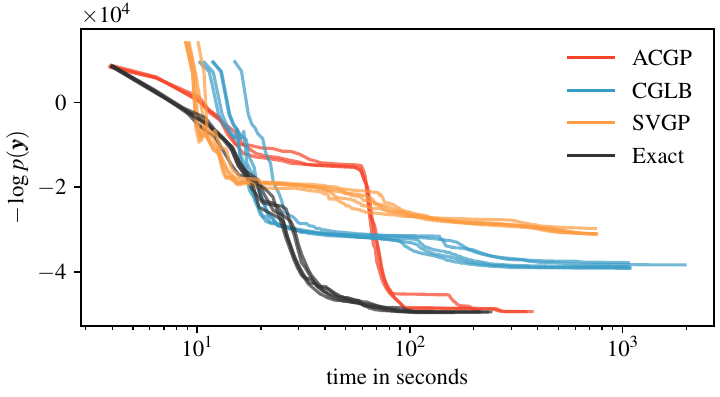}
\caption{$\log p(\vy)$ for the \texttt{bike} dataset.}
\label{fig:hyp_bike_lml}
\end{minipage}
\begin{minipage}[b]{.5\textwidth}
\centering
\includegraphics[width=0.9\textwidth]{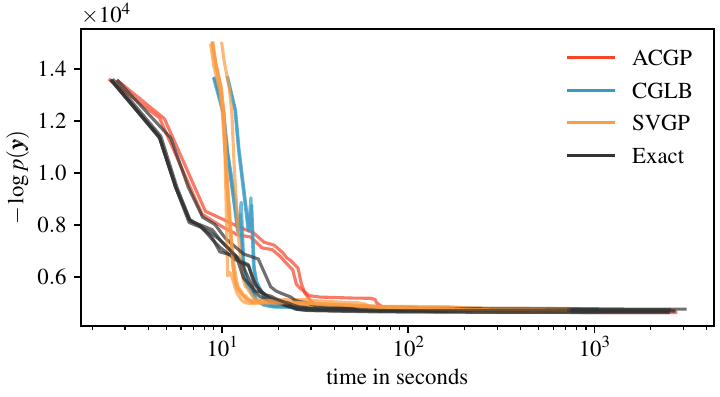}
\caption{$\log p(\vy)$ for the \texttt{elevators} dataset.}
\label{fig:hyp_elevators_lml}
\end{minipage}\\[5mm]
\begin{minipage}[b]{.5\textwidth}
\centering
\includegraphics[width=0.9\textwidth]{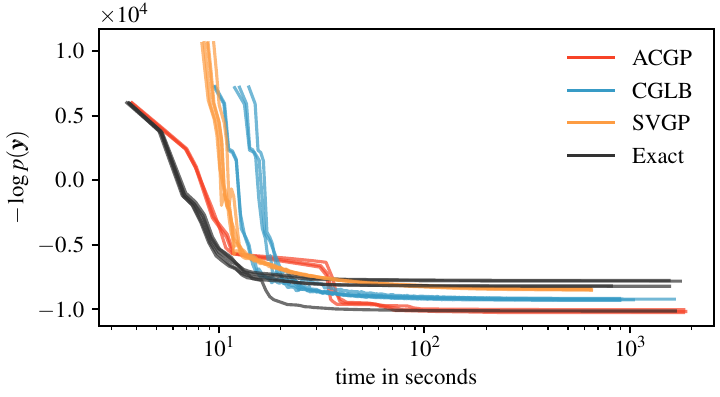}
\caption{$\log p(\vy)$ for the \texttt{pole} dataset.}
\label{fig:hyp_pole_lml}
\end{minipage}
\begin{minipage}[b]{.5\textwidth}
\centering
\includegraphics[width=0.9\textwidth]{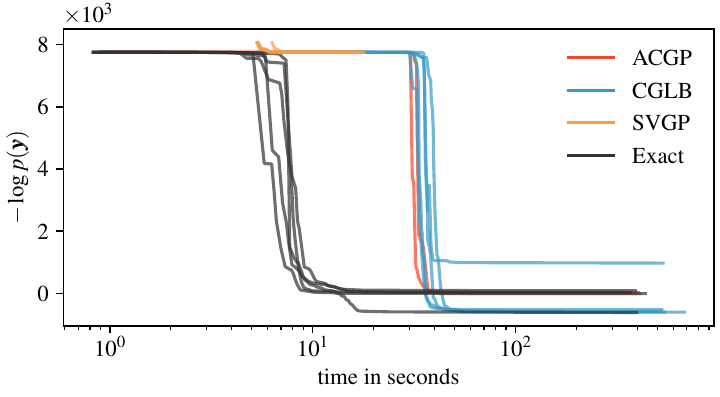}
\caption{$\log p(\vy)$ for the \texttt{pumadyn32nm} dataset.}
\label{fig:hyp_pumadyn_lml}
\end{minipage}
\end{figure}

\begin{figure}[htb]
\begin{minipage}[b]{.5\textwidth}
    \centering
    \includegraphics[width=0.9\textwidth]{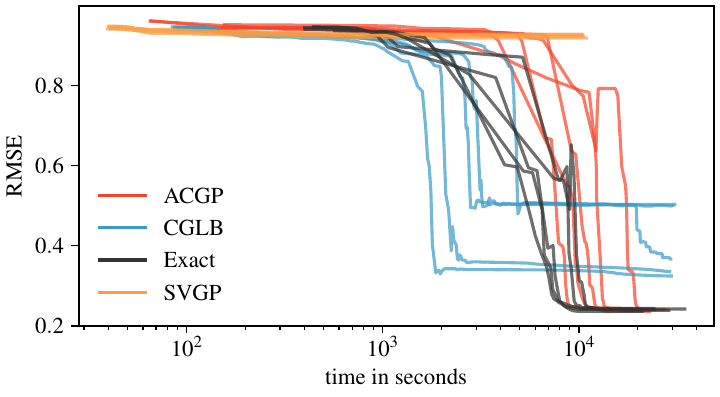}
    \caption{RMSE for the \texttt{metro} dataset.}
    \label{fig:hyp_metro_rmse}
\end{minipage}
\begin{minipage}[b]{.5\textwidth}
    \centering
    \includegraphics[width=0.9\textwidth]{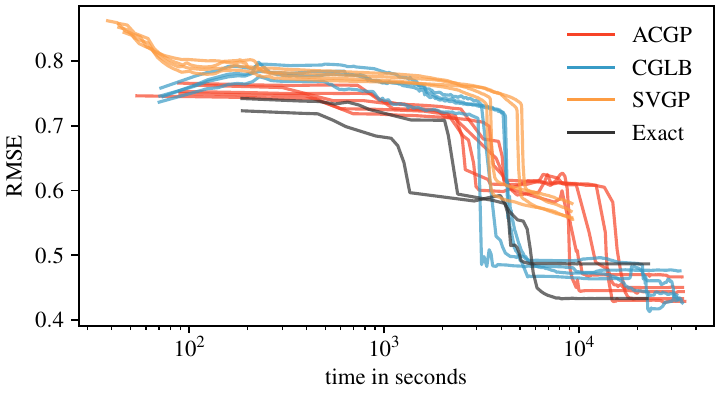}
    \caption{RMSE for the \texttt{pm25} dataset.}
    \label{fig:hyp_pm25_rmse}
\end{minipage}\\[5mm]
\begin{minipage}[b]{.5\textwidth}
    \centering
    \includegraphics[width=0.9\textwidth]{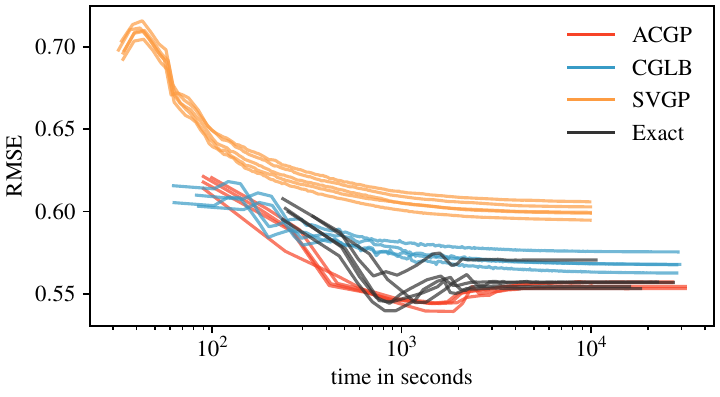}
    \caption{RMSE for the \texttt{protein} dataset.}
    \label{fig:hyp_protein_rmse}
\end{minipage}
\begin{minipage}[b]{.5\textwidth}
    \centering
    \includegraphics[width=0.9\textwidth]{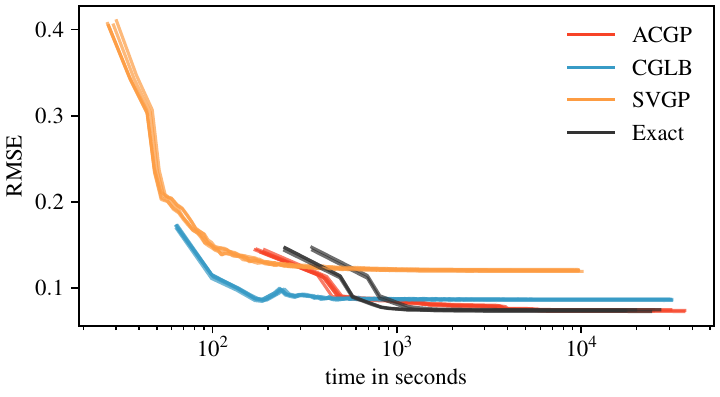}
    \caption{RMSE for the \texttt{kin40k} dataset.}
    \label{fig:hyp_kin40k_rmse}
\end{minipage}\\[5mm]
\begin{minipage}[b]{.5\textwidth}
	\centering
	\includegraphics[width=0.9\textwidth]{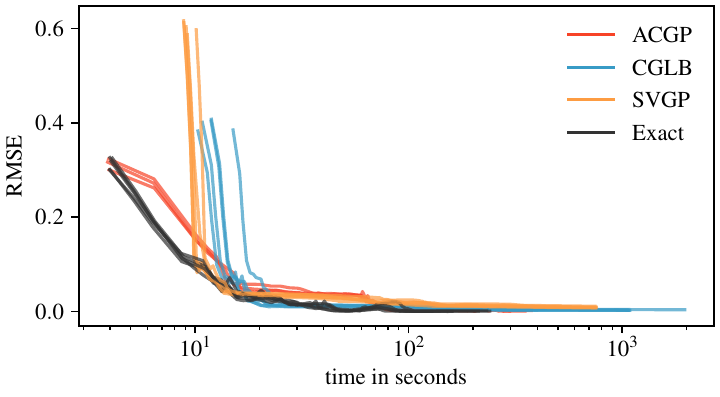}
	\caption{RMSE for the \texttt{bike} dataset.}
	\label{fig:hyp_bike_rmse}
\end{minipage}
\begin{minipage}[b]{.5\textwidth}
	\centering
	\includegraphics[width=0.9\textwidth]{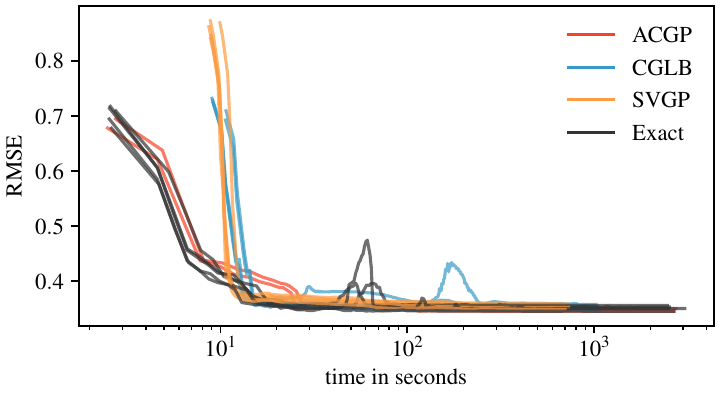}
	\caption{RMSE for the \texttt{elevators} dataset.}
	\label{fig:hyp_elevators_rmse}
\end{minipage}\\[5mm]
\begin{minipage}[b]{.5\textwidth}
	\centering
	\includegraphics[width=0.9\textwidth]{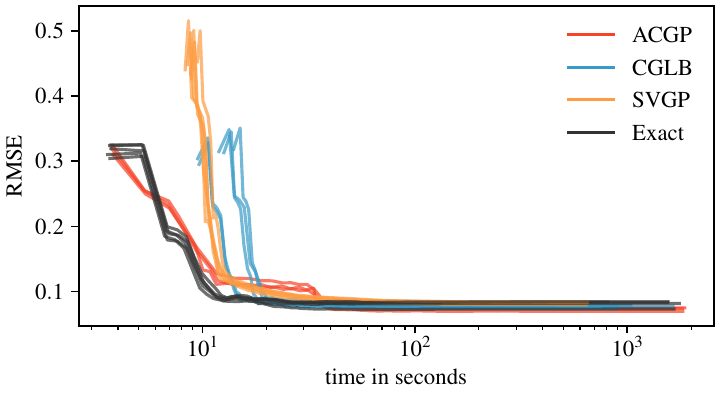}
	\caption{RMSE for the \texttt{pole} dataset.}
	\label{fig:hyp_pole_rmse}
\end{minipage}
\begin{minipage}[b]{.5\textwidth}
	\centering
	\includegraphics[width=0.9\textwidth]{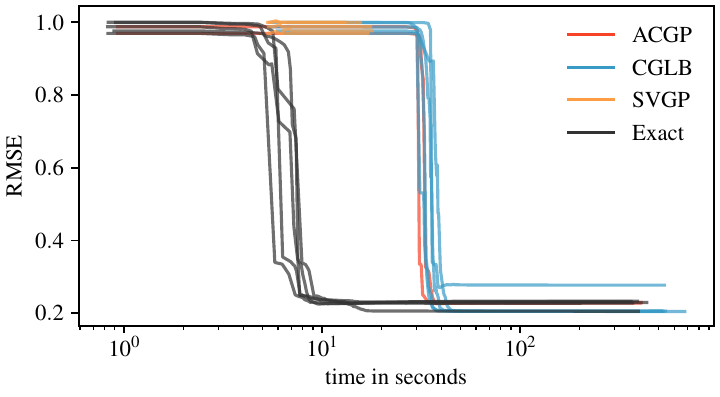}
	\caption{RMSE for the \texttt{pumadyn32nm} dataset.}
	\label{fig:hyp_pumadyn_rmse}
\end{minipage}
\end{figure}

\begin{figure}[htb]
	\begin{minipage}[b]{.5\textwidth}
		\centering
		\includegraphics[width=0.9\textwidth]{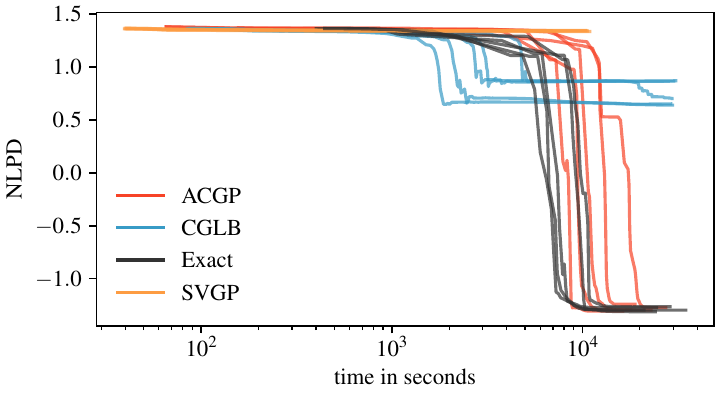}
		\caption{NLPD for the \texttt{metro} dataset.}
		\label{fig:hyp_metro_nlpd}
	\end{minipage}
	\begin{minipage}[b]{.5\textwidth}
		\centering
		\includegraphics[width=0.9\textwidth]{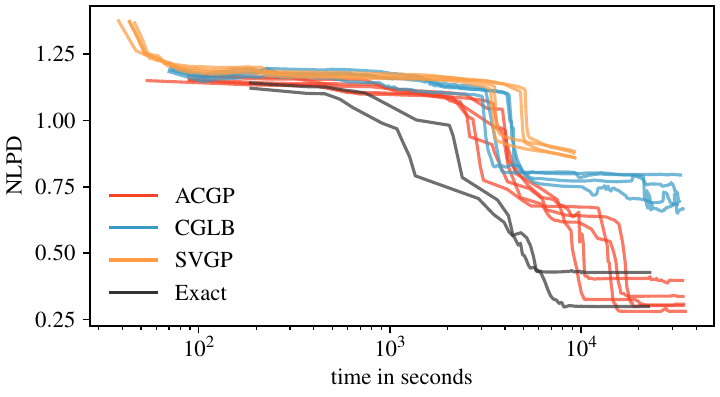}
		\caption{NLPD for the \texttt{pm25} dataset.}
		\label{fig:hyp_pm25_nlpd}
	\end{minipage}\\[5mm]
	\begin{minipage}[b]{.5\textwidth}
		\centering
		\includegraphics[width=0.9\textwidth]{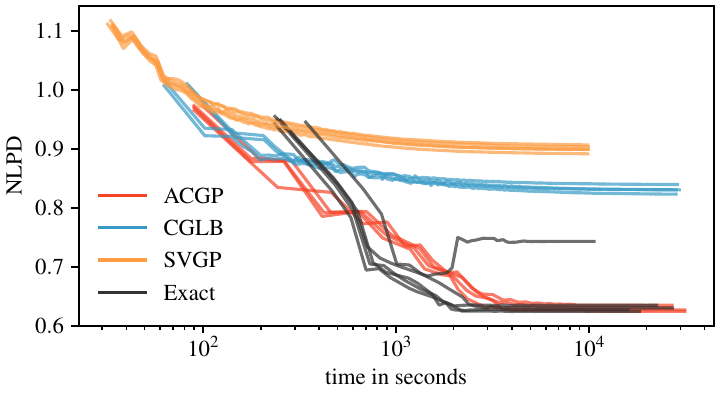}
		\caption{NLPD for the \texttt{protein} dataset.}
		\label{fig:hyp_protein_nlpd}
	\end{minipage}
	\begin{minipage}[b]{.5\textwidth}
		\centering
		\includegraphics[width=0.9\textwidth]{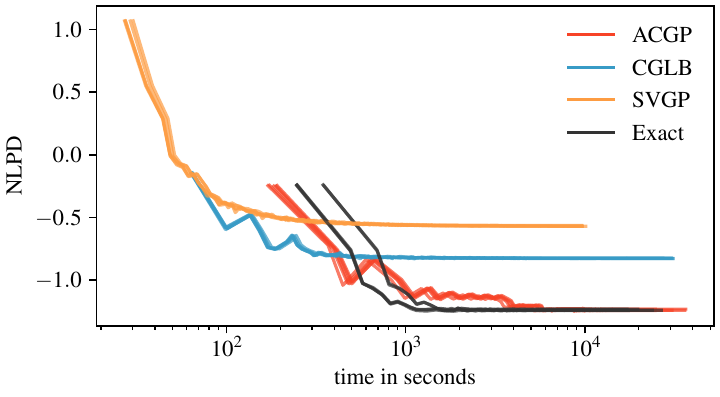}
		\caption{NLPD for the \texttt{kin40k} dataset.}
		\label{fig:hyp_kin40k_nlpd}
	\end{minipage}\\[5mm]
	\begin{minipage}[b]{.5\textwidth}
		\centering
		\includegraphics[width=0.9\textwidth]{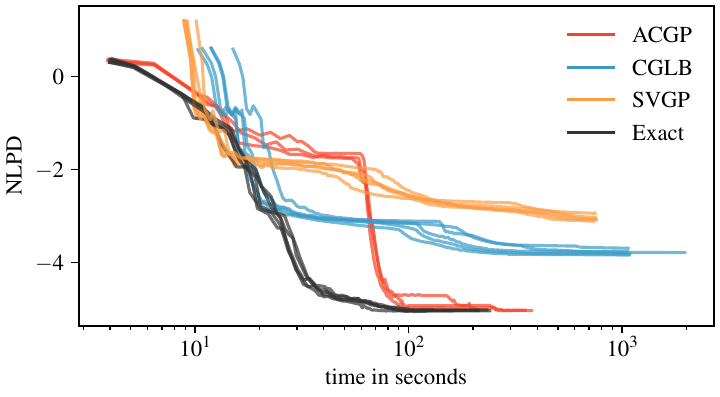}
		\caption{NLPD for the \texttt{bike} dataset.}
		\label{fig:hyp_bike_nlpd}
	\end{minipage}
	\begin{minipage}[b]{.5\textwidth}
		\centering
		\includegraphics[width=0.9\textwidth]{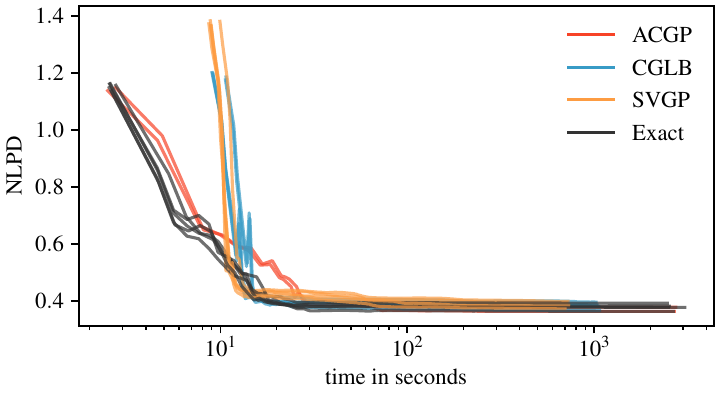}
		\caption{NLPD for the \texttt{elevators} dataset.}
		\label{fig:hyp_elevators_nlpd}
	\end{minipage}\\[5mm]
	\begin{minipage}[b]{.5\textwidth}
		\centering
		\includegraphics[width=0.9\textwidth]{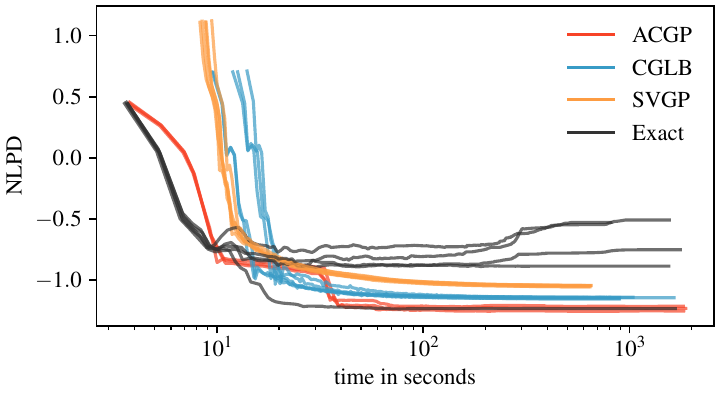}
		\caption{NLPD for the \texttt{pole} dataset.}
		\label{fig:hyp_pole_nlpd}
	\end{minipage}
	\begin{minipage}[b]{.5\textwidth}
		\centering
		\includegraphics[width=0.9\textwidth]{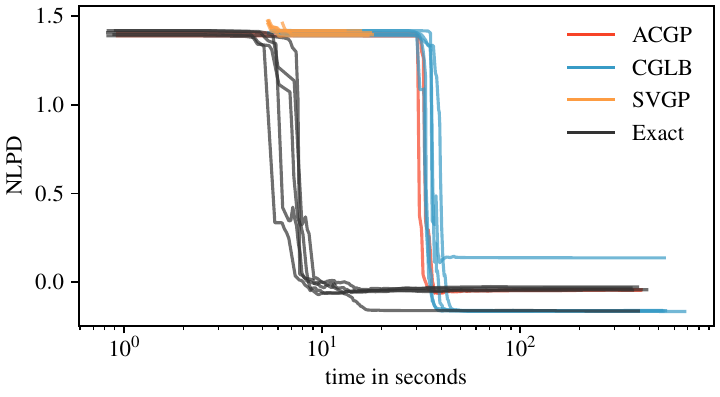}
		\caption{NLPD for the \texttt{pumadyn32nm} dataset.}
		\label{fig:hyp_pumadyn_nlpd}
	\end{minipage}
\end{figure}

\clearpage
\subsection{Additional plots for the bound quality experiments}
\label{subsec:bound_evolution}
\subsubsection{Bounds for experiments on \texttt{metro}}
\label{subsec:bounds_metro}
\newcommand{\subfigwidth}{0.95\textwidth}
\vspace*{-3mm}
\begin{figure}[htb!]
   \begin{minipage}[b]{.5\textwidth}
      \centering
      \includegraphics[width=\subfigwidth]{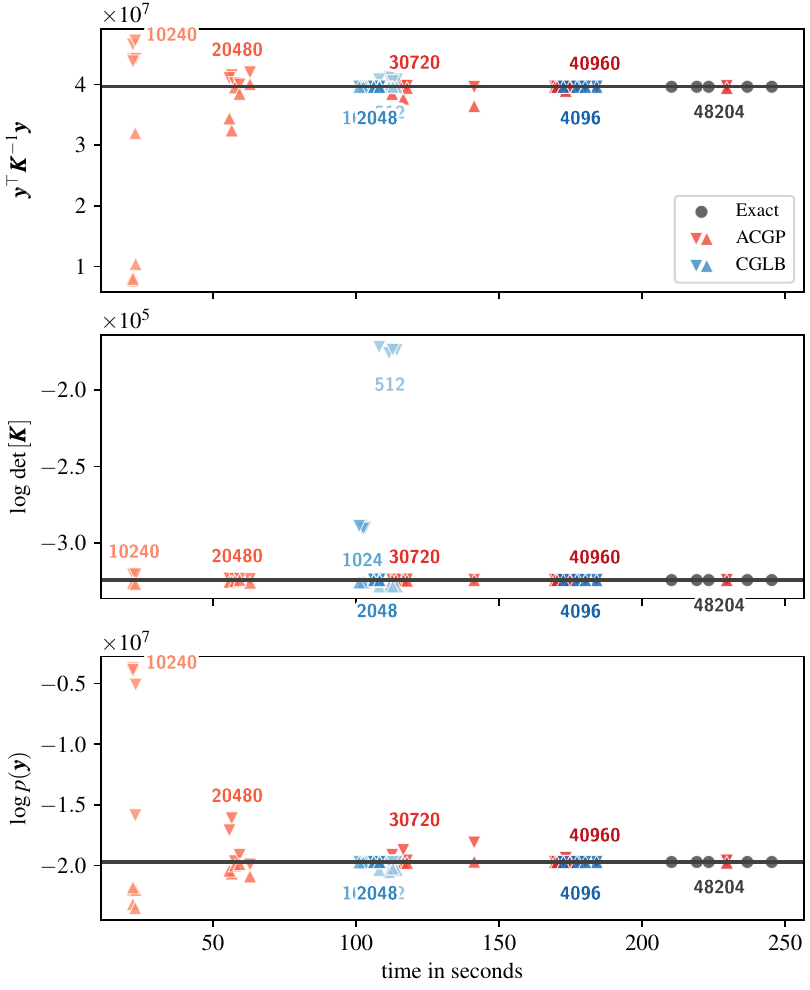}
      \subcaption{SE kernel, $\log\ell = -1$.}
      \label{subfig:metro_rbf_-1}
   \end{minipage}
   \begin{minipage}[b]{.5\textwidth}
      \centering
      \includegraphics[width=\subfigwidth]{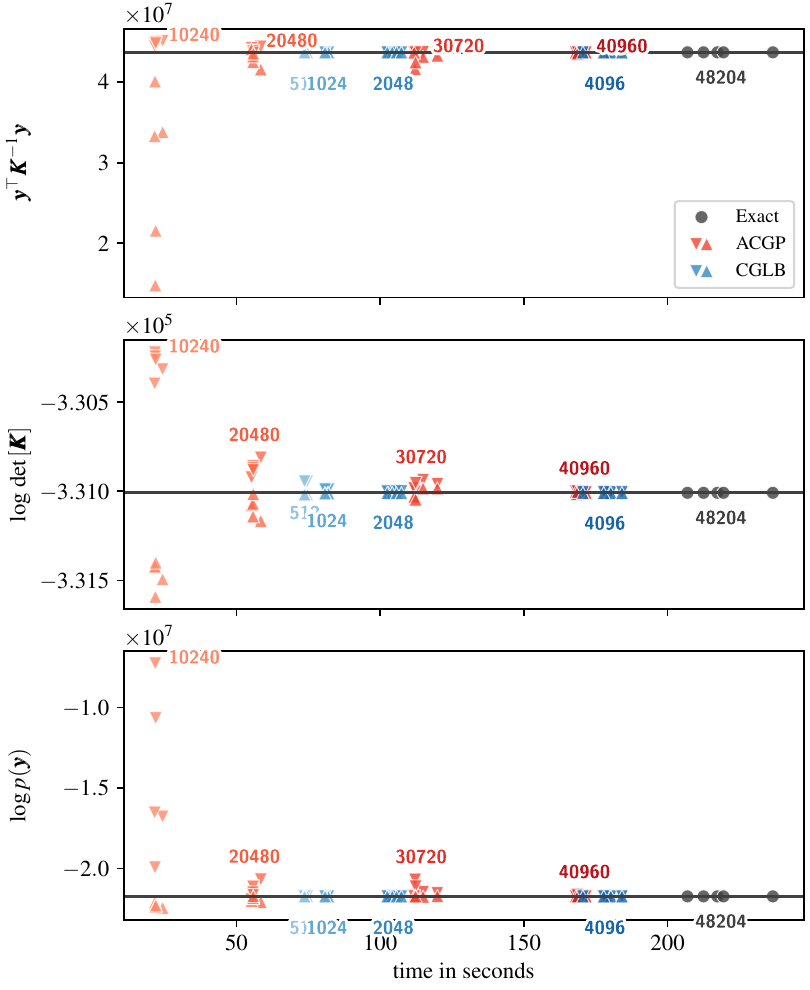}
      \subcaption{SE kernel, $\log\ell = 0$.}
      \label{subfig:metro_rbf_0}
   \end{minipage}
   \begin{minipage}[b]{.5\textwidth}
      \centering
      \includegraphics[width=\subfigwidth]{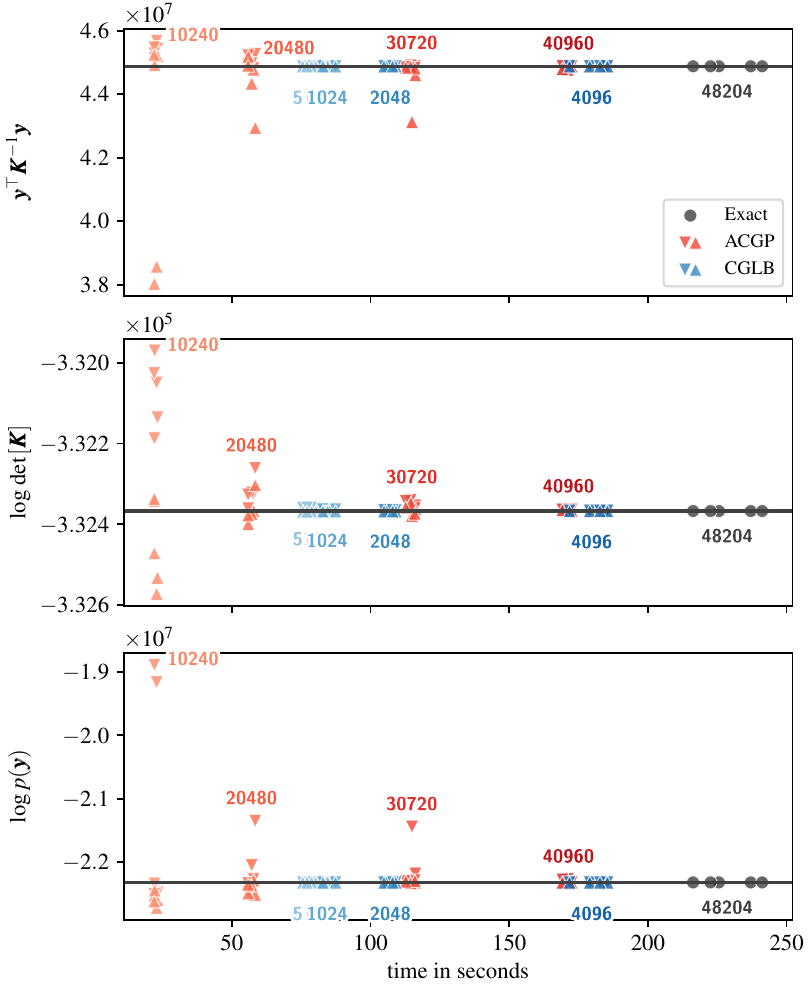}
      \subcaption{SE kernel, $\log\ell = 1$.}
      \label{subfig:metro_rbf_1}
   \end{minipage}
   \begin{minipage}[b]{.5\textwidth}
      \centering
      \includegraphics[width=\subfigwidth]{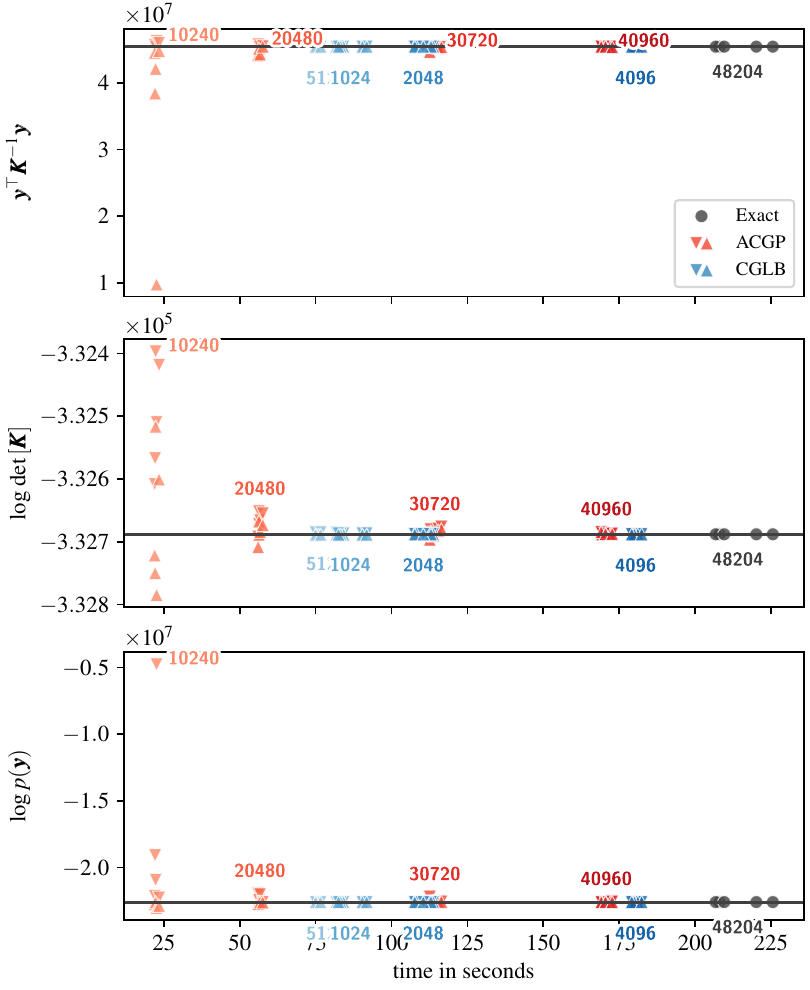}
      \subcaption{SE kernel, $\log\ell = 2$.}
      \label{subfig:metro_rbf_2}
   \end{minipage}
   \caption{Upper and lower bounds for the \texttt{metro} dataset when using a squared exponential (SE) kernel.}
	\label{fig:bounds_metro_rbf}
\end{figure}
\clearpage
\begin{figure}[htb!]
   \begin{minipage}[b]{.5\textwidth}
      \centering
      \includegraphics[width=\subfigwidth]{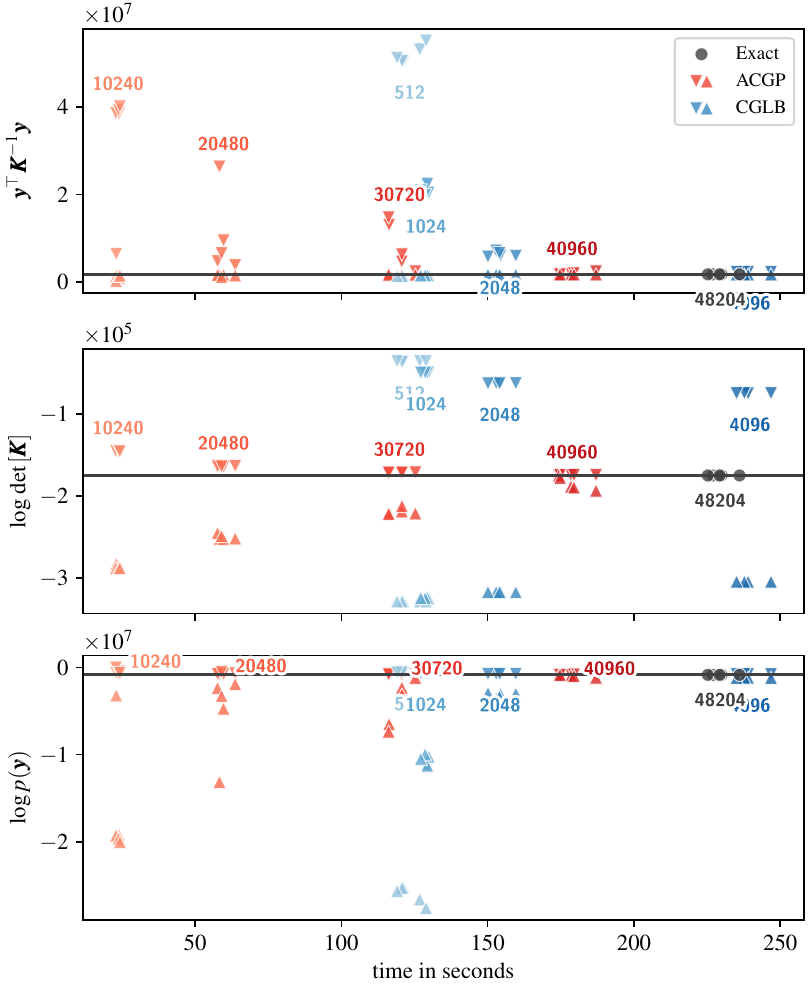}
      \subcaption{OU kernel, $\log\ell = -1$.}
      \label{subfig:metro_ou_-1}
   \end{minipage}
   \begin{minipage}[b]{.5\textwidth}
      \centering
      \includegraphics[width=\subfigwidth]{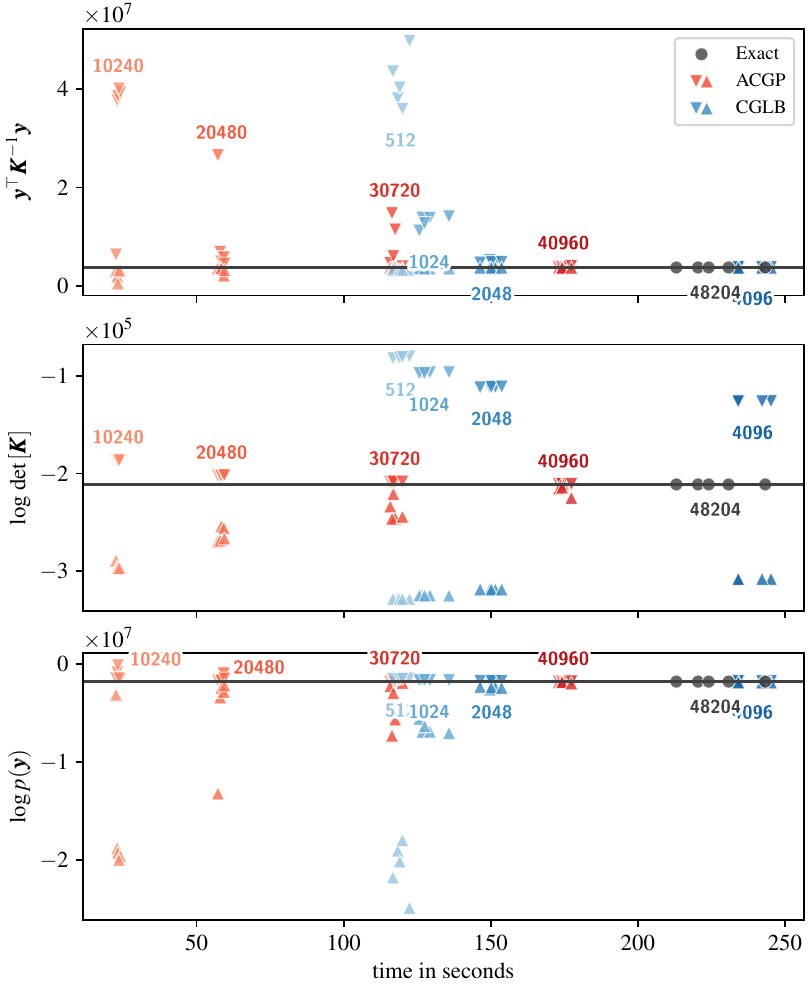}
      \subcaption{OU kernel, $\log\ell = 0$.}
      \label{subfig:metro_ou_0}
   \end{minipage}
   \begin{minipage}[b]{.5\textwidth}
      \centering
      \includegraphics[width=\subfigwidth]{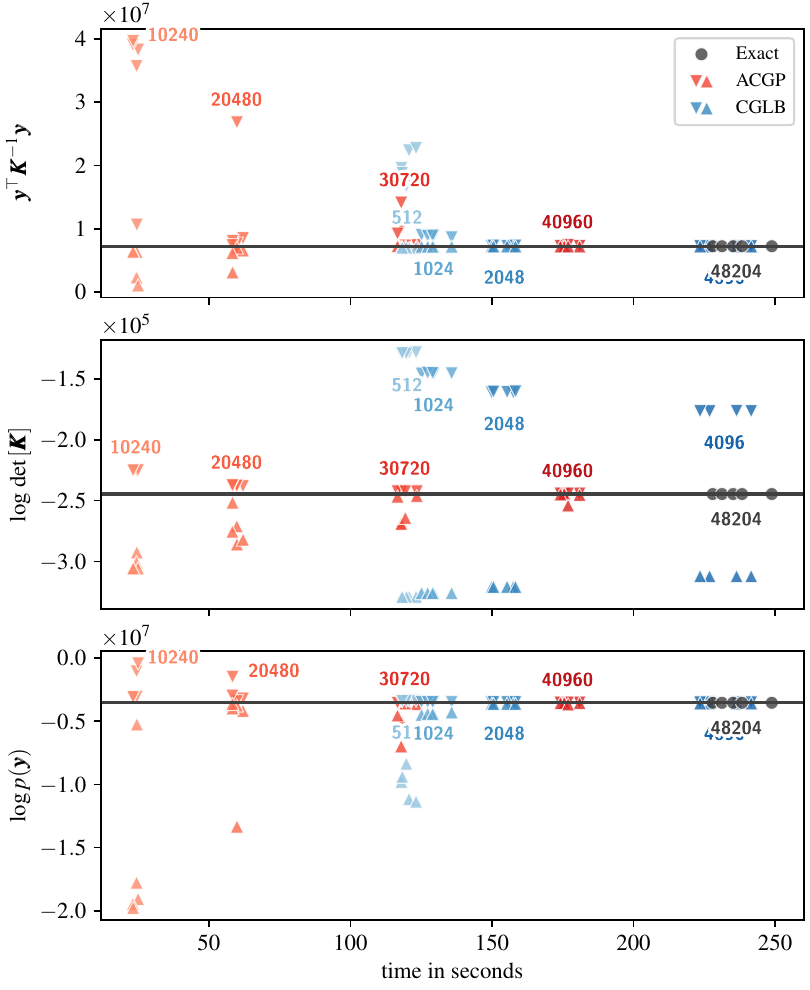}
      \subcaption{OU kernel, $\log\ell = 1$.}
      \label{subfig:metro_ou_1}
   \end{minipage}
   \begin{minipage}[b]{.5\textwidth}
      \centering
      \includegraphics[width=\subfigwidth]{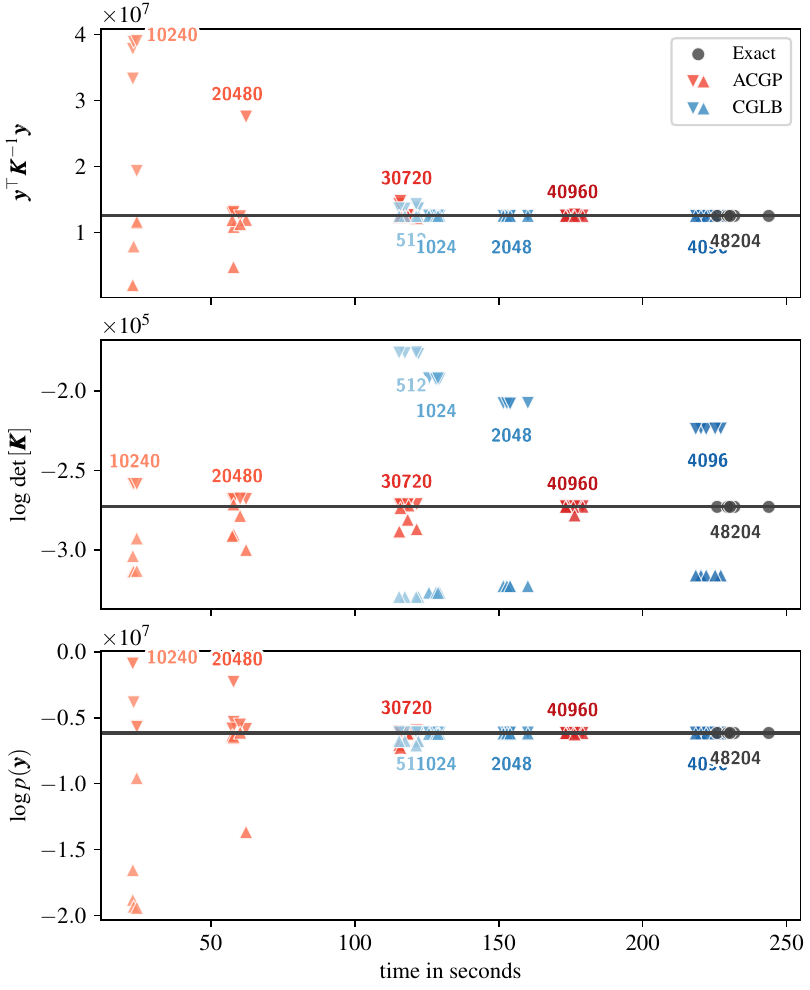}
      \subcaption{OU kernel, $\log\ell = 2$.}
      \label{subfig:metro_ou_2}
   \end{minipage}
   \caption{Upper and lower bounds for the \texttt{metro} dataset when using a Ornstein-Uhlenbeck (OU) kernel.}
	\label{fig:bounds_metro_ou}
\end{figure}
\clearpage

\subsubsection{Bounds for experiments on \texttt{pm25}}
\label{subsec:bounds_pm25}
\vspace*{-3mm}
\begin{figure}[htb!]
   \begin{minipage}[b]{.5\textwidth}
      \centering
      \includegraphics[width=\subfigwidth]{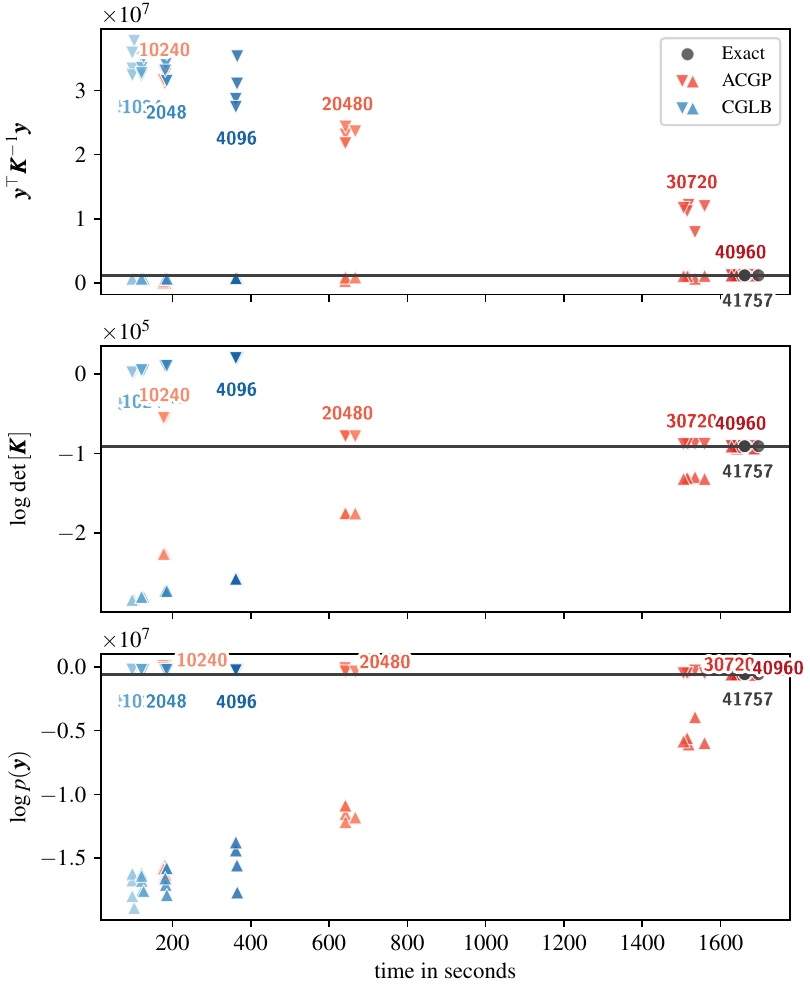}
      \subcaption{SE kernel, $\log\ell = -1$.}
      \label{subfig:pm25_rbf_-1}
   \end{minipage}
   \begin{minipage}[b]{.5\textwidth}
      \centering
      \includegraphics[width=\subfigwidth]{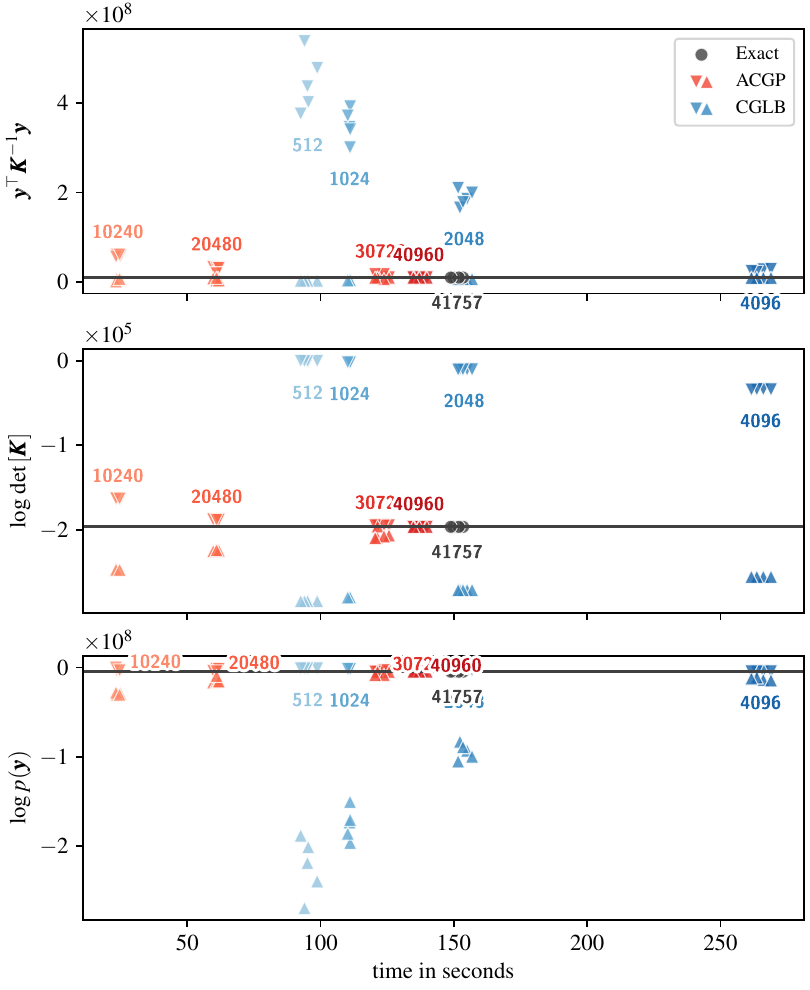}
      \subcaption{SE kernel, $\log\ell = 0$.}
      \label{subfig:pm25_rbf_0}
   \end{minipage}
   \begin{minipage}[b]{.5\textwidth}
      \centering
      \includegraphics[width=\subfigwidth]{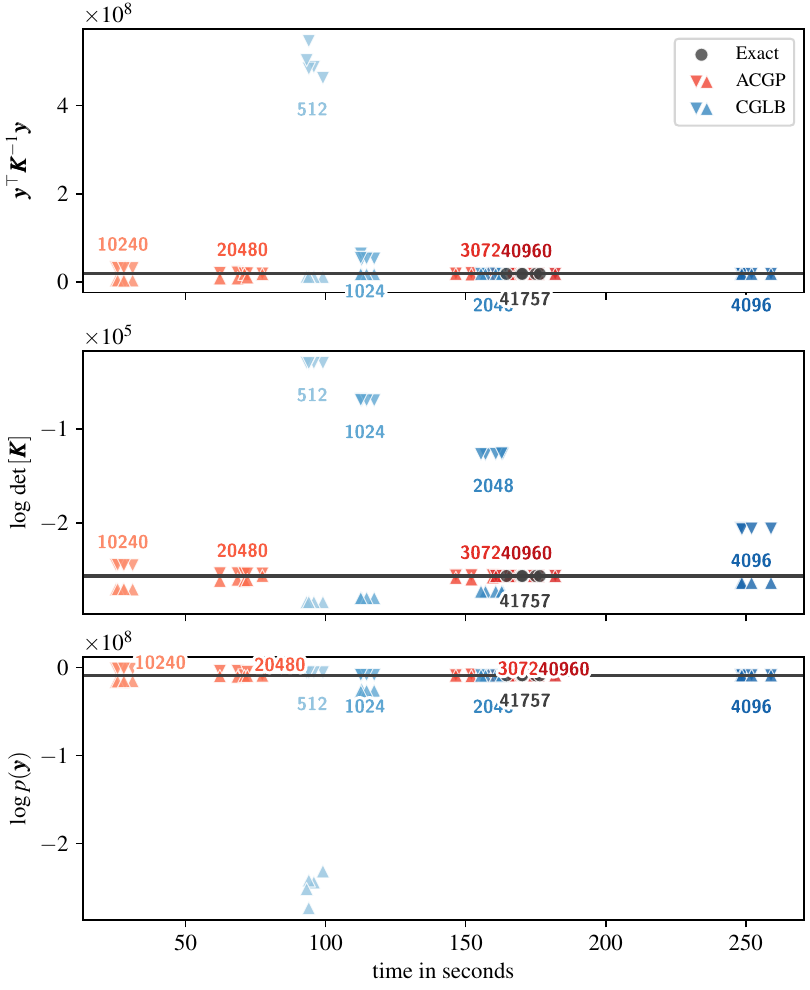}
      \subcaption{SE kernel, $\log\ell = 1$.}
      \label{subfig:pm25_rbf_1}
   \end{minipage}
   \begin{minipage}[b]{.5\textwidth}
      \centering
      \includegraphics[width=\subfigwidth]{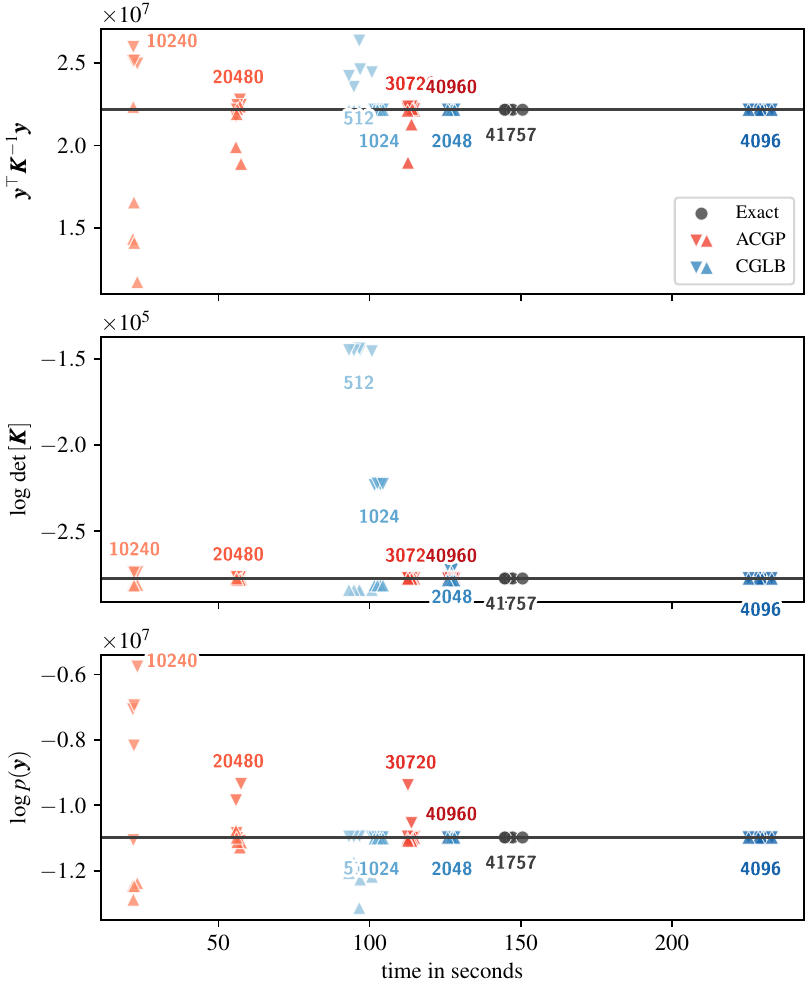}
      \subcaption{SE kernel, $\log\ell = 2$.}
      \label{subfig:pm25_rbf_2}
   \end{minipage}
   \caption{Upper and lower bounds for the \texttt{pm25} dataset when using a squared exponential (SE) kernel.}
	\label{fig:bounds_pm25_rbf}
\end{figure}
\clearpage
\begin{figure}[htb!]
   \begin{minipage}[b]{.5\textwidth}
      \centering
      \includegraphics[width=\subfigwidth]{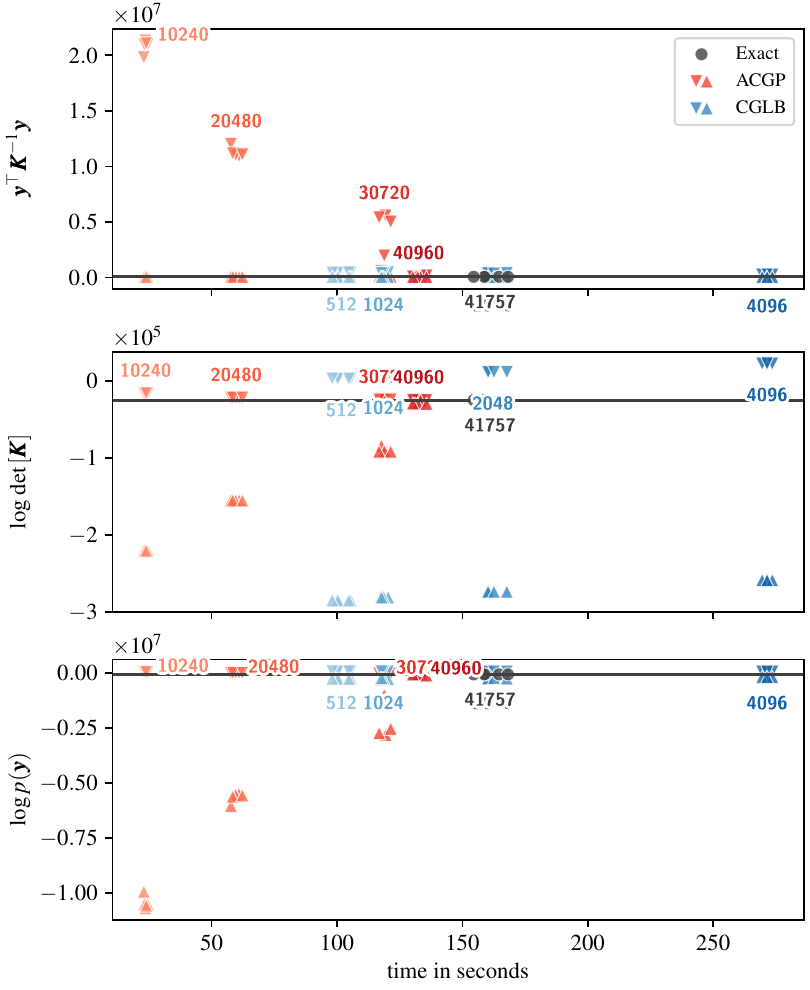}
      \subcaption{OU kernel, $\log\ell = -1$.}
      \label{subfig:pm25_ou_-1}
   \end{minipage}
   \begin{minipage}[b]{.5\textwidth}
      \centering
      \includegraphics[width=\subfigwidth]{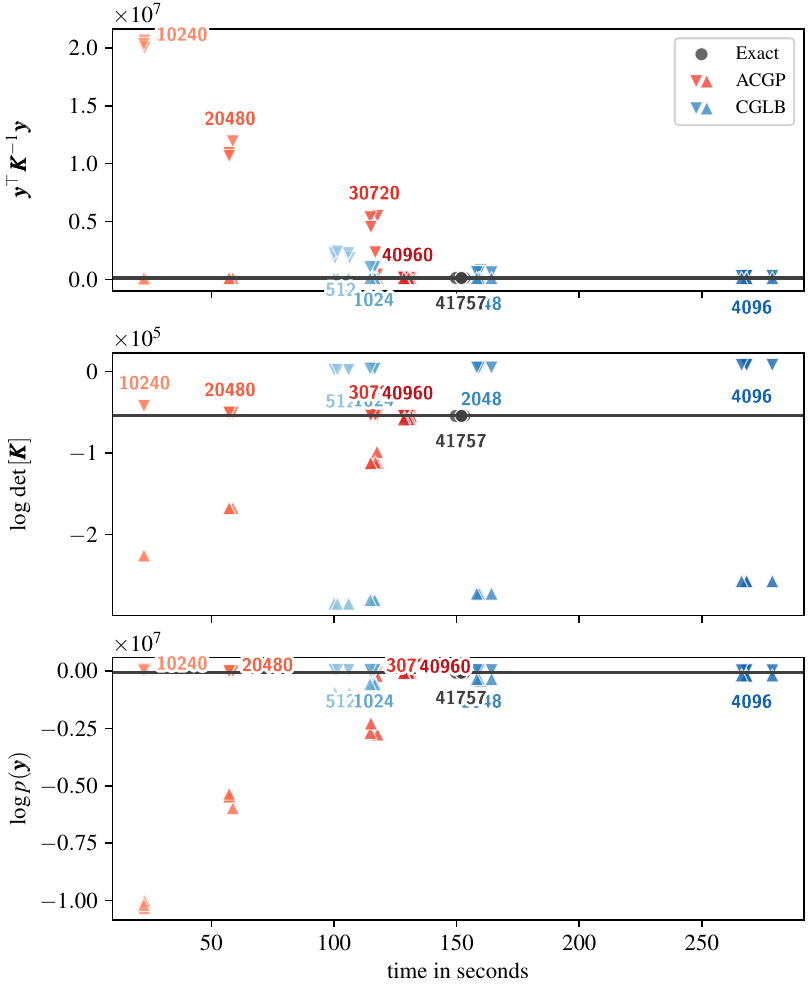}
      \subcaption{OU kernel, $\log\ell = 0$.}
      \label{subfig:pm25_ou_0}
   \end{minipage}
   \begin{minipage}[b]{.5\textwidth}
      \centering
      \includegraphics[width=\subfigwidth]{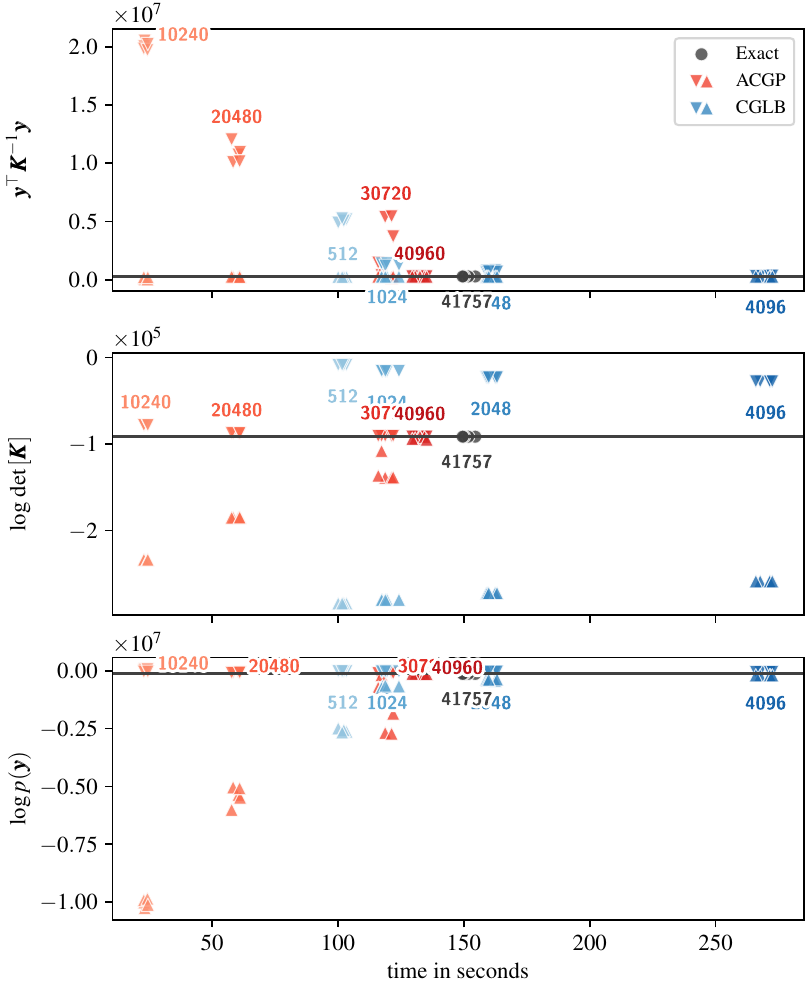}
      \subcaption{OU kernel, $\log\ell = 1$.}
      \label{subfig:pm25_ou_1}
   \end{minipage}
   \begin{minipage}[b]{.5\textwidth}
      \centering
      \includegraphics[width=\subfigwidth]{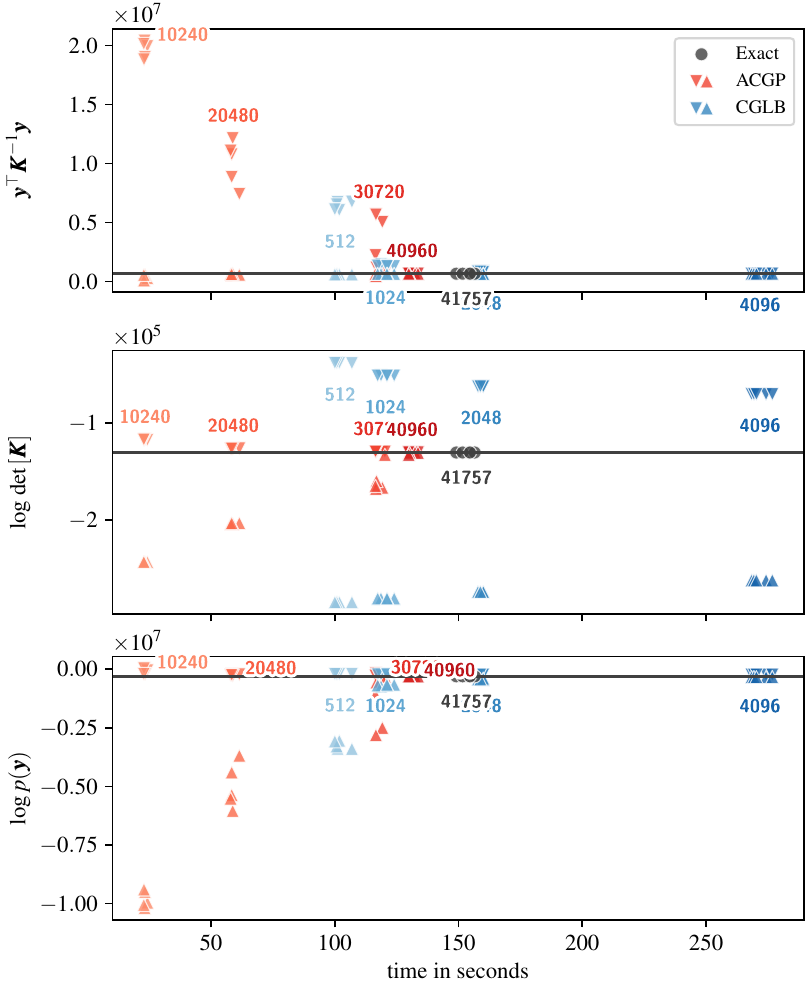}
      \subcaption{OU kernel, $\log\ell = 2$.}
      \label{subfig:pm25_ou_2}
   \end{minipage}
   \caption{Upper and lower bounds for the \texttt{pm25} dataset when using a Ornstein-Uhlenbeck (OU) kernel.}
	\label{fig:bounds_pm25_ou}
\end{figure}
\clearpage

\subsubsection{Bounds for experiments on \texttt{protein}}
\label{subsec:bounds_protein}
\vspace*{-3mm}
\begin{figure}[htb!]
   \begin{minipage}[b]{.5\textwidth}
      \centering
      \includegraphics[width=\subfigwidth]{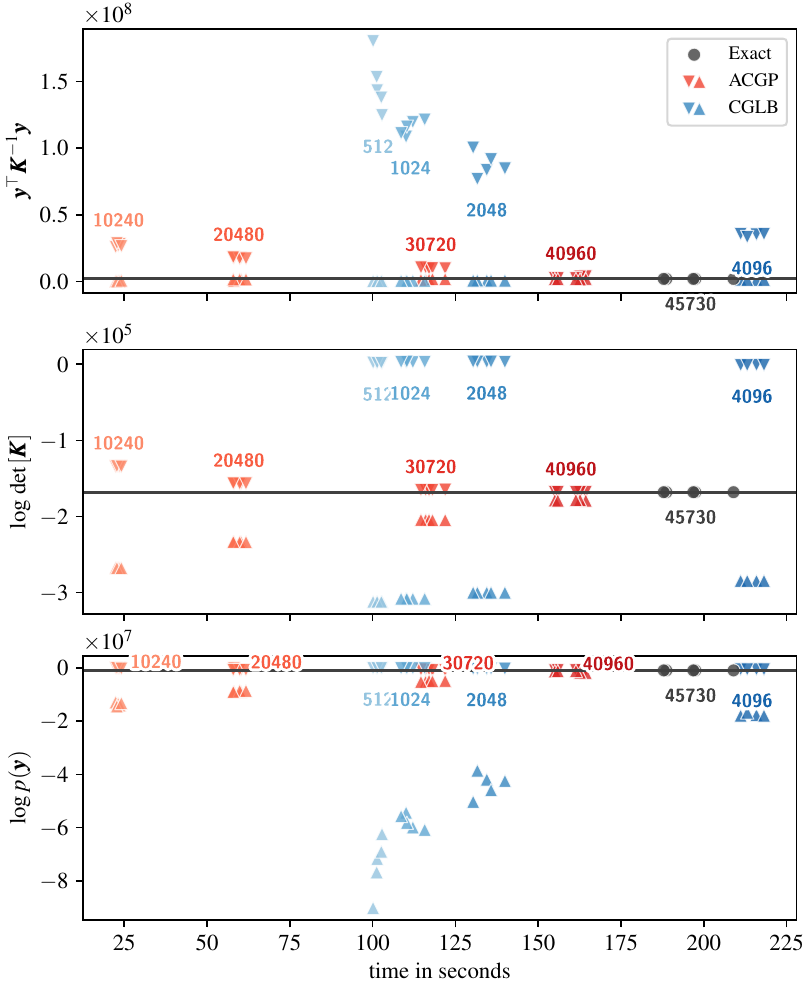}
      \subcaption{SE kernel, $\log\ell = -1$.}
      \label{subfig:protein_rbf_-1}
   \end{minipage}
   \begin{minipage}[b]{.5\textwidth}
      \centering
      \includegraphics[width=\subfigwidth]{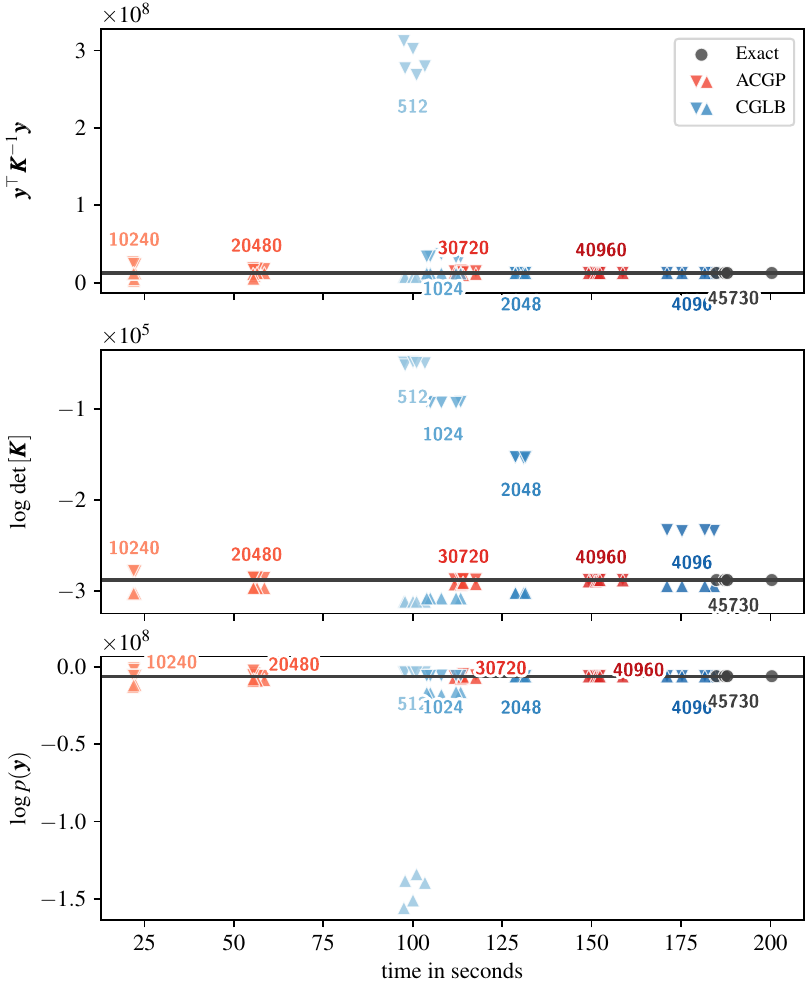}
      \subcaption{SE kernel, $\log\ell = 0$.}
      \label{subfig:protein_rbf_0}
   \end{minipage}
   \begin{minipage}[b]{.5\textwidth}
      \centering
      \includegraphics[width=\subfigwidth]{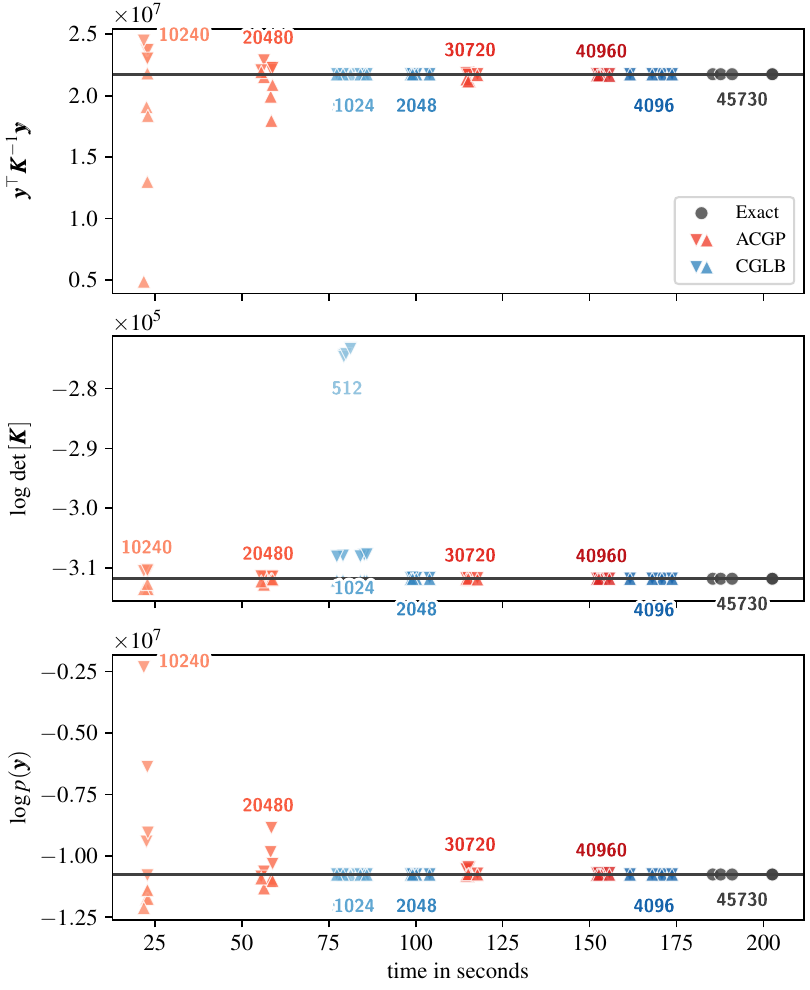}
      \subcaption{SE kernel, $\log\ell = 1$.}
      \label{subfig:protein_rbf_1}
   \end{minipage}
   \begin{minipage}[b]{.5\textwidth}
      \centering
      \includegraphics[width=\subfigwidth]{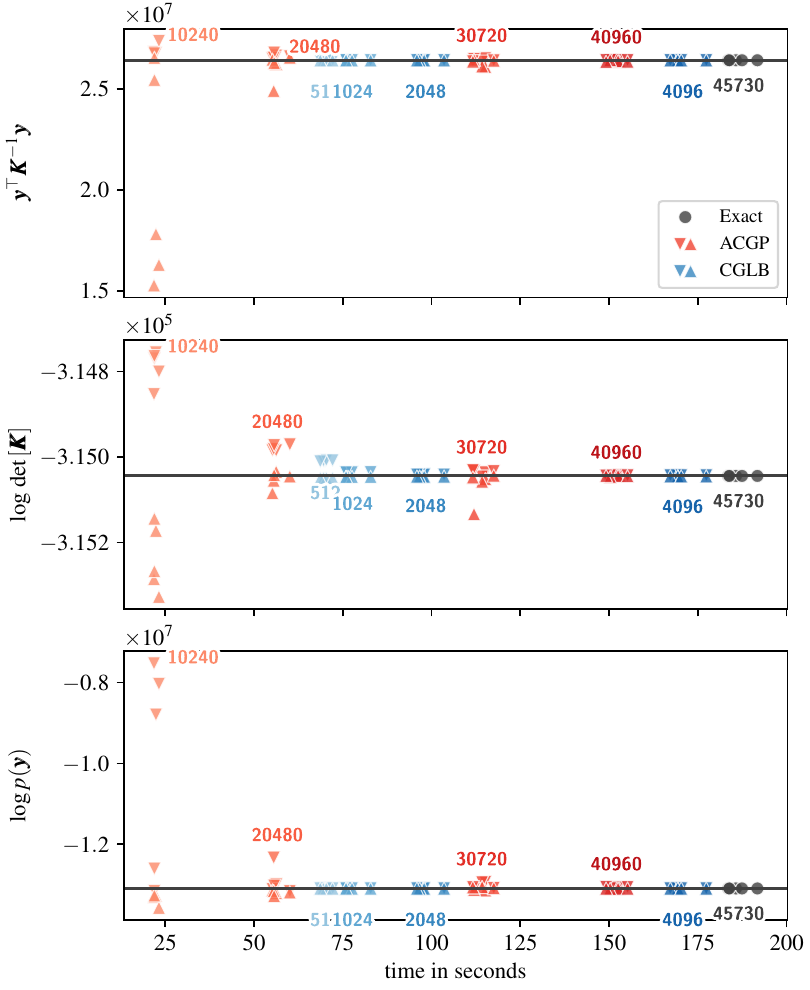}
      \subcaption{SE kernel, $\log\ell = 2$.}
      \label{subfig:protein_rbf_2}
   \end{minipage}
   \caption{Upper and lower bounds for the \texttt{protein} dataset when using a squared exponential (SE) kernel.}
	\label{fig:bounds_protein_rbf}
\end{figure}
\clearpage
\begin{figure}[htb!]
   \begin{minipage}[b]{.5\textwidth}
      \centering
      \includegraphics[width=\subfigwidth]{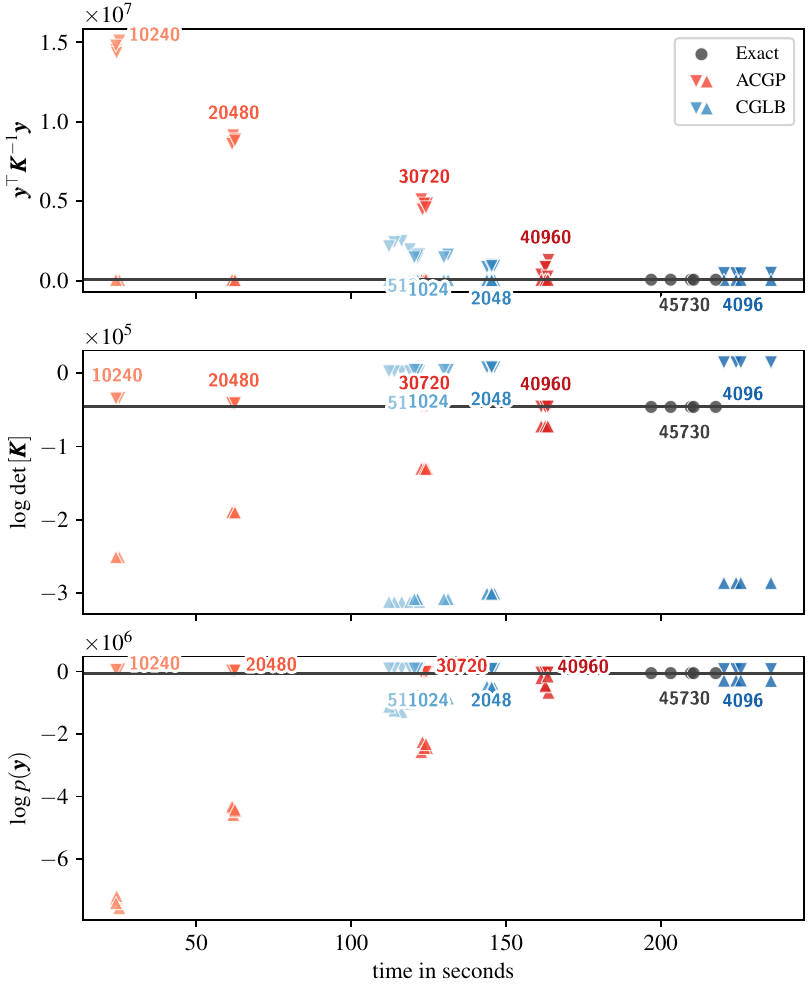}
      \subcaption{OU kernel, $\log\ell = -1$.}
      \label{subfig:protein_ou_-1}
   \end{minipage}
   \begin{minipage}[b]{.5\textwidth}
      \centering
      \includegraphics[width=\subfigwidth]{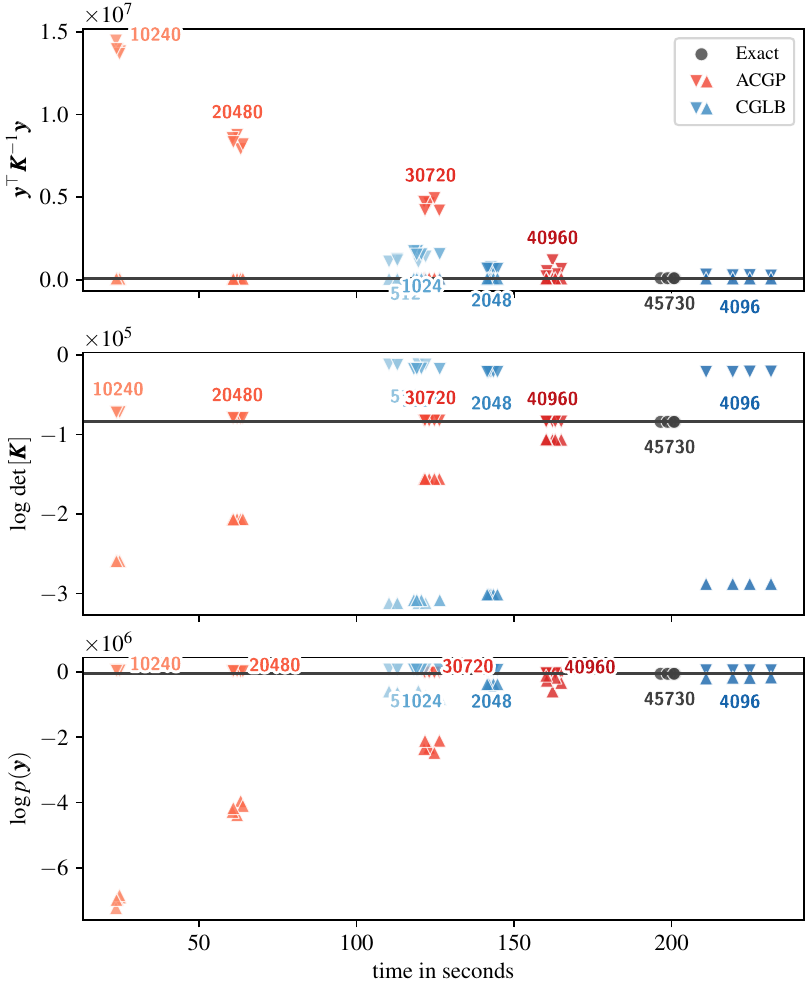}
      \subcaption{OU kernel, $\log\ell = 0$.}
      \label{subfig:protein_ou_0}
   \end{minipage}
   \begin{minipage}[b]{.5\textwidth}
      \centering
      \includegraphics[width=\subfigwidth]{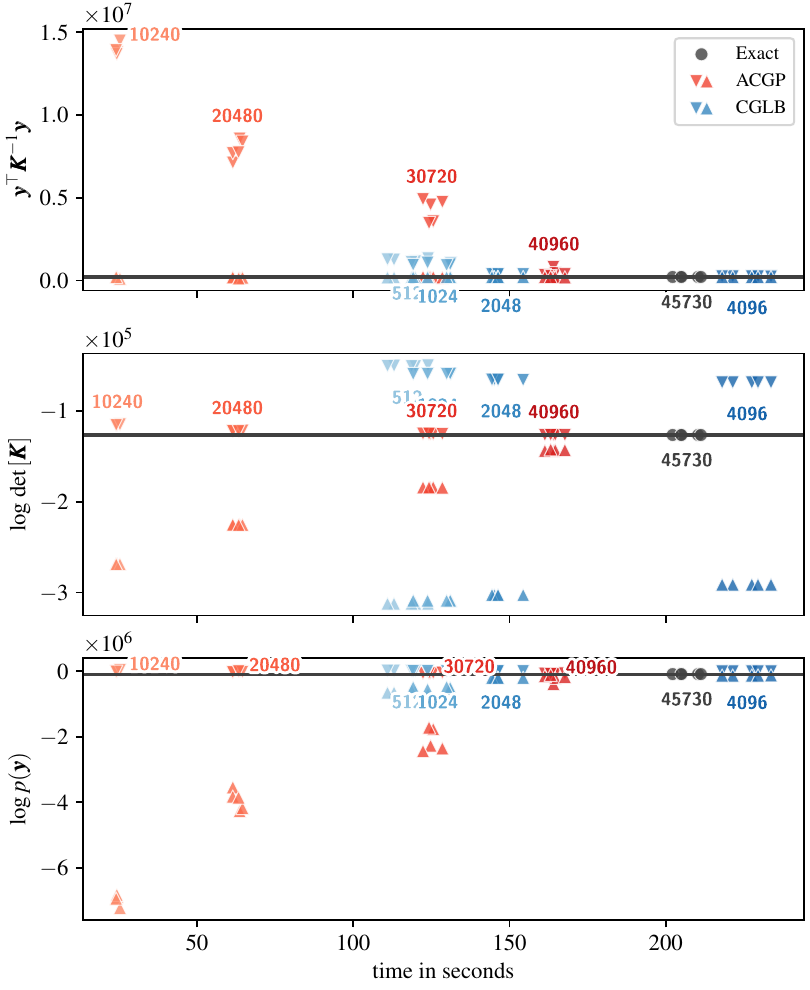}
      \subcaption{OU kernel, $\log\ell = 1$.}
      \label{subfig:protein_ou_1}
   \end{minipage}
   \begin{minipage}[b]{.5\textwidth}
      \centering
      \includegraphics[width=\subfigwidth]{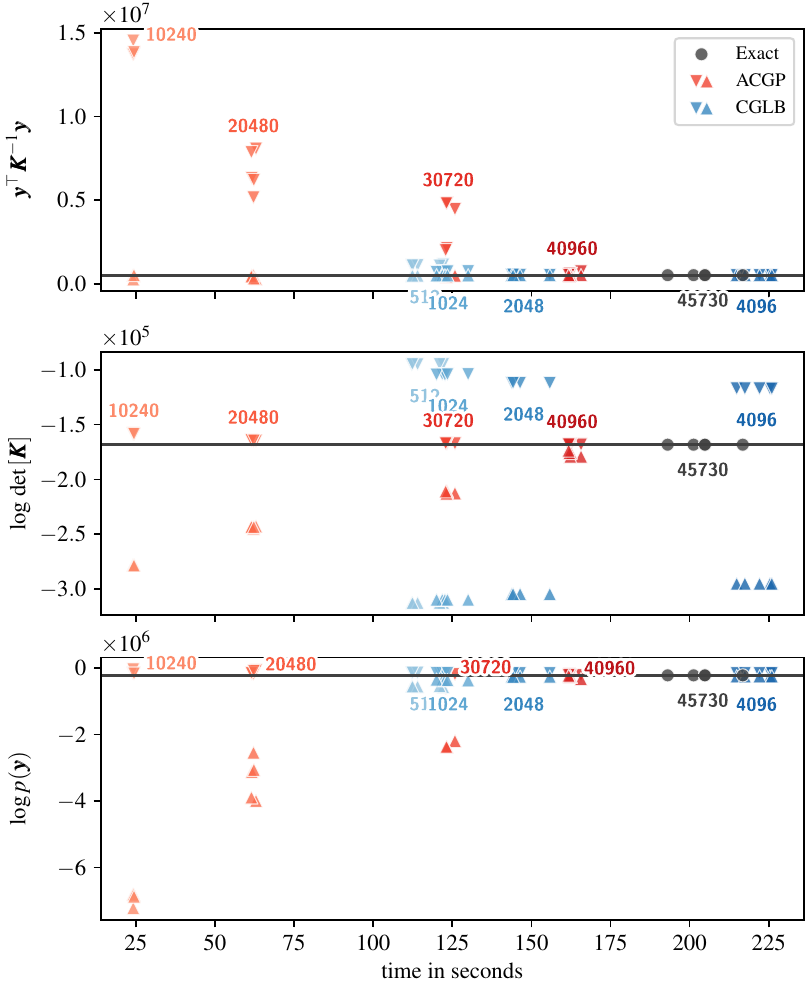}
      \subcaption{OU kernel, $\log\ell = 2$.}
      \label{subfig:protein_ou_2}
   \end{minipage}
   \caption{Upper and lower bounds for the \texttt{protein} dataset when using a Ornstein-Uhlenbeck (OU) kernel.}
	\label{fig:bounds_protein_ou}
\end{figure}
\clearpage

\subsubsection{Bounds for experiments on \texttt{kin40k}}
\label{subsec:bounds_kin40k}
\vspace*{-3mm}
\begin{figure}[htb!]
   \begin{minipage}[b]{.5\textwidth}
      \centering
      \includegraphics[width=\subfigwidth]{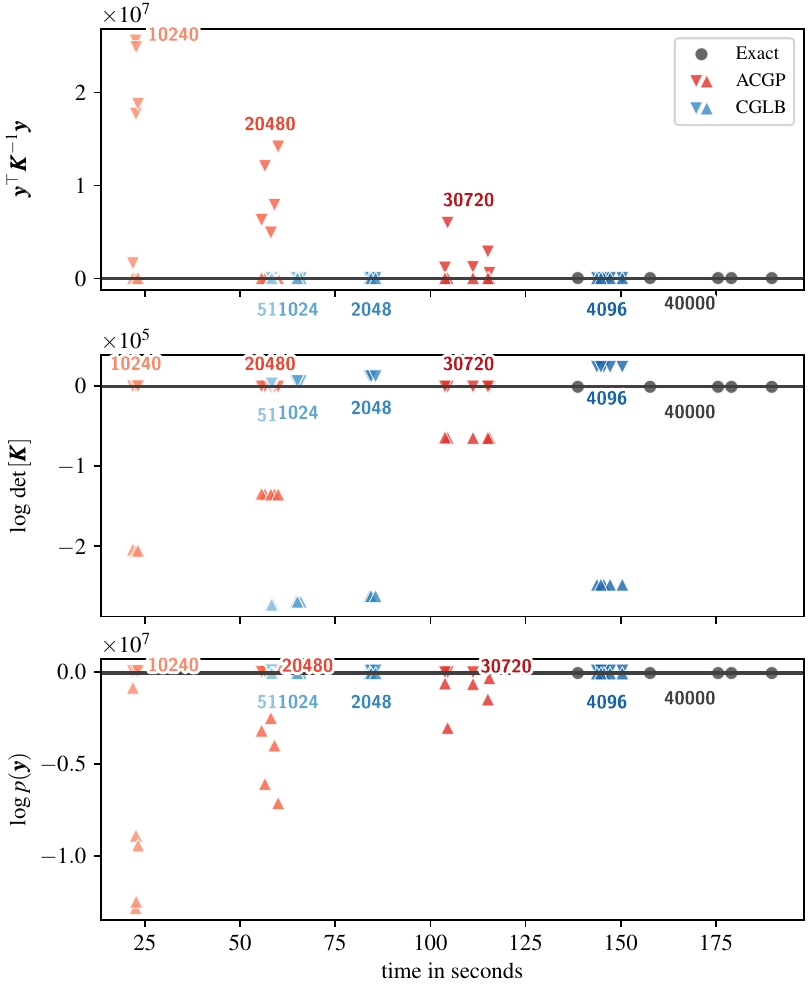}
      \subcaption{SE kernel, $\log\ell = -1$.}
      \label{subfig:kin40k_rbf_-1}
   \end{minipage}
   \begin{minipage}[b]{.5\textwidth}
      \centering
      \includegraphics[width=\subfigwidth]{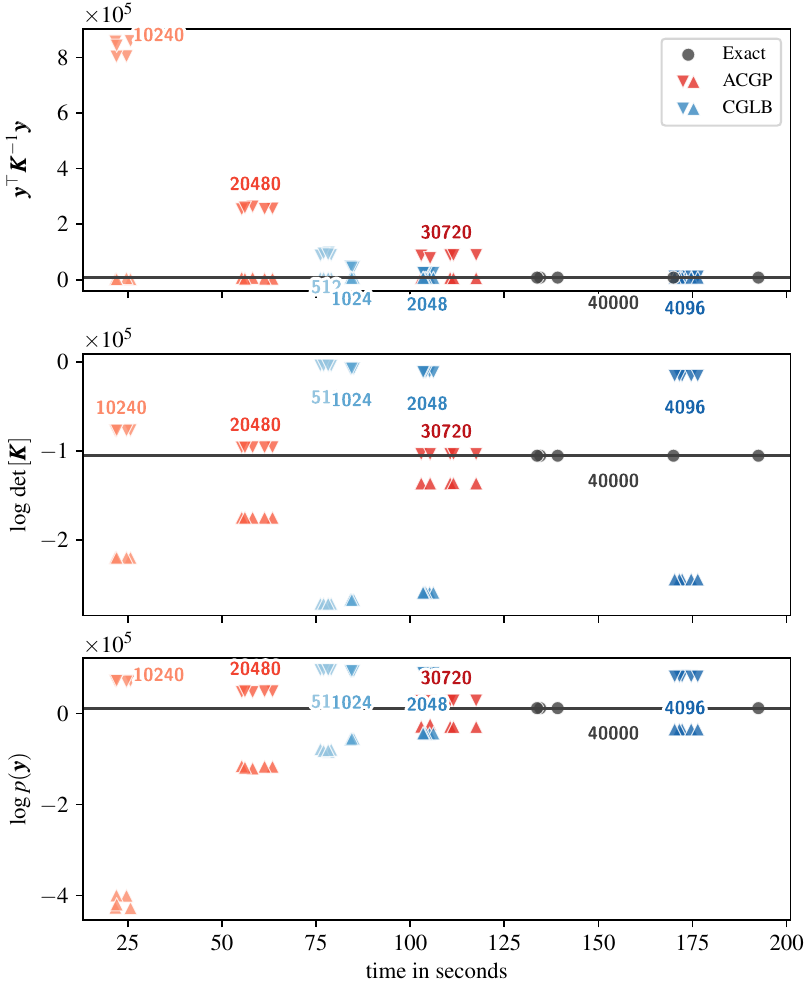}
      \subcaption{SE kernel, $\log\ell = 0$.}
      \label{subfig:kin40k_rbf_0}
   \end{minipage}
   \begin{minipage}[b]{.5\textwidth}
      \centering
      \includegraphics[width=\subfigwidth]{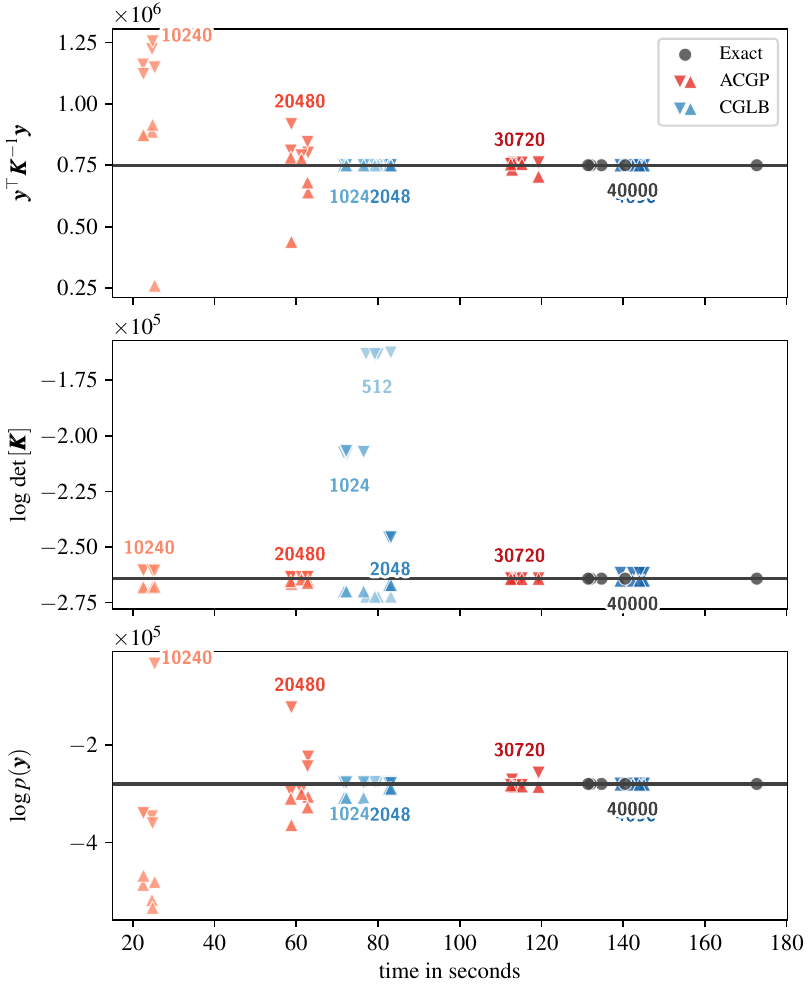}
      \subcaption{SE kernel, $\log\ell = 1$.}
      \label{subfig:kin40k_rbf_1}
   \end{minipage}
   \begin{minipage}[b]{.5\textwidth}
      \centering
      \includegraphics[width=\subfigwidth]{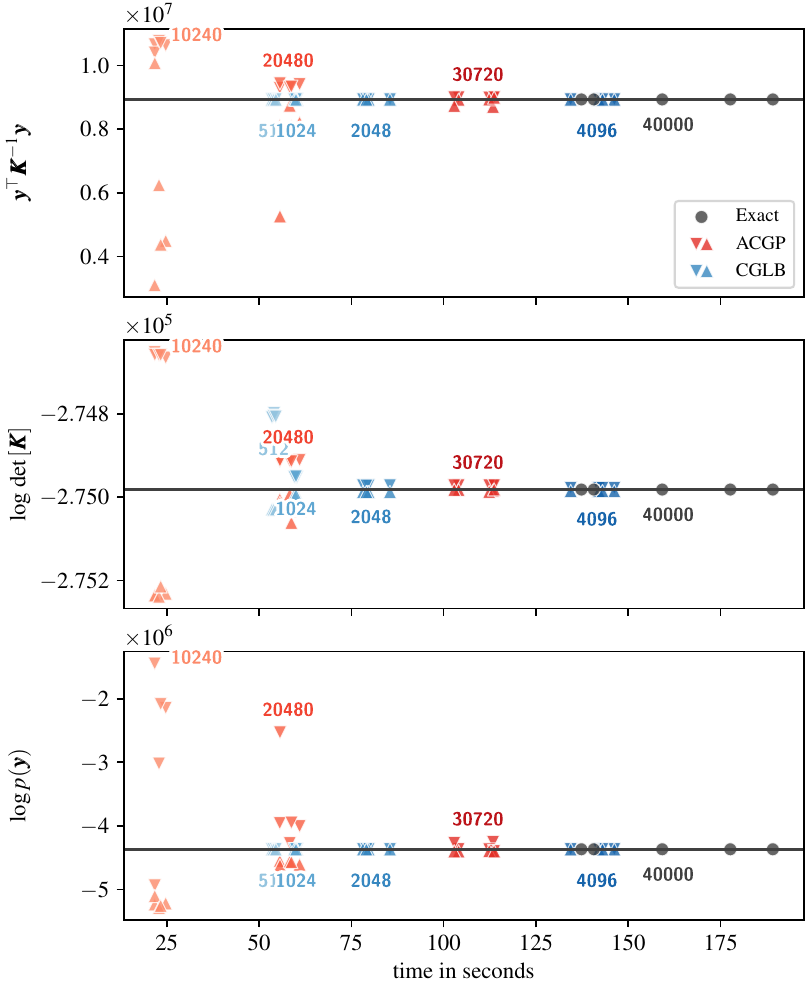}
      \subcaption{SE kernel, $\log\ell = 2$.}
      \label{subfig:kin40k_rbf_2}
   \end{minipage}
   \caption{Upper and lower bounds for the \texttt{kin40k} dataset when using a squared exponential (SE) kernel.}
	\label{fig:bounds_kin40k_rbf}
\end{figure}
\clearpage
\begin{figure}[htb!]
   \begin{minipage}[b]{.5\textwidth}
      \centering
      \includegraphics[width=\subfigwidth]{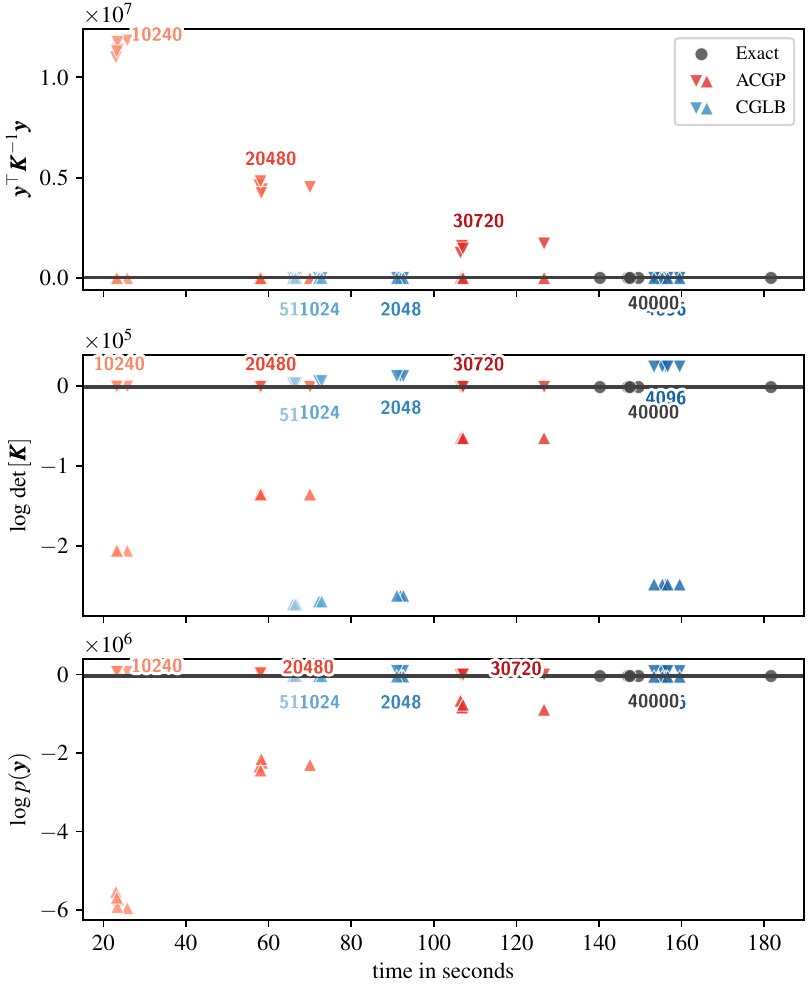}
      \subcaption{OU kernel, $\log\ell = -1$.}
      \label{subfig:kin40k_ou_-1}
   \end{minipage}
   \begin{minipage}[b]{.5\textwidth}
      \centering
      \includegraphics[width=\subfigwidth]{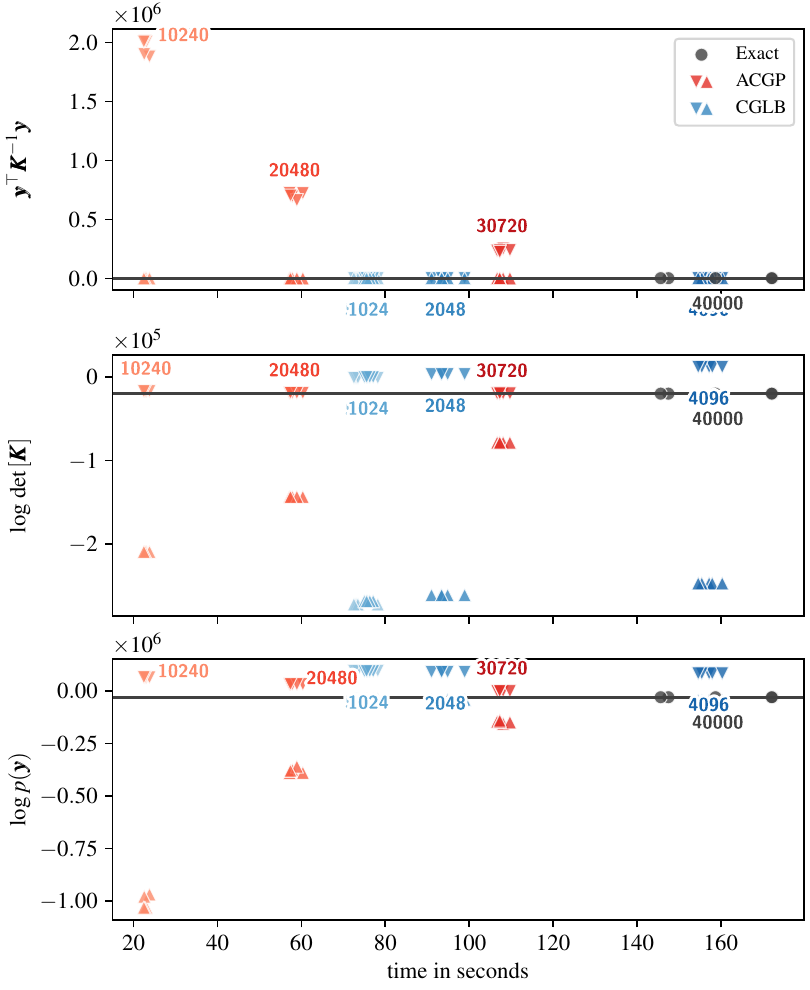}
      \subcaption{OU kernel, $\log\ell = 0$.}
      \label{subfig:kin40k_ou_0}
   \end{minipage}
   \begin{minipage}[b]{.5\textwidth}
      \centering
      \includegraphics[width=\subfigwidth]{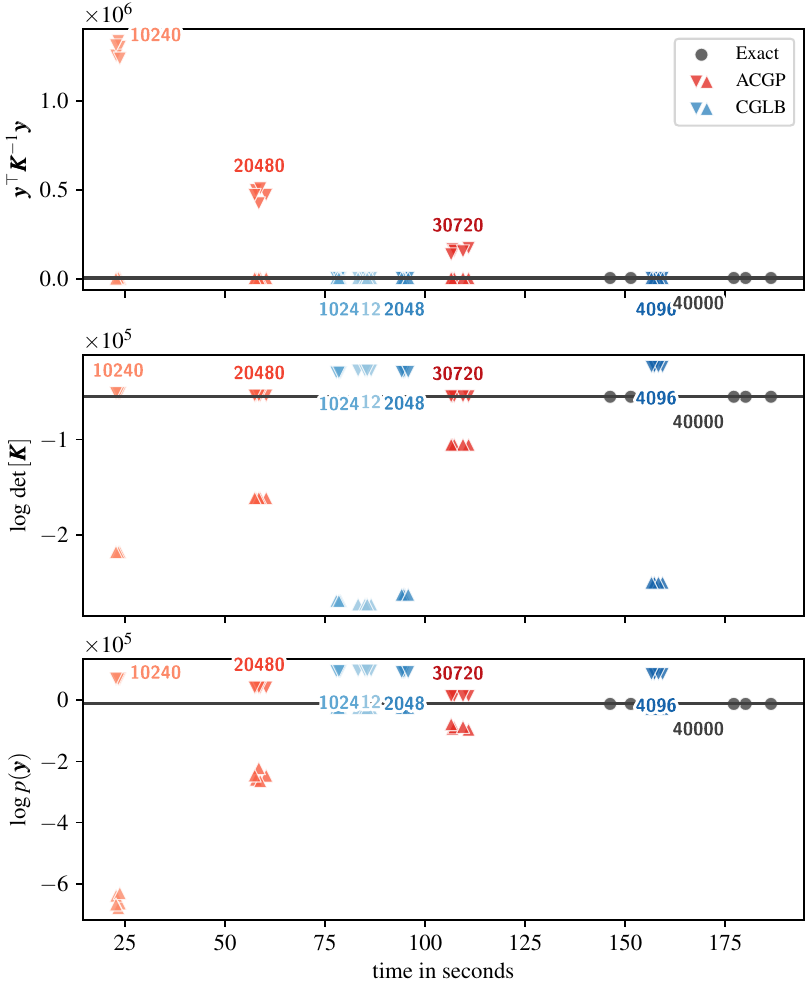}
      \subcaption{OU kernel, $\log\ell = 1$.}
      \label{subfig:kin40k_ou_1}
   \end{minipage}
   \begin{minipage}[b]{.5\textwidth}
      \centering
      \includegraphics[width=\subfigwidth]{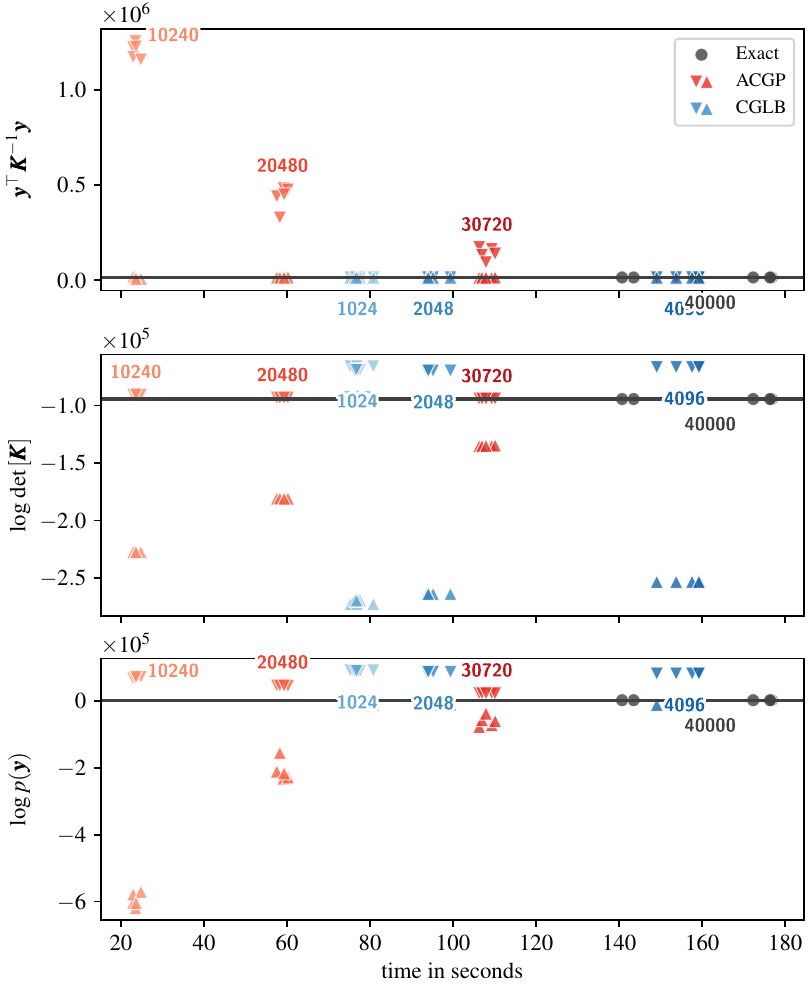}
      \subcaption{OU kernel, $\log\ell = 2$.}
      \label{subfig:kin40k_ou_2}
   \end{minipage}
   \caption{Upper and lower bounds for the \texttt{kin40k} dataset when using a Ornstein-Uhlenbeck (OU) kernel.}
	\label{fig:bounds_kin40k_ou}
\end{figure}
\clearpage

\clearpage

\section{Notation}
We use a \textsc{python}-inspired index notation, abbreviating for example $[y_1, \ldots, y_{n}]\T$ as $\vy_{:n}$---observe that the indexing starts at 1.
\todo[inline]{check that that's what we need}
Indexing binds before any other operation such that $\inv{\mat K_{:\tzero, :\tzero}}$ is the inverse of $\mat K_{:\tzero, :\tzero}$ and \emph{not} all elements up to $\tzero$ of $\inv{\mat K}$.
For $s \in \{1,\dots, N\}$ define $\mathcal{F}_s\ce\sigma(\vec x_1, y_1, \dots, \vec x_s, y_s)$ to be the $\sigma$-algebra generated by $\vec x_1, y_1, \dots, \vec x_\tzero, y_\tzero$.
With respect to the main article, we change the letter $M$ to $t$.
The motivation for the former notation is to highlight the role of the variable as a subset size, whereas in this part, the focus is on $M$ as a stopping time.

\section{Proof Sketch}
\label{sec:proof_sketch}
In this section of the appendix, we provide additional intuition on the theorems and proofs for the theory behind \acgp{}.

\subsection{The cumulative perspective}
%
Using Bayes rule, we can write $\log p(\vec y)$ equivalently as 
\begin{align}
	\label{eq:app_log_marginal_alternative}
&\log p(\bm{y}) = -\frac{1}{2}\left(\log\det{\mat K_N}+\vec y\Trans\inv{\mat K_N}\vec y+N\log(2\pi)\right)= \sum_{n=1}^N \log p(y_{n}\mid \vec y_{:n-1}).
\end{align}
For each potential stopping point $t$ we can decompose \cref{eq:app_log_marginal_alternative} into a sum of three terms:
\begin{align*}
  \log p(\vec y) &= \underbrace{\sum_{n=1}^\tzero \log p(y_{n}\mid \vec y_{:n-1})}_{A: \text{ fully processed}}
   +\!\!\! \underbrace{\sum_{n=\tzero+1}^{t}\!\!\! \log p(y_{n}\mid \vec y_{:n-1})}_{B: \text{ partially processed}}
   +\!\!\! \underbrace{\sum_{n=t+1}^{N}\!\!\! \log p(y_{n}\mid \vec y_{:n-1})}_{C: \text{ remaining}},
\end{align*}
where $\tzero<t$.
We will use the partially processed points between $\tzero$ and $t$, to obtain unbiased upper and lower bounds on the expected value of $\log p(\vec y_{\tzero+1:}\given \vec y_{:\tzero})$:
\begin{align}
 \Exp[\L_t\mid\vec x_1,y_1,\dots \vec x_{\tzero},y_{\tzero}]
 \leq A+
 \Exp[B+C \mid \vec x_1,y_1,\dots \vec x_{\tzero},y_{\tzero}]
 \leq
 \Exp[\U_t \negthickspace \mid \vec x_1,y_1,\dots \vec x_{\tzero},y_{\tzero}].
\end{align}

\subsection{General bounds}
%
The posterior of the $n$th observation conditioned on the previous is Gaussian with
\begin{align*}
  p(y_{n}\mid \vec y_{:n-1})
    =  & \N(\GPmean[n-1]{n}, \GPvar[n-1]{n}) \\
  \GPmean[n-1]{n}
    \ce& k(\vec x_n, \mat X_{:n-1})\inv{\mat K_{n-1}}\vec y_{:n-1} \\
  \postk{\vec x_n}{\vec x_n}{n-1}
    \ce& k(\vec x_n, \vec x_n)
     - k(\vec x_n, \mat X_{:n-1})\inv{\mat K_{n-1}}k(\mat X_{:n-1}, \vec x_n),
\end{align*}
where we assumed (w.l.o.g) that $\mu_0(\vec x)\ce 0$.
Inspecting these expressions one finds that
\begin{align}
\log\det{\mat K_N}&=\sum_{n=1}^N \log\left(\GPvar[n-1]{n}\right),
\\\vec y\Trans\inv{\mat K_N}\vec y&=\sum_{n=1}^N\frac{\left(y_n-\GPmean[n-1]{n}\right)^2}{\GPvar[n-1]{n}}.
\end{align}

Our strategy is to find function families $u$ (and $l$) which upper (and lower) bound the expectation
\begin{align*}
	l_{n,t}^d \leq_E \log \postk{\vec x_n}{\vec x_n}{n-1}\leq_E u_{n,t}^d\\
	l_{n,t}^q \leq_E \frac{\left(y_n-\GPmean[n-1]{n}\right)^2}{\GPvar[n-1]{n}}\leq_E u_{n,t}^q,
\end{align*}
where $\leq_E$ denotes that the inequality holds in expectation.
We will choose the function families such that the unseen variables interact only in a \emph{controlled} manner.
More specifically,
\begin{align}
\label{eq:required_form}
f^x_{n,t}(\vec x_n, y_n, \dots, \vec x_1, y_1)=\sum_{j=\tzero+1}^n g_{t}^{f,x}(\vec z_n, \vec z_j; \vec z_1, \dots \vec z_{\tzero})\notag,
\end{align}
with $f\in\{u,l\}$ and $x\in\{d,q\}$.
The effect of this restriction becomes apparent when taking the expectation.
The sum over the bounds becomes the sum of only two terms: variance and covariance, formally:
\begin{align}
&\Exp\left[\sum_{n=\tzero+1}^N f^x_{n,t}(\vec z_n, \dots, \vec z_1)\mid\sigma(\vec z_1, \dots, \vec z_{\tzero})\right]
\\&=\left(N-\tzero\right)\Exp\left[g(\vec z_{\tzero+1}, \vec z_{\tzero+1}, \vec z_{1}\dots, \vec z_n)\mid\sigma(\vec z_1, \dots, \vec z_{\tzero})\right] \notag
\\&+\left(N-\tzero\right) \frac{N-\tzero+1}{2} \Exp\left[g(\vec z_{\tzero+1}, \vec z_{\tzero+2}, \vec z_{1}\dots, \vec z_n)\mid\sigma(\vec z_1, \dots, \vec z_{\tzero})\right].
\end{align}
We can estimate this expectation from the observations we obtained between $\tzero$ and $t$.
\begin{align}
&\approx\frac{N-t}{t-\tzero}\sum_{n=\tzero+1}^t g(\vec z_n, \vec z_n, \vec z_{1}\dots, \vec z_{\tzero})
\\&+\frac{2(N-t)}{t-\tzero} \frac{N-\tzero+1}{2} \sum_{i=1}^{\frac{t-\tzero}{2}} g(\vec z_{\tzero+2i}, \vec z_{\tzero+2i-1}, \vec z_{1}\dots, \vec z_{\tzero}).  \notag
\end{align}



\subsection{Bounds on the log-determinant}
Since the posterior variance of a Gaussian process can never increase with more data, the average of the (log) posterior variances is an estimator for an upper bound on the log-determinant. 
Hence in this case, we simply ignore the interaction between the remaining variables.
We set $g(\vec x_n, \vec x_i)\ce \delta_{ni}\log\left(\GPvar{n}\right)$ where $\delta_{ni}$ denotes Kronecker's $\delta$.

To obtain a lower bound we use that for $c>0$ and $a\geq b\geq 0$, one can show that $\log\left(c+a-b\right)\geq\log\left(c+a\right)-\frac{b}{c}$ where the smaller $b$ the better the bound.
In our case $c=\noisef{\vec x_n}$, $a=\postkt{\vec x_n}{\vec x_n}$ and $b=\postkt{\vec x_n}{\estX[n-1]}\inv{\left(\postkt{\estX[n-1]}{\estX[n-1]}+\noisef{\estX[n-1]}\right)}\postkt{\estX[n-1]}{\vec x_n}$.
Underestimating the eigenvalues of $\postkt{\estX[n-1]}{\estX[n-1]}$ by 0 we obtain a lower bound, where each quantity can be estimated.
Formally, for any $\tzero\leq t$,
\begin{align}
&\log\left(\GPvar[n-1]{n}\right)\geq \left(\log\left(\GPvar{n}\right)-\sum_{i=\tzero+1}^{n-1}\frac{\postkt{\vec x_n}{\vec x_i}^2}{\noisef{\vec x_n}\noisef{\vec x_i}}\right).
\end{align}
This bound can be worse than the deterministic lower bound $\lBoundDet$.
It depends on how large $n$ is, how large the average correlation is and how small $\lBoundDet$ is.
Denote with $\mu$ the estimator for the left addend and with $\rho$ the estimator for the second addend.
We can determine the number of steps $n-\tzero$ that this bound is better by solving for the maxima of a quadratic equation:
\begin{align}
	p\left(\mu-\frac{p-1}{2}\rho\right)\geq p\lBoundDet
\end{align}
The tipping point $\psi$ is
\begin{align}
\psi\ce\max\left(N, \tzero+\left\lfloor\frac{\mu-\lBoundDet}{\rho}+\frac{1}{2}\right\rfloor\right).
\end{align}
Hence, for $n>\psi$ we set $u^d_n\ce \lBoundDet$.

Observe that, the smaller $\postkt{\vec x_j}{\vec x_{j+1}}^2$ the closer the bounds.
This term represents the correlation of datapoints conditioned on the $\tzero$ datapoints observed before.
Thus, our bounds come together, when incoming observations become independent conditioned on what was already observed.
Essentially, that $\postkt{\vec x_j}{\vec x_{j+1}}^2=0$ is the basic assumption of inducing input approximations \citep{quionero2005unifying}.

\subsection{Bounds on the quadratic form}
For an upper bound on the quadratic form we apply a similar trick:
\begin{align}
\frac{x}{c+a-b}\leq \frac{x(c+b)}{c(c+a)},
\end{align}
where $x\geq 0$.
Further we assume that in expectation the mean square error improves with more data.
\todo[inline]{the bound below holds only in expectation!}
Formally,
\begin{align}
\frac{\left(y_j-\GPmean[j-1]{j}\right)^2}{\GPvar[j-1]{j}}\leq_E
\frac{\left(y_j-\GPmean{j}\right)^2}{\noisef{\vec x_j}\left(\GPvar{j}\right)}\left(\noisef{\vec x_j}+\sum_{i=\tzero+1}^{j-1}\frac{\left(\postkt{\vec x_j}{\vec x_i}\right)^2}{\noisef{\vec x_i}}\right)
\end{align}

For a lower bound observe that
\begin{align}
\vec y\Trans\inv{\mat K}\vec y=\vec y_{:s}\Trans\inv{\mat K_{s}}\vec y_{:s}+\left(\vec y_{\tzero+1:N}-\GPmeanX{\remX}\right)\Trans\inv{\matQ[N]}\left(\vec y_{\tzero+1:N}-\GPmeanX{\remX}\right)
\end{align}
where $\matQ \ce \postkt{\estX[j]}{\estX[j]}+\noisefm{\estX[j]}$ 
with $j\geq \tzero+1$ for the posterior covariance matrix of $\estX[j]$ conditioned on $\mat X_{1:\tzero}$.
We use a trick we first encountered in \citet{kim2018scalableStructureDiscovery}: $\vec y\Trans\inv{\mat A}\vec y\geq 2\vec y\Trans\vec b-\vec b\Trans\mat A\vec b$, for any $\vec b$.
\newcommand{\meanerror}{\vec e}
For brevity introduce $\meanerror\ce \vec y_{\tzero+1:N}-\GPmeanX{\remX}$.
After applying the inequality with $\vec b\ce \inv{\Diag[\matQ[N]]}\vec e$, we obtain
\begin{align}
2 \sum_{n=\tzero+1}^N \frac{\left(y_n-\GPmean{n}\right)^2}{\GPvar{n}}
-\sum_{n,n'=\tzero+1}^N\frac{(y_n-\GPmean{n})}{\GPvar{n}}[\matQ[N]]_{nn'}\frac{(y_{n'}-\GPmean{n'})}{\GPvar{n'}}
\end{align}
which is now in the form of \cref{eq:required_form}.

Observe that, the smaller the square error $(y_j-\GPmean{j})^2$, the closer the bounds.
That is, if the model fit is good, the quadratic form can be easily identified.
\subsection{Using the Bounds for Stopping the Cholesky}
\label{sec:stopping}
\label{sec:bound_computation}
We will use the following stopping strategy:
when the difference between bounds becomes sufficiently small and their absolute value is far away from zero.
More precisely, when having deterministic bounds $\L\leq x\leq \U$ on a number $x$, with
\begin{align}
	\frac{\U-\L}{2\min(\Abs{\U},\Abs{\L})}\leq r \text{ and }
	\label{eq:stopping_condition_1}
	\\\sign{\U}=\sign{L},
	\label{eq:stopping_condition_2}
\end{align}
then the relative error of the estimate $\frac{1}{2}(\U+\L)$ is less than $r$, that is $|\frac{\frac{1}{2}(\U+\L)-x}{x}|\leq r$.

\begin{remark}
\label{remark:autodiff}
In our experiments, we do not use $\frac{1}{2}(\U+\L)$ as estimator, and instead use the biased estimator $(N-\tau)\frac{1}{\tau}\log p(\vec y_{:\tau})$.
Since stopping occurs when log-determinant and quadratic form evolve roughly linearly, the two estimators are not far off each other.
The main reason for using the biased estimator is of a technical nature; it is easier and faster to implement a custom backward function which can handle the in-place operations of our Cholesky implementation.
\end{remark}

\begin{remark}
\label{remark:estimated_correlation}
To estimate the average correlation between elements of the kernel matrix, we use all elements of the off-diagonal instead of only every second.
This has no effect on our main result, but it becomes important when developing PAC bounds.
\end{remark}

\begin{remark}
\label{remark:advancing_bounds}
The lower bound on the log-determinant, and the upper bound on the quadratic form switch their form at a step $\psi$ (\cref{thm:log_det_lower_bound,thm:quad_upper_bound}).
Currently, to prove our results, this requires $\psi$ to be $\mathcal{F}_{\tzero}$-measurable, and for that reason we use estimators using inputs only up to index $\tzero$, to define $\psi$.
However, a PAC bound proof would allow to condition on the event that the estimators (plus some $\epsilon$) overestimate their expected values with high probability.
Under that condition, we could use the true expected value (which is $\mathcal{F}_{\tzero}$-measurable) to define $\psi$.
Hence, in our practical implementation we use estimators based on inputs with indices up to $M$ to define $\psi$.
\end{remark}

The question remains how to use the bounds and stopping strategy to derive an approximation algorithm.
We transform the exact Cholesky decomposition for that purpose.
For brevity denote $\mat L_{\tzero}\ce \chol[k(\mat X_{:\tzero}, \mat X_{:\tzero})+\noisef{\mat X_{:\tzero}}]$ and $\mat T_{\tzero}\ce k(\mat X_{\tzero+1:}, \mat X_{}){\mat L_{\tzero}}\iTrans$.
For any $\tzero\in \{1,\dots N\}$:
\begin{align}
\mat L_N=
\begin{bmatrix}\mat L_{\tzero} & \mat 0\\
\mat T & \chol\left[k(\mat X_{\tzero+1:, \tzero+1:})-\mat T\mat T\Trans\right]
\end{bmatrix}
\end{align}
One can verify that $\mat L_N$ is indeed the Cholesky of $\mat K_N$ by evaluating $\mat L_N\mat L_N\Trans$.
Observe that $k(\mat X_{\tzero+1:, \tzero+1:})-\mat T\mat T\Trans$ is the posterior covariance matrix of the $\vec y_{\tzero+1:}$ conditioned on $\vec y_{:\tzero}$.
Hence, in the step before the Cholesky of the posterior covariance matrix is computed, we can estimate our log-determinant bounds.

Similar reasoning applies for solving the linear equation system.
We can write
\begin{align}
\label{eq:recursive_LES}
\vec \alpha_N=
\begin{bmatrix}
\vec \alpha_{\tzero}\\
\inv{\chol\left[k(\mat X_{\tzero+1:, \tzero+1:})-\mat T\mat T\Trans\right]}\left(\vec y_{\tzero+1:}-\mat T_{\tzero}\vec \alpha_{\tzero}\right)
\end{bmatrix}
\end{align}
Now observe that $\mat T_{\tzero}\vec \alpha_{\tzero}=\GPmeanX{\mat X_{\tzero+1:}}$.
Hence, before the solving the lower equation system (and before computing the posterior Cholesky), we can compute our bounds for the quadratic form.
There are different options to implement the Cholesky decomposition.
We use a blocked, row-wise implementation \citep{george1986parallelCholesky}. 
For a practical implementation see \cref{algo:chol_blocked} and \cref{alg:bounds}.

\begin{algorithm}
	\caption{blocked and recursive formulation of Cholesky decomposition and Gaussian elimination, augmented with our stopping conditions marked in gray.}
	\label{algo:chol_blocked}
	\begin{algorithmic}[1]
		\Procedure{\acgp{}}{$k(\cdot, \cdot)$, $\mu(\cdot)$, $\noisef{\cdot}$, $\mat X$, $\vec y$, $\blocksize$, $\maxN$}
		\State $\mat A\gets \mat 0^{\maxN\times \maxN}, \vec \alpha\gets \vec 0^{\maxN}$ \algorithmiccomment{allocate memory}
		\State $\mat A_{1:\blocksize,1:\blocksize}\gets k(\mat X_{1:\blocksize})+\noisef{\mat X_{1:\blocksize}}$
		\algorithmiccomment{initialize kernel matrix}
		\State $\vec \alpha_{1:\blocksize}\gets \vec y_{1:\blocksize}-\mu(\mat X_{1:\blocksize})$
		\algorithmiccomment{evaluate mean function for the same datapoints}
		\State $\mat A_{1:\blocksize,1:\blocksize}\gets \operatorname{chol}(\mat A_{1:\blocksize, 1:\blocksize})$
		\algorithmiccomment{call to low-level Cholesky}
		\State $\vec \alpha_{1:\blocksize}\gets \inv{\mat A_{1:\blocksize,1:\blocksize}}\vec \alpha_{1:\blocksize}$
		\algorithmiccomment{second back-substitution step}
		\State $i\gets \blocksize+1$, $j\gets \min(i+\blocksize, N)$
		\While{$i<\maxN$}
		\State $\mat A_{i:j,1:i}\gets k(\mat X_{i:j}, \mat X_{1:i})$
		\algorithmiccomment{evaluate required block-off-diagonal part of the kernel matrix}
		\State $\mat A_{i:j,1:i} \gets \mat A_{i:j,1:i}\mat A_{1:i,1:i}\iTrans$
		\algorithmiccomment{solve triangular linear equation system}
		\State $\mat A_{i:j,i:j}\gets k(\mat X_{i:j})+\noisef{\mat X_{i:j}}$
		\algorithmiccomment{evaluate required block-diagonal part of the kernel matrix}
		\State $\vec \alpha_{i:j}\gets \vec y_{i:j}-\mu(\mat X_{i:j})$
		\algorithmiccomment{evaluate mean function for the same datapoints}
		\State $\mat A_{i:j,i:j}\gets \mat A_{i:j,i:j} - \mat A_{i:j,1:i}\mat A_{i:j,1:i}\Trans$\algorithmiccomment{down-date}
		\State \algorithmiccomment{now $\mat A_{i:j,i:j}=\matQ[j]$}
		\State $\vec \alpha_{i:j}\gets \vec \alpha_{i:j}-\mat A_{i:j,1:i}\vec \alpha_{1:i}$\algorithmiccomment{now $\vec \alpha_{i:j}$ contains $\vec y_{i:j}-\GPmeanX[i]{\mat X_{i:j}}$}
		\AState{$\L, \U\gets$\textbf{EvaluateBounds}($i$, $j$)} \algorithmiccomment{costs $\O(j-i)$}
		\If{\cref{eq:stopping_condition_1,eq:stopping_condition_2} fulfilled}
		\AState{\textbf{return} estimator}
		\EndIf
		\State $\mat A_{i:j,i:j}\gets \operatorname{chol}(\mat A_{i:j,i:j})$
		\algorithmiccomment{finish computing Cholesky for data-points up to index $j$}
		\State $\vec \alpha_{i:j}\gets \inv{\mat A_{i:j,i:j}} \vec \alpha_{i:j}$
		\algorithmiccomment{finish solving linear equation system for index up to $j$}
		\State $i\gets i+m$, $j\gets \min(i+\blocksize, \maxN)$
		\EndWhile
		\algorithmiccomment{now $\mat A=\mat L$ and $\vec \alpha=\inv{\mat L}(\vec y-\mu(\mat X))$}
		\AState{\textbf{return} estimator}
		\EndProcedure
	\end{algorithmic}
\end{algorithm}
\begin{algorithm}
\caption{bound algorithm as used in our experiments. The algorithm deviates slightly from our theory. We use \cref{eq:quad_upper_bound_final} for the upper bound in the quadratic form, and we use all off-diagonal entries (instead of only every second).
}
	\label{alg:bounds}
	\begin{algorithmic}[1]
		\Procedure{EvaluateBounds}{$\tzero$, $t$}
		\State{$D\gets \sum_{j=1}^{\tzero} \log \mA_{jj}$}\algorithmiccomment{in practice we reuse the sum from the last iteration}
		\State{$Q\gets \sum_{j=1}^{\tzero} \vec \alpha_j^2$}
		\State{$\mu\gets \frac{1}{t-\tzero}\sum_{j=\tzero+1}^{t} \log \mA_{jj}$ }\algorithmiccomment{average variance of the new points conditioned on all points processed until $\tzero$}
		\State{$\U_D\gets D + (N-\tzero)\mu$}
		\State{$\rho\gets \frac{1}{t-\tzero-1}\sum_{j=\tzero+1}^{t-1} \frac{\mA_{j,j+1}^2}{\noisef{\vec x_{j}}\noisef{\vec x_{j+1}}}$}
		\algorithmiccomment{average square correlation (deviating from theory!)}
		\State{$\psi\gets\min(N, \tzero+\lfloor\frac{\mu-\lBoundDet}{\rho}+\frac{1}{2}\rfloor$}\algorithmiccomment{number of steps the probabilistic bound is better than the deterministic}
		\State{$\L_D\gets D + (\psi-\tzero)\left(\mu - \frac{\psi-\tzero-1}{2}\rho\right)+(N-\psi)\lBoundDet$}
		\State{$\mu \gets\frac{1}{t-\tzero}\sum_{j=\tzero+1}^t \frac{\vec \alpha_j^2}{\mA_{j,j}}$} 
		\algorithmiccomment{average error calibration}
		\State{$\rho \gets\max\left(0, \frac{1}{t-\tzero-1}\sum_{j=\tzero+1}^{t-1} \frac{\vec \alpha_j \vec \alpha_{j+1} \mA_{j,j+1}}{\mA_{j,j}\mA_{j+1,j+1}}\right)$} 
		\algorithmiccomment{calibrated error correlation}
		\State{$\L_Q\gets Q+\max(0, (N-\tzero)(2\mu-\rho(N-\tzero-1)))$}
		\State{$\rho \gets\frac{1}{t-\tzero-1}\sum_{j=\tzero+1}^{t-1} \frac{\vec \alpha_j^2 \mA_{j,j+1}^2}{\mA_{j,j}\noisef{\vec x_j}\noisef{\vec x_j}}$} 
		\algorithmiccomment{square error correlation}
		\State{$\hat{\mu} \gets\frac{1}{t-\tzero}\sum_{j=\tzero+1}^t \frac{\vec \alpha_j^2}{\noisef{\vec x_j}}$} 
		\algorithmiccomment{worst-case estimate for the quadratic}		
		\State{$\psi\gets\min(N, \tzero+\lfloor\frac{\mu-\hat{\mu}}{\rho}+\frac{1}{2}\rfloor$)}
		\algorithmiccomment{number of steps the bound is better than the worst-case estimate}
		\State{$\U_Q\gets Q+(\psi-\tzero)(\mu-\frac{\psi-\tzero-1}{2}\rho)+(N-\psi)\hat{\mu}$}
		\State{\textbf{return} $\L_D+\L_Q, \U_D+\U_Q$}
		\EndProcedure
	\end{algorithmic}
\end{algorithm}

\section{Assumptions}
\begin{assumption}
	\label{assume:exchangeability}
	\label{assumption:exchangeability}
	Let $(\Omega, \mathcal{F}, \Proba)$ be a probability space and let $(\vec x_j, y_j)_{j=1}^N$ be a sequence of independent and identically distributed random vectors with $\vec x:\Omega\rightarrow\Re^D$ and $y:\Omega\rightarrow\Re$.
\end{assumption}

\begin{assumption}
\label{assume:expected_quadratic_form_assumptions}
\label{assumption:expected_quadratic_form_assumptions}
For all $\tzero,i,j,t$ with $\tzero< i\leq j\leq N$ and functions $f(\vec x_j,\vec x_i;\vec x_1, \dots \vec x_{\tzero})\geq 0$
\begin{align}
	\label{eq:expected_quadratic_form_assumptions}
	\Exp\left[f(\vec x_j, \vec x_i)\left(y_j-\GPmean[j-1]{j}\right)^2\mid\mathcal{F}_{\tzero}\right]
	\leq \Exp\left[f(\vec x_j, \vec x_i)\left(y_j-\GPmean[\tzero]{j}\right)^2\mid\mathcal{F}_{\tzero}\right]
\end{align}
where $f(\vec x_j, \vec x_i)\in\left\{\frac{1}{\GPvar{j}}, \frac{\postkt{\vec x_j}{\vec x_i}^2}{(\GPvar{j})\noisef{\vec x_j}\noisef{\vec x_i}}\right\}$.
\end{assumption}
That is, we assume that in expectation the estimator improves with more data.
Note that, $f$ can not depend on any entries of $\vec y$.

\section{Main Theorem}
\label{sec:main_theorem_proof}
\todo[inline]{adapt notation}
This section restates \cref{thm:main} and connects the different proofs in the sections to follow.
\begin{theorem}
	Assume that \cref{assume:exchangeability} and \cref{assume:expected_quadratic_form_assumptions} hold.
	For any even $\blocksize\in\{2, 4, \dots, N-2\}$ and any $\tzero \in \{1, \dots, N-\blocksize\}$, the bounds defined in \cref{eq:bound_ud,thm:log_det_lower_bound,thm:quad_lower_bound,thm:quad_upper_bound} hold in expectation:
	\begin{align*}
		\Exp[\L_D\mid \mathcal{F}_{\tzero}]\leq \Exp[\log(\det{\mat K}) \mid \mathcal{F}_{\tzero}]\leq \Exp[\U_D\mid \mathcal{F}_{\tzero}] \text{ and}
		\\\Exp[\L_Q\mid \mathcal{F}_{\tzero}]\leq \Exp[\vec y\Trans\inv{\mat K}\vec y \mid \mathcal{F}_{\tzero}]\leq \Exp[\U_Q\mid \mathcal{F}_{\tzero}]\ .
	\end{align*}
\end{theorem}
\begin{proof}
Follows from \cref{thm:log_det_lower_bound,thm:quad_lower_bound,thm:quad_upper_bound,lemma:decreasing_expectation}. 
\end{proof}

\section{Proof for the Lower Bound on the Determinant}
\newcommand{\advanceL}{\psi_{\tzero}}
\begin{theorem}
	\label{thm:log_det_lower_bound}
	Assume that \cref{assume:exchangeability} holds, and that $\blocksize\in\{2,4,\dots, N\}$ is an even number.
	Set $t\ce\tzero+\blocksize$, then, for all $\tzero\in \{1,\dots,N-\blocksize\}$
	$$\Exp\left[\left.\L_D\right| \mathcal{F}_{\tzero}\right]\leq \Exp\left[\left.D_N\right| \mathcal{F}_{\tzero}\right].$$
\allowdisplaybreaks
\begin{align}
	\L_D&\ce \log\det{\mat K_{:\tzero,:\tzero}}+\left(\advanceL-\tzero\right)\left(\loglowermean-\frac{\advanceL-\tzero-1}{2}\avgcov\right)+(N-\advanceL)\lBoundDet
	\eqcomment{the lower bound}
	\\\loglowermean&\ce \frac{1}{\blocksize}\sum_{j=\tzero+1}^t \log \left(\noisef{\vec x_j}+\postkt{\vec x_j}{\vec x_j}\right)
	\eqcomment{(under-)estimate of the posterior variance conditioned on $\tzero$ points}
	\\\avgcov&\ce \frac{2}{\blocksize}\sum_{j=\frac{\tzero+1}{2}}^{\frac{t-1}{2}} \frac{\postkt{\vec x_{2j+1}}{\vec x_{2j}}^2}{\noisef{\vec x_{2j+1}}\noisef{\vec x_{2j}}}
	\eqcomment{(over-)estimate of the correlation conditioned on $\tzero$ points}
	\\\advanceL&\ce \tzero+\max p \text{ where $p\in\mathbb{N}$ is such that}
	\\&\qquad p\left(\loglowermean[t-\blocksize]-\frac{p-1}{2}\avgcov[t-\blocksize]\right)\geq p\lBoundDet
	\eqcomment{number of steps that we suspect the decrease in variance to be controllable}
	\\&=\min\left(N,\tzero+ \left\lfloor\frac{\loglowermean[t-\blocksize]-\log\sigma^2}{\avgcov[t-\blocksize]}+\frac{1}{2}\right\rfloor\right)
\end{align}
where, if $\tzero-\blocksize<1$, we set $\loglowermean[t-\blocksize]\ce\log\sigma^2$ and $\avgcov[t-\blocksize]\ce 1$.
\end{theorem}
\begin{remark}
For our proof we require $\advanceL$ to be $\mathcal{F}_{\tzero}$-measurable.
We conjecture that for a PAC bound proof this requirement can be relaxed.
Therefore, in our implementation, we use $\loglowermean[t-\blocksize]\ce \loglowermean[t]$ and $\avgcov[t-\blocksize]\ce\avgcov[t]$.
\end{remark}
\begin{proof}
	\begin{align*}
		&\Exp\left[\left.\L_D\right| \mathcal{F}_{\tzero}\right]-\Exp\left[\left.\log\det{\mat K}\right| \mathcal{F}_{\tzero}\right]=\Exp\left[\left.\L_D-\log\det{\mat K}\right| \mathcal{F}_{\tzero}\right]
		\\&= \Exp\left[\left. \left(\advanceL-\tzero\right)\left(\loglowermean-\frac{\advanceL-\tzero-1}{2}\avgcov\right)+(N-\advanceL)\lBoundDet-\sum_{j=\tzero+1}^{N}\log\left(\GPvar{j}\right)\right| \mathcal{F}_{\tzero}\right]
		\eqcomment{using the definition of $\L_t$ and slightly simplifying using \cref{lemma:log_det_as_variance}}
		\\&\leq \Exp\left[\left. \left(\advanceL-\tzero\right)\left(\loglowermean-\frac{\advanceL-\tzero-1}{2}\avgcov\right)-\sum_{j=\tzero+1}^{\advanceL}\log\left(\GPvar{j}\right)\right| \mathcal{F}_{\tzero}\right]
		\eqcomment{using that $\lBoundDet\leq \log\left(\GPvar{j}\right)$ for all $j$}
		\\&= \left(\advanceL-\tzero\right)\left(\Exp\left[\left.\log\left(\sigma^2+\postkt{\vec x_{t+1}}{\vec x_{t+1}}\right|  \mathcal{F}_{\tzero}\right)\right]-\frac{\advanceL-\tzero-1}{2}\Exp\left[\left.\frac{\postkt{\vec x_{t+1}}{\vec x_{t+2}}^2}{\noisef{\vec x_{t+1}}\noisef{\vec x_{t+2}}}\right| \mathcal{F}_{\tzero}\right]\right)
		\\&\quad-\Exp\left[\left.\sum_{j=\tzero+1}^{\advanceL}\log\left(\GPvar{j}\right)\right| \mathcal{F}_{\tzero}\right]
		\eqcomment{using \cref{assume:exchangeability}}
		\\&\leq 0
		\eqcomment{\cref{lemma:overestimating_expected_logdet}}
	\end{align*}
\end{proof}

\begin{lemma}
	\label{lemma:log_trick}
	For $c>0$ and $b\geq a\geq 0:$
	$$\log (c+b-a)\geq \log(c+b)-\frac{a}{c}$$
\end{lemma}
\begin{proof}
For $a=0$, the statement is true with equality.
We rewrite the inequality as $$\frac{a}{c}\geq \log\left(\frac{c+b}{c+b-a}\right)=\log\left(1+\frac{a}{c+b-a}\right).$$
For the case $a=b$, apply the exponential function on both sides, and the statement follows from $e^x\geq x+1$ for all $x$.
For $a\in(0, b)$, consider $f(a)\ce \log (c+b-a)+\frac{a}{c}-\log(c+b)$.
The first derivative of this function is $f'(a)=-\frac{1}{c+b-a}+\frac{1}{c}$, which is always positive for $a\in(0, b)$.
Since $f(0)=0$, we must have $f(a)\geq 0$ for all $a\in(a,b)$.
\end{proof}

\begin{lemma}
	\label{lemma:overestimating_expected_logdet}
	For all $n\geq t\geq \tzero \in\mathbb{N}$:
	\begin{align*}
		\Exp\left[\left.\sum_{j=t+1}^n\log\left(\GPvar{j}\right)\right| \mathcal{F}_{\tzero}\right]\geq (n-t)\Bigg( &\Exp\left[\left.\log\left(\noisef{\vec x_{t+1}}+\postkt{\vec x_{t+1}}{\vec x_{t+1}}\right)\right| \mathcal{F}_{\tzero}\right]\\
		&- \frac{n-t-1}{2\sigma^4}\Exp\left[\left.\postkt{\vec x_{t+1}}{\vec x_{t+2}}^2\right| \mathcal{F}_{\tzero}\right]\Bigg)
	\end{align*}
\end{lemma}
\begin{proof}
	\renewcommand{\otherX}{\overline{\mat X}}
	Introduce $\otherX_j\ce [\vec x_{\tzero+1}, \dots, \vec x_{j-1}]$ with the convention $\postkt{\vec x_{\tzero+1}}{\otherX_{\tzero+1}}\ce 0$.
	\begin{align}
		&\Exp\left[\left.\sum_{j=t+1}^n\log\GPvar[j-1]{j}\right| \mathcal{F}_{\tzero}\right]
		\\&= \Exp\left[\left.\sum_{j=t+1}^n\log\left(\noisef{\vec x_j}+\postkt{\vec x_j}{\vec x_j}-\postkt{\vec x_j}{\otherX_j}\inv{\left(\postkt{\otherX_j}{\otherX_j}+\noisef{\otherX}\right)}\postkt{\otherX_j}{\vec x_j}\right)\right| \mathcal{F}_{\tzero}\right]
		\eqcomment{\cref{lemma:post_kernel}}
		\\&\geq \Exp\left[\left.\sum_{j=t+1}^n\left(\log\left(\noisef{\vec x_j}+\postkt{\vec x_j}{\vec x_j}\right)-\frac{1}{\noisef{\vec x_j}}\postkt{\vec x_j}{\otherX_j}\inv{\left(\postkt{\otherX_j}{\otherX_j}+\noisef{\otherX_j}\right)}\postkt{\otherX_j}{\vec x_j}\right)\right| \mathcal{F}_{\tzero}\right]
		\eqcomment{\cref{lemma:log_trick}}
		\\&\geq \Exp\left[\left.\sum_{j=t+1}^n\left(\log\left(\noisef{\vec x_j}+\postkt{\vec x_j}{\vec x_j}\right)-\frac{1}{\noisef{\vec x_j}}\postkt{\vec x_j}{\otherX_j}\inv{\left(\noisef{\otherX_j}\right)}\postkt{\otherX_j}{\vec x_j}\right)\right| \mathcal{F}_{\tzero}\right]
		\eqcomment{underestimating $\postkt{\otherX_j}{\otherX_j}$ by $\mat 0$}
		\\&\geq \Exp\left[\left.\sum_{j=t+1}^n\left(\log\left(\noisef{\vec x_j}+\postkt{\vec x_j}{\vec x_j}\right)-\frac{1}{\noisef{\vec x_j}}\sum_{i=t+1}^{j-1}\frac{\postkt{\vec x_j}{\vec x_j}^2}{\noisef{\vec x_i}}\right)\right| \mathcal{F}_{\tzero}\right]
		\eqcomment{writing the vector multiplication as sum}
		\\&=(n-t)\Exp\left[\left.\log\left(\sigma^2+\postkt{\vec x_{t+1}}{\vec x_{t+1}}\right)\right| \mathcal{F}_{\tzero}\right]
		-\frac{(n-t)(n-t-1)}{2}\Exp\left[\left.\frac{\postkt{\vec x_{t+1}}{\vec x_{t+2}}^2}{\noisef{\vec x_{t+1}}\noisef{\vec x_{t+2}}}\right| \mathcal{F}_{\tzero}\right]\notag
		\eqcomment{using \cref{assume:exchangeability} and then applying \cref{lemma:little_gauss}}
	\end{align}
\end{proof}

\section{Proof for the Upper Bound on the Quadratic Form}
\newcommand{\upperq}[1][t]{\mu_{#1}}
\newcommand{\badupperq}[1][t]{\overline{\mu}_{#1}}
\begin{theorem}
\label{thm:quad_upper_bound}
Assume that \cref{assume:exchangeability} and \cref{assume:expected_quadratic_form_assumptions} hold.
Let $\blocksize \in \mathbb{N}$ be even, then for all $\tzero \in \{1,\dots, N-m\}$
$$\Exp[\vec y\Trans \inv{\mat K}\vec y\mid\mathcal{F}_s]\leq \Exp[\U_Q\mid\mathcal{F}_s]\ ,$$
where
\begin{align}
	\U_Q&\ce \q{\tzero}+\left(\psi_{\tzero}-\tzero\right)\left(\upperq+\frac{\psi_{\tzero}-\tzero-1}{2}\avgcov\right)+(N-\psi_{\tzero})\badupperq
	\eqcomment{the upper bound}
	\\\upperq&\ce \frac{1}{t-\tzero}\sum_{j=\tzero+1}^t \frac{(y_j-\GPmean{j})^2}{\GPvar{j}}
	\\\avgcov&\ce \frac{2}{t-\tzero}\sum_{j=\frac{\tzero+2}{2}}^{\frac{t}{2}} \frac{(y_{2j}-\GPmean{2j})^2\postkt{\vec x_{2j}}{\vec x_{2j-1}}^2}{\left(\GPvar{2j}\right)\noisef{\vec x_{2j}}\noisef{\vec x_{2j-1}}}
	\\\badupperq&\ce \frac{1}{t-\tzero}\sum_{j=\tzero+1}^t \frac{(y_j-\GPmean{j})^2}{\noisef{\vec x_j}}
	\\\psi_{\tzero}&\ce \min\left(N,\tzero+ \left\lfloor\frac{\badupperq[t-\blocksize]-\upperq[t-\blocksize]}{\avgcov[t-\blocksize]}+\frac{1}{2}\right\rfloor\right)
	\ .
\end{align}
where, if $\tzero-\blocksize<1$, we set $\badupperq[t-\blocksize]=\upperq[t-\blocksize]\ce 0$ and $\avgcov[t-\blocksize]\ce 1$.
\end{theorem}
\begin{proof}
	\begin{align*}
		&\Exp\left[\vec y\Trans\inv{\mat K}\vec y\mid\mathcal{F}_{\tzero}\right]-\Exp\left[\U_Q\mid\mathcal{F}_{\tzero}\right]
		\\&=\Exp\left[\sum_{j=\tzero+1}^N \frac{\left(y_j-\GPmean[j-1]{j}\right)^2}{\GPvar[j-1]{j}}-\left(\psi_{\tzero}-\tzero\right)\left(\upperq+\frac{\psi_{\tzero}-\tzero-1}{2}\avgcov\right)+(N-\psi_{\tzero})\badupperq\mid\mathcal{F}_{\tzero}\right]
		\eqcomment{using the definition of $\U_Q$ and slightly simplifying with \cref{thm:quad_form_sum}}
		\\&= \Exp\left[\sum_{j=\tzero+1}^{N} \frac{\left(y_j-\GPmean[j-1]{j}\right)^2}{\GPvar[j-1]{j}}\mid\mathcal{F}_{\tzero}\right]
		-(\psi_{\tzero}-\tzero)\Exp\left[\frac{(y_{\tzero+1}-\GPmean{\tzero+1})^2}{\GPvar{\tzero+1}}\mid\mathcal{F}_{\tzero}\right]\\
		&\quad-(\psi_{\tzero}-\tzero)\frac{\psi_{\tzero}-\tzero-1}{2}\Exp\left[\frac{(y_{\tzero+1}-\GPmean{\tzero+1})^2\postkt{\vec x_{\tzero+1}}{\vec x_{\tzero+2}}^2}{\left(\GPvar{\tzero+1}\right)\noisef{\vec x_{\tzero+2}}}\mid\mathcal{F}_{\tzero}\right]\\
		&\quad-(N-\psi_{\tzero})\Exp\left[\frac{(y_{\tzero+1}-\GPmean{\tzero+1})^2}{\noisef{\vec x_{\tzero+1}}}\mid\mathcal{F}_{\tzero}\right]
		\eqcomment{using \cref{assume:exchangeability}}
		\\&\leq 0
		\eqcomment{\cref{lemma:overestimating_expected_qform} with $n=\psi_{\tzero}$ and $t=\tzero$, and \cref{lemma:bad_quad_upper_bound} with $n=N$}
	\end{align*}
\end{proof}

\begin{lemma}
\label{lemma:bad_quad_upper_bound}
For all $\psi, n, \tzero \in\mathbb{N}$ with $\tzero\leq\psi\leq n$:
\begin{align*}
	&\Exp\left[\sum_{j=\psi+1}^{n} \frac{\left(y_j-\GPmean[j-1]{j}\right)^2}{\GPvar[j-1]{j}}\mid\mathcal{F}_{\tzero}\right]\leq (n-\psi)\Exp\left[\frac{(y_{\tzero+1}-\GPmean{\tzero+1})^2}{\noisef{\vec x_{\tzero+1}}}\mid\mathcal{F}_{\tzero}\right]
\end{align*}
\end{lemma}
\begin{proof}
\begin{align}
\Exp\left[\sum_{j=\psi+1}^{n} \frac{\left(y_j-\GPmean[j-1]{j}\right)^2}{\GPvar[j-1]{j}}\mid\mathcal{F}_{\tzero}\right]&\leq \Exp\left[\sum_{j=\psi+1}^{n} \frac{\left(y_j-\GPmean[j-1]{j}\right)^2}{\noisef{\vec x_j}}\mid\mathcal{F}_{\tzero}\right]
\eqcomment{the posterior variance cannot fall below $\noisef{\vec x_j}$}
\\&\leq \Exp\left[\sum_{j=\psi+1}^{n} \frac{\left(y_j-\GPmean[\tzero]{j}\right)^2}{\noisef{\vec x_j}}\mid\mathcal{F}_{\tzero}\right]
\eqcomment{by \cref{assumption:expected_quadratic_form_assumptions}}
\\&=(n-\psi)\Exp\left[\frac{\left(y_{\tzero+1}-\GPmean[\tzero]{\tzero+1}\right)^2}{\noisef{\vec x_{\tzero+1}}}\mid\mathcal{F}_{\tzero}\right]
\eqcomment{using \cref{assume:exchangeability}}
\end{align}
\end{proof}

\begin{lemma}
	\label{lemma:fraction_trick2}
	For $c>0$, $b,x\geq0$ and $a\geq b$:
	$$\frac{x}{c+a-b}\leq \frac{x}{c}\left(1-\frac{a-b}{c+a}\right)=\frac{x(c+b)}{c(c+a)}$$
\end{lemma}
\begin{proof}
\allowdisplaybreaks
	\begin{align*}
		\frac{x}{c+a-b}&=\frac{x}{c}\left(\frac{c}{c+a-b}\right)
		\\&=\frac{x}{c}\left(1-\frac{a-b}{c+a-b}\right)
		\\&\leq \frac{x}{c}\left(1-\frac{c+a-b}{c+a}\frac{a-b}{c+a-b}\right)
		\eqcomment{since $\frac{c+a-b}{c+a}\leq 1$}
		\reqcomment{I wonder if using $\frac{a-b}{a}$ could lead to a tighter bound.}
		\\&=\frac{x}{c}\left(1-\frac{a-b}{c+a}\right)
		\eqcomment{cancelling terms}
	\end{align*}
\end{proof}

\begin{lemma}
\label{lemma:overestimating_expected_qform}
For all $\tzero,t, n \in\mathbb{N}$ with $n\geq t\geq \tzero$:
\begin{align*}
	&\Exp\left[\sum_{j=t+1}^n\frac{(y_j-\GPmean[j-1]{j})^2}{\GPvar[j-1]{j}}\mid\mathcal{F}_{\tzero}\right]
	\\&\leq (n-t)\left(\Exp\left[\frac{(y_{\tzero+1}-\GPmean{\tzero+1})^2}{\GPvar{\tzero+1}}\mid\mathcal{F}_{\tzero}\right]\right)\\
	&+(n-t)\left(\left(\frac{n+t+1}{2}-\tzero\right)\Exp\left[\frac{(y_{\tzero+1}-\GPmean{\tzero+1})^2\postkt{\vec x_{\tzero+1}}{\vec x_{\tzero+2}}^2}{\left(\GPvar{\tzero+1}\right)\noisef{\vec x_{\tzero+2}}\noisef{\vec x_{\tzero+1}}}\mid\mathcal{F}_{\tzero}\right]\right)
\end{align*}
\end{lemma}

\begin{proof}
	\allowdisplaybreaks
	Introduce $\otherX_j\ce [\vec x_{\tzero+1}, \dots, \vec x_{j-1}]$ with the convention $\postkt{\vec x_{\tzero+1}}{\otherX_{\tzero+1}}\ce 0$.
	\begin{align*}
		&\Exp\left[\sum_{j=t+1}^n\frac{(y_j-\GPmean[j-1]{j})^2}{\GPvar[j-1]{j}}\mid\mathcal{F}_{\tzero}\right]
		\\&= \Exp\left[\sum_{j=t+1}^n\frac{(y_j-\GPmean[j-1]{j})^2}{\noisef{\vec x_j}+\postkt{\vec x_j}{\vec x_j}-\postkt{\vec x_j}{\otherX_j}\inv{\left(\postkt{\otherX_j}{\otherX_j}+\noisef{\otherX}\right)}\postkt{\otherX_j}{\vec x_j}}\mid\mathcal{F}_{\tzero}\right]
		\eqcomment{\cref{lemma:post_kernel}}
		\\&\leq \Exp\left[\sum_{j=t+1}^n\frac{(y_j-\GPmean[j-1]{j})^2\left(\noisef{\vec x_j}+\postkt{\vec x_j}{\otherX_j}\inv{\left(\postkt{\otherX_j}{\otherX_j}+\noisef{\otherX}\right)}\postkt{\otherX_j}{\vec x_j}\right)}{\noisef{\vec x_j}\left(\noisef{\vec x_j}+\postkt{\vec x_j}{\vec x_j}\right)}\mid\mathcal{F}_{\tzero}\right]
		\eqcomment{\cref{lemma:fraction_trick2}}
		\\&\leq \Exp\left[\sum_{j=t+1}^n\frac{(y_j-\GPmean[j-1]{j})^2\left(\noisef{\vec x_j}+\postkt{\vec x_j}{\otherX_j}\inv{\left(\noisef{\otherX}\right)}\postkt{\otherX_j}{\vec x_j}\right)}{\noisef{\vec x_j}\left(\noisef{\vec x_j}+\postkt{\vec x_j}{\vec x_j}\right)}\mid\mathcal{F}_{\tzero}\right]
		\eqcomment{underestimating the eigenvalues of $\postkt{\otherX}{\otherX}$ by $0$}
		\\&= \Exp\left[\sum_{j=t+1}^n\frac{(y_j-\GPmean[j-1]{j})^2\left(\noisef{\vec x_j}+\sum_{i=\tzero+1}^{j-1}\frac{\postkt{\vec x_j}{\vec x_i}^2}{\noisef{\vec x_i}}\right)}{\noisef{\vec x_j}\left(\noisef{\vec x_j}+\postkt{\vec x_j}{\vec x_j}\right)}\mid\mathcal{F}_{\tzero}\right]
		\eqcomment{writing the vector-product explicitly as a sum}
		\\&=\sum_{j=t+1}^n\left(\Exp\left[\frac{(y_j-\GPmean[j-1]{j})^2}{\noisef{\vec x_j}+\postkt{\vec x_j}{\vec x_j}}\mid\mathcal{F}_{\tzero}\right]+\sum_{i=\tzero+1}^{j-1}\Exp\left[\frac{(y_j-\GPmean[j-1]{j})^2\postkt{\vec x_j}{\vec x_i}^2}{\left(\GPvar{j}\right)\noisef{\vec x_i}\noisef{\vec x_j}}\mid\mathcal{F}_{\tzero}\right]\right)
		\eqcomment{linearity of expectation}
		\\&=\sum_{j=t+1}^n\left(\Exp\left[\frac{(y_j-\GPmean{j})^2}{\noisef{\vec x_j}+\postkt{\vec x_j}{\vec x_j}}\mid\mathcal{F}_{\tzero}\right]+\sum_{i=\tzero+1}^{j-1}\Exp\left[\frac{(y_j-\GPmean{j})^2\postkt{\vec x_j}{\vec x_i}^2}{\left(\GPvar{j}\right)\noisef{\vec x_i}\noisef{\vec x_j}}\mid\mathcal{F}_{\tzero}\right]\right)
		\eqcomment{by assumption \cref{eq:expected_quadratic_form_assumptions}}
		\\&=\sum_{j=t+1}^n\left(\Exp\left[\frac{(y_{\tzero+1}-\GPmean{\tzero+1})^2}{\noisef{\vec x_{\tzero+1}}+\postkt{\vec x_{\tzero+1}}{\vec x_{\tzero+1}}}\mid\mathcal{F}_{\tzero}\right]+\sum_{i=\tzero+1}^{j-1}\Exp\left[\frac{(y_{\tzero+1}-\GPmean{\tzero+1})^2\postkt{\vec x_{\tzero+1}}{\vec x_{\tzero+2}}^2}{\left(\GPvar{\tzero+1}\right)\noisef{\vec x_{\tzero+2}}\noisef{\vec x_{\tzero+1}}}\mid\mathcal{F}_{\tzero}\right]\right)
		\eqcomment{using \cref{assume:exchangeability}}
		\\&=(n-t)\left(\Exp\left[\frac{(y_{\tzero+1}-\GPmean{\tzero+1})^2}{\noisef{\vec x_{\tzero+1}}+\postkt{\vec x_{\tzero+1}}{\vec x_{\tzero+1}}}\mid\mathcal{F}_{\tzero}\right]\right)\\
		&+(n-t)\left(\left(\frac{n+t-1}{2}-\tzero\right)\Exp\left[\frac{(y_{\tzero+1}-\GPmean{\tzero+1})^2\postkt{\vec x_{\tzero+1}}{\vec x_{\tzero+2}}^2}{\left(\GPvar{\tzero+1}\right)\noisef{\vec x_{\tzero+2}}\noisef{\vec x_{\tzero+1}}}\mid\mathcal{F}_{\tzero}\right]\right)
		\eqcomment{by \cref{lemma:little_gauss}}
	\end{align*}
\end{proof}

\begin{remark}
\label{remark:better_quad_bound}
Similar to the proof of \cref{thm:log_det_lower_bound}, we can improve the bound by monitoring how many steps the sum of average correlations is below the average variance.
More precisely, we solve for the largest $p\in [0,N-\tzero]$ such that
$$\upperq[\tzero]+\frac{p-1}{2}\avgcov[\tzero]\leq \frac{1}{m}\sum_{j=\tzero-m+1}^{\tzero}\frac{(y_j-\GPmean{\vec x_j})^2}{\noisef{\vec x_j}}\text{,}$$
and replace the upper bound by
\begin{align}
\label{eq:quad_upper_bound_final}
\U_t\ce \q{\tzero}+p\left(\upperq+\frac{p-1}{2}\avgcov\right)+\frac{N-p-\tzero}{t-\tzero}\sum_{j=\tzero+1}^t \frac{(y_j-\GPmean{\vec x_j})^2}{\noisef{\vec x_j}}\text{.}
\end{align}
\end{remark}

\section{Proof for the Lower Bound on the Quadratic Form}
\newcommand{\postA}{\left(\postkt{\mat X_{\tzero+1:}}{\mat X_{\tzero+1:}}+\noisefm{\mat X_{\tzero+1:}}\right)}

\newcommand{\seenX}{{\underline{\mat X}_t}}
\newcommand{\remy}{\vec y_{\tzero+1:}}
\newcommand{\signa}{\mat S}  


\begin{theorem}
\label{thm:quad_lower_bound}
Assume that \cref{assume:exchangeability} holds. 
Let $\blocksize\in\{2,\dots, N-2\}$ be an even number less than $N$.
For $\tzero\in\{1,\dots,N-\blocksize\}$,
$$\Exp\left[\left.\L_Q\right|\mathcal{F}_{\tzero}\right] \leq \Exp\left[\left.\vec y\Trans\inv{\mat K}\vec y\right|\mathcal{F}_{\tzero}\right]$$
where
\begin{align}
\L_Q&\ce \q{\tzero}+(N-\tzero)\left(\mu_t-(N-\tzero-1)\max(0, \rho_t)\right)
\\\mu_t&\ce \frac{1}{t-\tzero}\sum_{j=\tzero+1}^t \frac{(y_j-\GPmean{j})^2}{\GPvar{j}}
\\\rho_t&\ce \frac{2}{t-\tzero}\sum_{j=\frac{\tzero+2}{2}}^{\frac{t}{2}} \frac{(y_{2j}-\GPmean{2j})(y_{2j-1}-\GPmean{2j-1})\postkt{\vec x_{2j}}{\vec x_{2j-1}}}{(\GPvar{2j})(\GPvar{2j-1})}
\end{align}
\end{theorem}

\begin{proof}
\begin{align*}
	&\Exp\left[\L_Q\mid\mathcal{F}_{\tzero}\right]-\Exp\left[\vec y\Trans\inv{\mat K}\vec y\mid\mathcal{F}_{\tzero}\right]
	\\&=\Exp\left[(N-\tzero)\left(\mu_t-(N-\tzero-1)\max(0, \rho_t)\right)-\sum_{j=\tzero+1}^N\frac{\left(y_j-\GPmean[j-1]{j}\right)^2}{\GPvar[j-1]{j}}\mid\mathcal{F}_{\tzero}\right]
	\eqcomment{using the definition of $\L_Q$ and slightly simplifying}
	\\&\leq\Exp\left[(N-\tzero)\left(\mu_t-(N-\tzero-1)\rho_t\right)-\sum_{j=\tzero+1}^N\frac{\left(y_j-\GPmean[j-1]{j}\right)^2}{\GPvar[j-1]{j}}\mid\mathcal{F}_{\tzero}\right]
	\eqcomment{allowing $\rho$ to be negative increases the lower bound}
	\\&=(N-\tzero)\Exp\left[\left.\frac{(y_{\tzero+1}-\GPmean{\tzero+1})^2}{\GPvar{\tzero+1}}\right|\mathcal{F}_{\tzero}\right]
	\\&\quad
	-(N-\tzero)(N-\tzero-1)\Exp\left[\left.\frac{(y_{\tzero+1}-\GPmean{\tzero+1})(y_{\tzero+2}-\GPmean{\tzero+2})\postkt{\vec x_{\tzero+1}}{\vec x_{\tzero+2}}}{(\GPvar{\tzero+1})(\GPvar{\tzero+2})}\right|\mathcal{F}_{\tzero}\right]
	\eqcomment{using \cref{assume:exchangeability}}
	\\&\leq 0
	\eqcomment{using \cref{thm:underestimating_expected_qform}}
\end{align*}
\end{proof}

\allowdisplaybreaks
\begin{lemma}
\label{thm:underestimating_expected_qform}
For all $\mathcal{F}_{\tzero}$-measurable $\alpha\in\Re:$
\begin{align}
\Exp\left[\vec y\Trans \inv{\mat K}\vec y \mid \mathcal{F}_{\tzero} \right]&\geq \q{\tzero}+\alpha(2-\alpha)(N-\tzero)\Exp\left[\left.\frac{(y_{\tzero+1}-\GPmean{\tzero+1})^2}{\GPvar{\tzero+1}}\right| \mathcal{F}_{\tzero}\right] \notag
\\&\quad-\alpha^2(N-\tzero)(N-\tzero-1)\Exp\left[\left.\frac{(y_{\tzero+1}-\GPmean{\tzero+1})(y_{\tzero+2}-\GPmean{\tzero+2})\postkt{\vec x_{\tzero+1}}{\vec x_{\tzero+2}}}{(\GPvar{\tzero+1})(\GPvar{\tzero+2})}\right| \mathcal{F}_{\tzero}\right]
\end{align}
\end{lemma}
\allowdisplaybreaks[0]
\begin{proof}
Using \cref{thm:quad_form_sum}, we can write the quadratic form as a sum of two quadratic forms:
\begin{align}
	\vec y\Trans \inv{\mat K}\vec y&=\q{\tzero} +(\remy-\GPmeanX{\mat X_{\tzero+1:}})\Trans\inv{\postA}(\remy-\GPmeanX{\mat X_{\tzero+1:}}).
\end{align}
For the right-hand addend, we use a trick we first encountered in \citet{kim2018scalableStructureDiscovery}: $\vec a\Trans\inv{\mat A}\vec a\geq 2\vec a\Trans\vec b-\vec b\Trans\mat A\vec b$, for any $\vec b$ given $\mat A$ is symmetric and positive definite.
Define $\vec e \ce (\remy-\GPmeanX{\mat X_{\tzero+1:}})$, $\mat D\ce \Diag[\postA]$ and choose $\vec b\ce \alpha\inv{\mat D}\vec e$.
\begin{align}
&\Exp\left[\vec y\Trans \inv{\mat K}\vec y \mid \mathcal{F}_{\tzero} \right]
\\&=\q{\tzero} +\Exp\left[(\remy-\GPmeanX{\mat X_{\tzero+1:}})\Trans\inv{\postA}(\remy-\GPmeanX{\mat X_{\tzero+1:}})\mid \mathcal{F}_{\tzero} \right]
\eqcomment{since $\q{\tzero}$ is $\mathcal{F}_{\tzero}$-measurable}
\\&\geq \q{\tzero}+2\alpha \Exp\left[\vec e\Trans \inv{\mat D}\vec e\mid\mathcal{F}_{\tzero}\right]-\alpha^2 \Exp\left[\vec e\Trans \inv{\mat D}\postA \inv{\mat D}\vec e\mid \mathcal{F}_{\tzero}\right]
\eqcomment{applying the inequality for quadratic forms and using the $\mathcal{F}_{\tzero}$-measurabilituy of $\alpha$}
\\&=\q{\tzero}+2\alpha \sum_{j=\tzero+1}^N \Exp\left[\left.\frac{(y_{j}-\GPmean{j})^2}{\GPvar{j}}\right| \mathcal{F}_{\tzero}\right]\notag
\\&\quad-\alpha^2\sum_{j=\tzero+1}^N\sum_{i=\tzero+1}^N\Exp\left[\left.\frac{(y_{j}-\GPmean{j})(y_{i}-\GPmean{i})\postkt{\vec x_{j}}{\vec x_{i}}+\delta_{ij}\noisef{\vec x_{j}}}{(\GPvar{j})(\GPvar{i})}\right| \mathcal{F}_{\tzero}\right]
\eqcomment{writing the vector products as sums}
\\&=\q{\tzero}+2\alpha (N-\tzero)\Exp\left[\left.\frac{(y_{\tzero+1}-\GPmean{\tzero+1})^2}{\GPvar{\tzero+1}}\right|\mathcal{F}_{\tzero}\right] \notag
\\&\quad-\alpha^2(N-\tzero)\Exp\left[\left.\frac{(y_{\tzero+1}-\GPmean{\tzero+1})^2}{\GPvar{\tzero+1}}\right|\mathcal{F}_{\tzero}\right] \notag
\\&\quad-\alpha^2\left((N-\tzero)^2-(N-\tzero)\right)\Exp\left[\left.\frac{(y_{\tzero+1}-\GPmean{\tzero+1})(y_{\tzero+2}-\GPmean{\tzero+2})\postkt{\vec x_{\tzero+1}}{\vec x_{\tzero+2}}}{(\GPvar{\tzero+1})(\GPvar{\tzero+2})}\right|\mathcal{F}_{\tzero}\right]
\eqcomment{using \cref{assume:exchangeability}, grouping variance and covariance terms separately}
\\&=\q{\tzero}+\alpha (2-\alpha)(N-\tzero)\Exp\left[\left.\frac{(y_{\tzero+1}-\GPmean{\tzero+1})^2}{\GPvar{\tzero+1}}\right|\mathcal{F}_{\tzero}\right]
\\&\quad-\alpha^2(N-\tzero)(N-\tzero-1)\Exp\left[\left.\frac{(y_{\tzero+1}-\GPmean{\tzero+1})(y_{\tzero+2}-\GPmean{\tzero+2})\postkt{\vec x_{\tzero+1}}{\vec x_{\tzero+2}}}{(\GPvar{\tzero+1})(\GPvar{\tzero+2})}\right|\mathcal{F}_{\tzero}\right]
\eqcomment{simplifying}
\end{align}
\end{proof}

\section{Utility Proofs}
\begin{lemma}[Bounding the relative error (Lemma 15 in \anonymized{\citet{bartels2021stoppedcholesky}})]
	\label{lemma:rel_err_bound}
	Let $D, \hat{D}\in [\L, \U]$, and assume $\sign(\L)=\sign(\U)\neq 0$.
	Then the relative error of the estimator $\hat{D}$ can be bounded as
	\[
	\frac{\smallabs{D-\hat{D}}}{\abs{D}}\leq \frac{\max(\U-\hat{D}, \hat{D}-\L)}{\min(\abs{\L}, \abs{\U})}
	\, .
	\]
\end{lemma}
\begin{proof}
	First observe that if $D_N>\hat{D}$ then $\smallabs{D_N-\hat{D}}=D_N-\hat{D}\leq \U-\hat{D}$.
	If $D_N\leq \hat{D}$, then $\smallabs{D_N-\hat{D}}=\hat{D}-D_N\leq \hat{D}-\L$.
	Hence, $$\smallabs{D_N-\hat{D}}\leq \max(\U-\hat{D}, \hat{D}-\L).$$
	Case $\L>0$: In this case $\smallabs{D_N}=D_N\geq \L=\smallabs{\L}$, and we obtain for the relative error:
	\begin{align*}
		\frac{\max(\U-\hat{D}, \hat{D}-\L)}{\smallabs{D_N}}&\leq \frac{\max(\U-\hat{D}, \hat{D}-\L)}{\smallabs{\L}}\, .
	\end{align*}

	Case $\U<0$: In that case $\smallabs{\L}\geq \smallabs{D_N}\geq \smallabs{\U}$, and the relative error can be bounded as follows.
	\begin{align*}
		\frac{\max(\U-\hat{D}, \hat{D}-\L)}{\smallabs{D_N}}&\leq\frac{\max(\U-\hat{D}, \hat{D}-\L)}{\smallabs{\U}}
	\end{align*}
	Since we assumed $\sign(\L)=\sign(\U)$ these were all cases that required consideration.
	Combining all observations yields
	\begin{align*}
		\frac{\smallabs{D_N-\hat{D}}}{\smallabs{D_N}}
		&\leq \max(\U-\hat{D}, \hat{D}-\L)\max\left(\frac{1}{\smallabs{\U}}, \frac{1}{\smallabs{\L}}\right)
		\\&= \frac{\max(\U-\hat{D}, \hat{D}-\L)}{\min(\smallabs{\U}, \smallabs{\L})}
	\end{align*}
\end{proof}
\todo[inline]{introduce Corollary for $\frac{1}{2}\L+\frac{1}{2}\U$}

\begin{lemma}
	\label{lemma:log_det_as_variance}
	The $\log$-determinant of a kernel matrix can be written as a sum of conditional variances.

	\begin{align}
		\log\det{\mat K} = \sum_{j=1}^N \log(\GPvar[j-1]{j})
	\end{align}
\end{lemma}
\begin{proof}
	Denote with $\mat L$ the Cholesky decomposition of $\mat K$.
	Then we obtain
	\begin{align}
		\log\det{\mat K}&=\log\det{\mat L\mat L\Trans}
		\eqcomment{using $\mat A=\mat L\mat L\Trans$}
		\\&=\log\left(\det{\mat L}\det{\mat L\Trans}\right)
		\eqcomment{for square matrices $\mat B, \mat C$: $\det{\mat B \mat C}=\det{\mat B}\det{\mat C}$}
		\\&=\log\left(\prod_{j=1}^N\mat L_{jj}^2\right)
		\eqcomment{for triangular matrices the determinant is the product of the diagonal elements}
		\\&=\sum_{j=1}^N 2\log \mat L_{jj}
		\eqcomment{property of $\log$}
	\end{align}
	With \cref{lemma:cholesky_and_gp_variance} the result follows.
\end{proof}

\begin{lemma}[The $f_j$s are decreasing in expectation (Lemma 7 in \anonymized{\citet{bartels2021stoppedcholesky}})]%
\label{lemma:decreasing_expectation}%
\newcommand{\vjinv}{c}
\newcommand{\pj}[1][j]{q_{#1}(\vec x)}%
\newcommand{\kj}[1][j]{\vec k_{#1}(\vec x)}
\newcommand{\kjx}[2]{\vec k_{#1}(\vec #2)}
Assume $\vec x_1, \dots, \vec x_N \in \mathbb{X}$ are independent and identically distributed.
Denote with $\Proba$ the law of the $\vec x_1,\ldots, \vec x_N$ and with ~$\mat L$ the Cholesky decomposition of $\mat K$.
Define the probability space $(\mathbb{X}, \sigma(\vec x_1, \ldots, \vec x_N), \Proba)$ and the canonical filtration $\mathcal{F}_j\ce \sigma(\vec x_1,\dots, \vec x_j)$ for $j=1, \ldots, N$.
Then the $v_j$ decrease in conditional expectation, that is,
\[
\Exp[\log v_{j+1}\mid \sigma(\vec x_1, \ldots, \vec x_{j})]\leq \Exp[\log v_j\mid \sigma(\vec x_1, \ldots, \vec x_{j-1})]
\, ,
\]
where
\begin{align}
	\kj[n] & \ce [k(\vec x, \vec x_1), \ldots, k(\vec x, \vec x_n)]\Trans\in\Re^n \text{ ,}\\
	\vec k_{n+1} & \ce \vec k_{n}(\vec x_{n+1})\in \Re^n \text{ and}\\
	v_n & \ce k(\vec x_n, \vec x_n) + \sigma^2 - \vec k_n\Trans \inv{(\mat K_{n-1}+\sigma^2\mat I_{n-1})}\vec k_n\, .
\end{align}%
\end{lemma}

\begin{proof}
\newcommand{\vjinv}{c} 
\newcommand{\pj}[1][j]{q_{#1}(\vec x)}
\newcommand{\kj}[1][j]{\vec k_{#1}(\vec x)} 
\newcommand{\kjx}[2]{\vec k_{#1}(\vec #2)} 
Denote with $\mathbb{Q}_j(\dx{\vec x})\ce \proba{\dx{\vec x}\mid \vec x_1, \dots, \vec x_j}$, the regular conditional probability.
Define the shorthand 
$\pj\ce\kj\Trans\inv{(\mat K_{j}+\sigma^2\mat I)}\kj$.
We will show later in the proof, in Eq.~\eqref{eq:pj_decrease}, that $\pj=\pj[j-1]+r_{j}(\vec x)$ where $r_{j}(\vec x)\geq 0$.
Taking Eq.~\eqref{eq:pj_decrease} as granted for now, we can show the claim as follows.
\begin{align}
	\Exp[\log v_{j+1}\mid \sigma(\vec x_1, \ldots, \vec x_j)]
	&=\Exp[\log \mat L_{j+1,j+1}^2\mid \sigma(\vec x_1, \ldots, \vec x_j)]
	\eqcomment{definition of $f_j$}
	\\&=\int \log\left(k(\vec x, \vec x)+\sigma^2-\kj\Trans(\vec K_j+\sigma^2\vec I)^{-1}\kj\right)\ \mathbb{Q}_{j}(\dx{\vec x}) \eqcomment{property of conditional expectation}
	\\&=\int \log\left(k(\vec x, \vec x)+\sigma^2-\pj\right)\ \mathbb{Q}_{j}(\dx{\vec x}) \eqcomment{definition of $\pj$}
	\\&=\int \log\left(k(\vec x, \vec x)+\sigma^2-\pj[j-1] - r_{j}(\vec x)\right)\ \mathbb{Q}_{j}(\dx{\vec x}) \eqcomment{using Eq.~\eqref{eq:pj_recursion}}
	\\&\leq \int \log\left(k(\vec x, \vec x)+\sigma^2-\pj[j-1]\right)\ \mathbb{Q}_{j}(\dx{\vec x}) \eqcomment{using Eq.~\eqref{eq:pj_decrease} and monotonicity of the logarithm}
	\\&=\int \log\left(k(\vec x, \vec x)+\sigma^2-\pj[j-1]\right)\ \mathbb{Q}_{j-1}(\dx{\vec x}) \eqcomment{with Fubini's theorem}
	\\&= \Exp[\log v_j\mid \sigma(\vec x_1, \ldots, \vec x_{j-1})]
	\eqcomment{property of conditional expectation}
\end{align}
It remains to show $\pj=\pj[j-1]+r_{j}(\vec x)$ where $r_{j}(\vec x)\geq 0$.
For readability, we define $\vec v_{\vec x}\ce (\vec K_{j-1}+\sigma^2\vec I)^{-1}\kj[j-1]$ and $c\ce v_j^{\!-\!1}$.
First note, that using block-matrix inversion we can write
\begin{align}
	(\vec K_j+\sigma^2\vec I_j)^{-1}=
	\begin{bmatrix}
		(\vec K_{j-1}+\sigma^2\vec I_{j-1})^{-1}+\vec v_{\vec x_j}\vjinv\vec v_{\vec x_j}\Trans & -\vec v_{\vec x_j}\vjinv
		\\-\vec v_{\vec x_j}\Trans \vjinv & \vjinv
	\end{bmatrix}.
\end{align}
Using above observation, we can transform $\pj$.
\begin{align}
	\pj&=
	\begin{bmatrix}
		\kj[j-1]\Trans & k(\vec x_j, \vec x)
	\end{bmatrix}
	\\&\quad\cdot
	\begin{bmatrix}
		(\vec K_{j-1}+\sigma^2\vec I)^{-1}+\vec v_{\vec x_j}\vjinv\vec v_{\vec x_j}\Trans & -\vec v_{\vec x_j}\vjinv
		\\-\vec v_{\vec x_j}\Trans \vjinv & \vjinv
	\end{bmatrix}
	\\&\quad\cdot
	\begin{bmatrix}
		\kj[j-1] \\ k(\vec x_j, \vec x)
	\end{bmatrix}
	\eqcomment{definition of $\pj$ and using above observation}
	\\&=
	\begin{bmatrix}
		\kj[j-1]\Trans & k(\vec x, \vec x_j)
	\end{bmatrix}
	\\&\quad\cdot
	\begin{bmatrix}
		\vec v_{\vec x}+\vec v_{\vec x_j}\vjinv\vec v_{\vec x_j}\Trans\kj[j-1] -\vec v_{\vec x_j}\vjinv k(\vec x, \vec x_j)
		\\-\vec v_{\vec x_j}\Trans \kj[j-1] \vjinv +  \vjinv k(\vec x, \vec x_j)
	\end{bmatrix}
	\eqcomment{evaluating the RHS matrix-vector multiplication}
	\\&=
	\kj[j-1]\Trans\vec v_{\vec x}+\vjinv(\vec v_{\vec x_j}\Trans\kj[j-1])^2
	\\&\quad -2\vec v_{\vec x_j}\Trans\kj[j-1]\vjinv k(\vec x, \vec x_j) + \vjinv k(\vec x, \vec x_j)^2
	\eqcomment{evaluating the vector product}
	\\&=\kj[j-1]\Trans\vec v_{\vec x}+\vjinv(k(\vec x, \vec x_j)-\vec v_{\vec x_j}\Trans \kj[j-1])^2
	\eqcomment{rearranging terms into a quadratic}
	\\&=\pj[j-1]+\vjinv(k(\vec x, \vec x_j)-\vec v_{\vec x_j}\Trans \kj[j-1])^2
	\eqcomment{definition of $\pj[j-1]$}
\end{align}
This shows that
\begin{align}
	\pj&=\pj[j-1]+r_{j}(\vec x) \text{ , where} \label{eq:pj_recursion}\\
	r_{j}(\vec x)&\ce\vjinv(k(\vec x, \vec x_j)-\vec v_{\vec x_j}\Trans \kj[j-1])^2\geq 0. \label{eq:pj_decrease}
\end{align}
\end{proof}

\begin{lemma}
	\label{thm:quad_form_sum}
	\newcommand{\kj}[1][j]{\vec k_{#1}(\vec x)} 
	\renewcommand{\vj}[1][n+1]{p_{#1}}
	The term $\vec y\Trans\inv{\mat K}\vec y$ can be written as
	\begin{align}
		\vec y\Trans\inv{\mat K}\vec y=\sum_{n=1}^N \frac{\left(y_n-\GPmean[n-1]{n}\right)}{\GPvar[n-1]{n}}.
	\end{align}
\end{lemma}
\begin{proof}
	\newcommand{\kj}[1][j]{\vec k_{#1}(\vec x)} 
	\renewcommand{\vj}[1][n+1]{p_{#1}}
	Define
	\begin{align*}
	\kj & \ce [k(\vec x, \vec x_1), ..., k(\vec x, \vec x_j)]\Trans\in\Re^j
	\\\vec k_{j+1} & \ce \vec k_j(\vec x_{j+1})\in \Re^j
	\\\vj[j] & \ce k(\vec x_j,\vec x_j)+\sigma^2-\vec k_{j}\Trans\inv{(\mat K_{j-1}+\sigma^2\mat I)}\vec k_{j}
	\\\vec\alpha &\ce (\vec K_{n}+\sigma^2\vec I)^{-1}\vec k_{n+1}
	\end{align*}
	First note, that using block-matrix inversion we can write
	\begin{align*}
		(\vec K_{n+1}+\sigma^2\vec I)^{-1}=
		\begin{bmatrix}
			(\vec K_{n}+\sigma^2\vec I)^{-1}+\vec\alpha\vj^{-1}\vec\alpha\Trans & -\vec\alpha\vj^{-1}
			\\-\vec\alpha\Trans \vj^{-1} & \vj^{-1}
		\end{bmatrix}.
	\end{align*}
	This allows to write
	\begin{align*}
		&\vec y_{n+1}\Trans \inv{(\mat K_{n+1}+\sigma^2\mat I)}\vec y_{n+1}
		\\&=
		\concat{\vec y_n\Trans}{y_{n+1}}
		\begin{bmatrix}
			(\vec K_{n}+\sigma^2\vec I)^{-1}+\vec\alpha\vj^{-1}\vec\alpha\Trans & -\vec\alpha\vj^{-1}
			\\-\vec\alpha\Trans \vj^{-1} & \vj^{-1}
		\end{bmatrix}
		\stack{\vec y_n}{y_{n+1}}
		\eqcomment{using above observation}
		\\&=
		\concat{\vec y_n\Trans}{y_{n+1}}
		\stack{
			(\vec K_{n}+\sigma^2\vec I)^{-1}\vec y_n+\vec\alpha\vj^{-1}\vec\alpha\Trans\vec y_n -\vec\alpha\vj^{-1}y_{n+1}}
		{-\vec\alpha\Trans \vj^{-1}\vec y_n + \vj^{-1}y_{n+1}}
		\eqcomment{simplifying from the right}
		\\&=\vec y_n\Trans (\vec K_{n}+\sigma^2\vec I)^{-1}\vec y_n+\vec y_n\Trans \vec\alpha\vj^{-1}\vec\alpha\Trans\vec y_n -\vec y_n\Trans\vec\alpha\vj^{-1}y_{n+1}-y_{n+1}\vec\alpha\Trans \vj^{-1}\vec y_n + y_{n+1}\vj^{-1}y_{n+1}
		\eqcomment{simplifying from the left}
		\\&=\vec y_n\Trans (\vec K_{n}+\sigma^2\vec I)^{-1}\vec y_n+\vj^{-1}(\vec y_n\Trans \vec\alpha\vec\alpha\Trans\vec y_n -\vec y_n\Trans\vec\alpha y_{n+1}-y_{n+1}\vec\alpha\Trans \vec y_n + y_{n+1}y_{n+1})
		\eqcomment{pulling out $\vj^{-1}$}
		\\&=\vec y_n\Trans (\vec K_{n}+\sigma^2\vec I)^{-1}\vec y_n+\vj^{-1}((\vec y_n\Trans \vec\alpha)^2 -2\vec y_n\Trans\vec\alpha y_{n+1}+ y_{n+1}^2)
		\eqcomment{simplifying}
		\\&=\vec y_n\Trans (\vec K_{n}+\sigma^2\vec I)^{-1}\vec y_n+\vj^{-1}(\vec y_n\Trans \vec\alpha -y_{n+1})^2
		\eqcomment{simplifying}
	\end{align*}
Now observe that the last addend is indeed the mean square error divided by the posterior variance.
By induction the result follows.
\end{proof}

\renewcommand{\otherX}{\overline{\mat X}}
\begin{lemma}
\label{lemma:post_kernel}
For all $t,m \in\mathbb{N}$ with $1\leq t+m \leq N$
\begin{align*}
k_{t+m}(\vec x_a, \vec x_b)=k_t(\vec x_a, \vec x_b)-k_t(\vec x_a, \otherX)\inv{\left(k_t(\otherX)+\sigma^2\mat I_{m}\right)}k_t(\otherX, \vec x_b)
\end{align*}
where $k_t(\vec x_a, \vec x_b)\ce k(\vec x_a, \vec x_b)-k(\vec x_a, \mat X_t)\inv{\left(k(\mat X_{:t}, \mat X_{:t})+\noisefm{\mat X_{:t}}\right)}k(\mat X_t, \vec x_b)$ and $\otherX\ce\mat X_{t:t+m}$.
\end{lemma}
\begin{proof}
	\begin{align*}
		&k_{t+m}(\vec x_a, \vec x_b)
		\\&=k(\vec x_a, \vec x_b)-k(\vec x_a, \mat X_{t+m})\inv{\mat A_{t+m}}k(\mat X_{t+m}, \vec x_b)
		\eqcomment{by definition}
		\\&=k(\vec x_a, \vec x_b)-\concat{k(\vec x_a, \mat X_t)}{k(\vec x_a, \otherX)}
		\inv{
			\begin{bmatrix}
				k(\mat X_t)+\sigma^2\mat I_t & k(\mat X_t, \otherX)\\
				k(\otherX, \mat X_t) & k(\otherX) + \sigma^2 \mat I_t
			\end{bmatrix}
		}
		\stack{k(\mat X_t, \vec x_b)}{k(\otherX, \vec x_b)}
		\eqcomment{in block notation}
		\\&=k(\vec x_a, \vec x_b)-\concat{k(\vec x_a, \mat X_t)}{k(\vec x_a, \otherX)}
		\inv{
			\begin{bmatrix}
				\mat A_t & k(\mat X_t, \otherX)\\
				k(\otherX, \mat X_t) & k(\otherX) + \sigma^2 \mat I_t
			\end{bmatrix}
		}
		\stack{k(\mat X_t, \vec x_b)}{k(\otherX, \vec x_b)}
		\eqcomment{using the definition of $\mat A_t$}
		\\&=k(\vec x_a, \vec x_b)-\concat{k(\vec x_a, \mat X_t)}{k(\vec x_a, \otherX)}\cdot
		\\&{
		\scriptsize
		\begin{bmatrix}
			\inv{\mat A_t} +  \inv{\mat A_t}k(\mat X_t, \otherX)\inv{\left(k(\otherX) + \sigma^2 \mat I_t-k(\otherX, \mat X_t)\inv{\mat A_t}k(\mat X_t, \otherX)\right)}k(\otherX, \mat X_t)\inv{\mat A_t}
			&
			-\inv{\mat A_t}k(\mat X_t, \otherX)\inv{\left(k(\otherX) + \sigma^2 \mat I_t-k(\otherX, \mat X_t)\inv{\mat A_t}k(\mat X_t, \otherX)\right)}\\
			-\inv{\left(k(\otherX) + \sigma^2 \mat I_t-k(\otherX, \mat X_t)\inv{\mat A_t}k(\mat X_t, \otherX)\right)}k(\otherX, \mat X_t)\inv{\mat A_t}
			&
			\inv{\left(k(\otherX) + \sigma^2 \mat I_t-k(\otherX, \mat X_t)\inv{\mat A_t}k(\mat X_t, \otherX)\right)}
		\end{bmatrix}\cdot}\\
		&\stack{k(\mat X_t, \vec x_b)}{k(\otherX, \vec x_b)}
		\eqcomment{applying block-matrix inversion}
		\\&=k(\vec x_a, \vec x_b)-\concat{k(\vec x_a, \mat X_t)}{k(\vec x_a, \otherX)}\cdot
		\\&\begin{bmatrix}
			\inv{\mat A_t} +  \inv{\mat A_t}k(\mat X_t, \otherX)\inv{\left(k_t(\otherX) + \sigma^2 \mat I_t\right)}k(\otherX, \mat X_t)\inv{\mat A_t} & -\inv{\mat A_t}k(\mat X_t, \otherX)\inv{\left(k_t(\otherX) + \sigma^2 \mat I_t\right)}\\
			-\inv{\left(k_t(\otherX) + \sigma^2 \mat I_t\right)}k(\otherX, \mat X_t)\inv{\mat A_t} & \inv{\left(k_t(\otherX) + \sigma^2 \mat I_t\right)}
		\end{bmatrix}\cdot\\
		&\stack{k(\mat X_t, \vec x_b)}{k(\otherX, \vec x_b)}
		\eqcomment{applying the definition of $k_t$}
		\\&=k(\vec x_a, \vec x_b)-\concat{k(\vec x_a, \mat X_t)}{k(\vec x_a, \otherX)}\cdot
		\\&\begin{bmatrix}
			\inv{\mat A_t}k(\mat X_t, \vec x_b) +  \inv{\mat A_t}k(\mat X_t, \otherX)\inv{\left(k_t(\otherX) + \sigma^2 \mat I_t\right)}k(\otherX, \mat X_t)\inv{\mat A_t}k(\mat X_t, \vec x_b) -\inv{\mat A_t}k(\mat X_t, \otherX)\inv{\left(k_t(\otherX) + \sigma^2 \mat I_t\right)}k(\otherX, \vec x_b)\\
			-\inv{\left(k_t(\otherX) + \sigma^2 \mat I_t\right)}k(\otherX, \mat X_t)\inv{\mat A_t}k(\mat X_t, \vec x_b) + \inv{\left(k_t(\otherX) + \sigma^2 \mat I_t\right)}k(\otherX, \vec x_b)
		\end{bmatrix}
		\eqcomment{evaluating multiplication with right-most vector}
		\\&=k(\vec x_a, \vec x_b)-\concat{k(\vec x_a, \mat X_t)}{k(\vec x_a, \otherX)}\cdot
		\\&\begin{bmatrix}
			\inv{\mat A_t}k(\mat X_t, \vec x_b) -  \inv{\mat A_t}k(\mat X_t, \otherX)\inv{\left(k_t(\otherX) + \sigma^2 \mat I_t\right)}\left(k(\otherX, \vec x_b)-k(\otherX, \mat X_t)\inv{\mat A_t}k(\mat X_t, \vec x_b)\right)\\
			\inv{\left(k_t(\otherX) + \sigma^2 \mat I_t\right)}\left(k(\otherX, \vec x_b)-k(\otherX, \mat X_t)\inv{\mat A_t}k(\mat X_t, \vec x_b)\right)
		\end{bmatrix}
		\eqcomment{rearranging}
		\\&=k(\vec x_a, \vec x_b)-\concat{k(\vec x_a, \mat X_t)}{k(\vec x_a, \otherX)}\cdot
		\\&\begin{bmatrix}
			\inv{\mat A_t}k(\mat X_t, \vec x_b) -  \inv{\mat A_t}k(\mat X_t, \otherX)\inv{\left(k_t(\otherX) + \sigma^2 \mat I_t\right)}k_t(\otherX, \vec x_b)\\
			\inv{\left(k_t(\otherX) + \sigma^2 \mat I_t\right)}k_t(\otherX, \vec x_b)
		\end{bmatrix}
		\eqcomment{applying the definition of $k_t$}
		\\&=k(\vec x_a, \vec x_b)
		-k(\vec x_a, \mat X_t)\inv{\mat A_t}k(\mat X_t, \vec x_b) \\
		&\quad+  k(\vec x_a, \mat X_t)\inv{\mat A_t}k(\mat X_t, \otherX)\inv{\left(k_t(\otherX) + \sigma^2 \mat I_t\right)}k_t(\otherX, \vec x_b)
		-k(\vec x_a, \otherX)\inv{\left(k_t(\otherX) + \sigma^2 \mat I_t\right)}k_t(\otherX, \vec x_b)
		\eqcomment{evaluating the vector product}
		\\&=k(\vec x_a, \vec x_b)
		-k(\vec x_a, \mat X_t)\inv{\mat A_t}k(\mat X_t, \vec x_b)
		-\left(k(\vec x_a, \otherX)-k(\vec x_a, \mat X_t)\inv{\mat A_t}k(\mat X_t, \otherX)\right)\inv{\left(k_t(\otherX) + \sigma^2 \mat I_t\right)}k_t(\otherX, \vec x_b)
		\eqcomment{rearranging}
		\\&=k_t(\vec x_a, \vec x_b)-k_t(\vec x_a, \otherX)\inv{\left(k_t(\otherX) + \sigma^2 \mat I_t\right)}k_t(\otherX, \vec x_b)
		\eqcomment{applying the definition of $k_t$}
	\end{align*}
\end{proof}
\begin{lemma}
	\todo{In German this is called ``little Gau\ss{}''. Is there an equivalent English name?}
	\label{lemma:little_gauss}
	\begin{align}
		\sum_{j=t+1}^n\sum_{i=t_0+1}^{j-1}1=(n-t)\left(\frac{n+t-1}{2}-t_0\right)
	\end{align}
\end{lemma}
\begin{proof}
	\allowdisplaybreaks
	\begin{align}
		\sum_{j=t+1}^n\sum_{i=t_0+1}^{j-1}1&=\sum_{j=t+1}^n(j-1-t_0)
		\\&=\sum_{j=0}^{n-t-1}(j-1-t_0+t+1)
		\\&=\sum_{j=0}^{n-t-1}(j+t-t_0)
		\\&=(t-t_0)(n-t)+\sum_{j=0}^{n-t-1}j
		\\&=(t-t_0)(n-t)+\frac{(n-t-1)(n-t)}{2}
		\\&=(n-t)\left(\frac{n-t-1}{2}+t-t_0\right)
		\\&=(n-t)\left(\frac{n+t-1}{2}-t_0\right)
	\end{align}
\end{proof}

\begin{lemma}[Link between the Cholesky and Gaussian process regression]
	\label{lemma:cholesky_and_gp_variance}
	Denote with $\mat C_N$ the Cholesky decomposition of $\mat K$, so that $\mat C_N\mat C_N\Trans=\mat K$.
	The $n$-th diagonal element of $\mat C_N$, squared, is equivalent to $\GPvar[n-1]{n}$:
	\[
	[\mat C_N]_{nn}^2=\GPvar[n-1]{n}
	\, .
	\]
\end{lemma}
\begin{proof}
	With abuse of notation, define $\mat C_1\ce \sqrt{k(\vec x_1, \vec x_1)}$ and
	$$\mat C_N \ce \begin{bmatrix}
		\mat C_{N-1} & \vec 0 \\
		\vec k_N\Trans\mat C_{N-1}\iTrans & \sqrt{k(\vec x_N, \vec x_N)+\sigma^2-\vec k_N\Trans\inv{(\mat K_{n-1}+\sigma^2\mat I_{n-1})}\vec k_N}
	\end{bmatrix}.$$
	We will show that the lower triangular matrix $\mat C_N$ satisfies $\mat C_N\mat C_N\Trans = \mat K_N+\sigma^2\mat I_N$.
	Since the Cholesky decomposition is unique \citep[Theorem~4.2.7]{golub2013matrix4}, $\mat C_N$ must be the Cholesky decomposition of $\mat K$.
	Furthermore, by definition of $\mat C_N$, $[\mat C_N]_{NN}^2=k(\vec x_N, \vec x_N)+\sigma^2-\vec k_N\Trans\inv{(\mat K_{n-1}+\sigma^2\mat I_{n-1})}\vec k_N$.
	The statement then follows by induction.

	To remain within the text margins, define
	$$x\ce \vec k_N\Trans\mat C_{N-1}\iTrans\mat C_{N-1}^{\!-1}\vec k_N+k(\vec x_N, \vec x_N)+\sigma^2-\vec k_N\Trans\inv{(\mat K_{n-1}+\sigma^2\mat I_{n-1})}\vec k_N.$$
	We want to show that $\mat C_N\mat C_N\Trans = \mat K_N+\sigma^2\mat I_N$.
	\begin{align*}
		\mat C_{N}\mat C_{N}\Trans &= \begin{bmatrix}
			\vec C_{N-1} & \vec 0 \\
			\vec k_N\Trans\mat C_{N-1}\iTrans & \sqrt{k(\vec x_N, \vec x_N)+\sigma^2-\vec k_N\Trans\inv{(\mat K_{n-1}+\sigma^2\mat I_{n-1})}\vec k_N}
		\end{bmatrix}\\
		&\quad\cdot\begin{bmatrix}
			\vec C_{N-1}\Trans & \mat C_{N-1}^{\!-1}\vec k_N \\
			\vec 0\Trans & \sqrt{k(\vec x_N, \vec x_N)+\sigma^2-\vec k_N\Trans\inv{(\mat K_{n-1}+\sigma^2\mat I_{n-1})}\vec k_N}
		\end{bmatrix}
		\\&=\begin{bmatrix}
			\mat C_{N-1}\mat C_{N-1}\Trans & \mat C_{N-1}\mat C_{N-1}^{\!-1}\vec k_N\\
			\vec k_N\Trans \mat C_{N-1}\iTrans \mat C_{N-1}\Trans & x
		\end{bmatrix}
		\\&=\begin{bmatrix}
			\mat K_{N-1} +\sigma^2\vec I_{N-1} & \vec k_N\\
			\vec k_N\Trans & x
		\end{bmatrix}
	\end{align*}
	Also $x$ can be simplified further.
	\begin{align*}
		x&=\vec k_N\Trans\mat C_{N-1}\iTrans\mat C_{N-1}^{\!-1}\vec k_N+k(\vec x_N, \vec x_N)+\sigma^2-\vec k_N\Trans\inv{(\mat K_{n-1}+\sigma^2\mat I_{n-1})}\vec k_N
		\\&=\vec k_N\Trans\inv{(\mat K_{n-1}+\sigma^2\mat I_{n-1})}\vec k_N+k(\vec x_N, \vec x_N)+\sigma^2-\vec k_N\Trans\inv{(\mat K_{n-1}+\sigma^2\mat I_{n-1})}\vec k_N
		\\&=k(\vec x_N, \vec x_N)+\sigma^2.
	\end{align*}
\end{proof}

	\end{document}